\documentclass{article}
\usepackage{times}
\usepackage{xcolor}
\usepackage{soul}
\usepackage[small]{caption}
\usepackage{subfigure}
\usepackage{graphicx} 
\usepackage{todonotes}

\usepackage[letterpaper]{geometry}
\usepackage[parfill]{parskip}
\PassOptionsToPackage{numbers, compress}{natbib}
\usepackage[numbers, compress]{natbib}
\usepackage{xr-hyper,refcount}

\usepackage{color,soul}
\usepackage{microtype}

\usepackage{hyperref}
\usepackage{xspace}

\usepackage{setspace}

\usepackage{pifont}

\usepackage[utf8]{inputenc} 
\usepackage[T1]{fontenc}    
\usepackage{hyperref}       
\usepackage{url}            
\usepackage{booktabs}       
\usepackage{amsfonts}       
\usepackage{nicefrac}       
\usepackage{microtype}      
\usepackage{graphicx}
\usepackage{subfigure}
\usepackage{booktabs} 
\usepackage{pstool}

\usepackage{wrapfig}

\usepackage{color,soul}
\usepackage{microtype}
\usepackage{graphicx}
\usepackage{subfigure}
\usepackage{booktabs} 
\usepackage{hyperref}
\usepackage{algorithm}
\usepackage[noend]{algorithmic}
\usepackage{xspace}
\usepackage{amsmath,amsfonts,epsf,psfrag,amssymb,bbm,bm,amsthm}
\usepackage{setspace}
\usepackage{makecell}

\newcommand{\opt}{{\omega^*}}

\newcommand{\wt}[1]{\widetilde{#1}}
\newcommand{\wh}[1]{\widehat{#1}}
\newcommand{\wo}[1]{\overline{#1}}
\newcommand{\wb}[1]{\overline{#1}}

\DeclareMathOperator*{\argmax}{arg\,max}

\newcommand{\TS}{\text{TS}}
\newcommand{\OFU}{\text{OFU}}
\newcommand{\PS}{\text{PS}}
\newcommand{\DQN}{\textsc{\small{DQN}}\xspace}

\newcommand{\DDQN}{\textsc{\small{DDQN}}\xspace}

\newcommand{\bdqn}{\textsc{\small{BDQN}}\xspace}
\newcommand{\PSRL}{\textsc{\small{PSRL}}\xspace}
\newcommand{\LinPSRL}{\textsc{\small{LinPSRL}}\xspace}
\newcommand{\LinUCB}{\textsc{\small{LinUCB}}\xspace}

\renewcommand{\Re}{\mathbb{R}}
\renewcommand{\S}{\mathcal{S}}
\newcommand{\F}{\mathcal F}
\newcommand{\A}{\mathcal A}
\newcommand{\X}{\mathcal X}
\newcommand{\D}{\mathcal D}
\renewcommand{\L}{\mathcal L}

\newcommand{\y}{\textbf{y}}
\newcommand{\C}{\mathcal C}
\newcommand{\E}{\mathbb E}
\newcommand{\Prob}{\mathbb P}
\newcommand{\I}{\mathbb{I}}
\newcommand{\N}{\mathcal N}
\newcommand{\R}{\mathcal{R}}
\newtheorem{lemma}{Lemma}

\newtheorem{assumption}{Assumption}
\newtheorem{theorem}{Theorem}

\title{Efficient Exploration through Bayesian  Deep Q-Networks}

\begin{document}
\author{
  Kamyar Azizzadenesheli, Animashree Anandkumar\\
  Caltech
}

\date{}
\maketitle

\begin{abstract}
We study reinforcement learning (RL) in high dimensional episodic Markov decision processes (MDP). We consider value-based RL when the optimal Q-value  is a linear function of $d$-dimensional state-action feature representation. For instance, in deep-Q networks (\DQN), the Q-value is a linear function of the feature representation layer (output layer). We propose two algorithms, one based on optimism, \LinUCB,  and another based on posterior sampling, \LinPSRL. We guarantee frequentist and Bayesian regret upper bounds of $\wt{\mathcal{O}}(d\sqrt{T})$   for these two algorithms, where  $T$ is the number of episodes.
We extend these methods to deep RL and propose Bayesian deep Q-networks ($\bdqn$), which uses an efficient Thompson sampling algorithm for high dimensional RL. We deploy the double \DQN (\DDQN) approach, and instead of learning the last layer of Q-network using linear regression, we use Bayesian linear regression, resulting in an approximated posterior over Q-function. This allows us to directly incorporate the uncertainty over the Q-function and deploy Thompson sampling on the learned posterior distribution resulting in efficient exploration/exploitation trade-off. 
%
%
%
We empirically study the behavior of \bdqn on a wide range of Atari games. Since \bdqn carries out more efficient exploration and exploitation, it is able to reach higher return substantially faster compared to \DDQN.
\end{abstract}



 \section{Introduction}
One of the central challenges in reinforcement learning (RL) is to design algorithms with efficient exploration-exploitation trade-off that scale to high-dimensional state and action spaces. Recently, deep RL has shown significant promise in tackling high-dimensional (and continuous) environments. These successes are mainly demonstrated in simulated domains where exploration is inexpensive and simple exploration-exploitation approaches such as $\varepsilon$-greedy or Boltzmann strategies are deployed. $\varepsilon$-greedy  chooses the  greedy action with $1-\varepsilon$ probability and randomizes over all the actions, and does not consider  the estimated Q-values or its uncertainties. The Boltzmann strategy considers the estimated Q-values to guide the decision making but still does not exploit their uncertainties in estimation. For complex environments, more statistically efficient strategies are required. One such strategy is optimism in the face of uncertainty (\OFU), where we follow the decision suggested by the optimistic estimation of the environment and guarantee efficient exploration/exploitation strategies. Despite compelling theoretical results, these methods are mainly model based and limited to tabular settings~\citep{jaksch2010near,auer2003using}.

An alternative to \OFU~is posterior sampling (\PS), or more general randomized approach, is Thompson Sampling ~\citep{thompson1933likelihood} which, under the Bayesian framework, maintains a posterior distribution over the environment model, see Table~\ref{table:reasons}. 
Thompson sampling has shown strong performance in many low dimensional settings such as multi-arm bandits~\citep{chapelle2011empirical} and small tabular MDPs~\citep{osband2013more}.  Thompson sampling requires  sequentially sampling of the models from the (approximate) posterior or uncertainty and to act according to the sampled models to trade-off exploration and exploitation. 
However, the computational costs in posterior computation and planning become intractable as the problem dimension grows. 

To mitigate the computational bottleneck,~\citet{osband2014generalization} consider episodic and tabular MDPs where the optimal Q-function is linear in the state-action representation. They deploy Bayesian linear regression (BLR)~\citep{rasmussen2006gaussian} to construct an approximated posterior distributing over the Q-function and employ Thompson sampling for exploration/exploitation. The authors guarantee an order optimal regret upper bound on the tabular MDPs in the presence of a Dirichlet prior on the model parameters. Our paper is a high dimensional and general extension of \citep{osband2014generalization}.

\newcommand{\cmark}{\ding{51}}
\newcommand{\xmark}{\ding{55}}
\begin{table*}[ht]
  \centering
  \caption{Thompson Sampling, similar to \OFU~and \PS, incorporates the estimated Q-values, including the greedy actions, and uncertainties to guide exploration-exploitation trade-off. $\varepsilon$-greedy and Boltzmann exploration fall short in properly incorporating them. $\varepsilon$-greedy considers the most greedy action, and Boltzmann exploration just exploit the estimated returns. Full discussion in Appendix~\ref{apx:TSvsE}.}
  \footnotesize
  \begin{tabular}{l|c|c|c}
    \toprule
 Strategy  & Greedy-Action  & Estimated Q-values  &  Estimated uncertainties \\
  \midrule
$\varepsilon$-greedy &\cmark &\xmark&\xmark \\
\midrule
Boltzmann exploration &\cmark&\cmark&\xmark \\
\midrule
Thompson Sampling &\cmark&\cmark&\cmark \\
    \bottomrule
  \end{tabular}
  \label{table:reasons}
\end{table*}


While the study of RL in general MDPs is challenging, recent advances in the understanding of linear bandits, as an episodic MDPs with episode length of one, allows tackling high dimensional environment. This class of RL problems is known as LinReL. In linear bandits, both \OFU~\citep{abbasi-yadkori2011improved, li2010contextual} and Thompson sampling ~\citep{russo2014learning,agrawal2013thompson,abeille2017linear}  guarantee promising results for high dimensional problems. In this paper, we extend LinReL to MDPs.


\noindent{\bf Contribution 1 -- Bayesian and frequentist regret analysis:}
We study RL in episodic MDPs where the optimal Q-function is a linear function of a $d$-dimensional feature representation of state-action pairs. We propose two algorithms, \LinPSRL, a Bayesian method using \PS, and \LinUCB, a frequentist method using \OFU. \LinPSRL constructs a posterior distribution over the linear parameters of the Q-function. At the beginning of each episode, \LinPSRL draws a sample from the posterior then acts optimally according to that model. \LinUCB constructs the upper confidence bound on the linear parameters and in each episode acts optimally with respect to the  optimistic model. We provide theoretical performance guarantees and show that after $T$ episodes, the Bayesian regret of \LinPSRL and the frequentist regret of \LinUCB are both upper bounded by $\widetilde{\mathcal{O}}(d\sqrt{T})$\footnote{The dependency in the episode length is more involved and details are in Section~\ref{sec:Theory}.}.

\noindent{\bf Contribution 2 -- From theory to practice: } 
While both \LinUCB and \LinPSRL are statistically designed for high dimensional RL, their computational complexity can make them practically infeasible, e.g., maintaining the posterior can  become intractable. To mitigate this shortcoming, we propose a unified method based on the BLR approximation of these two methods. This unification is inspired by the analyses in ~\citet{abeille2017linear,abbasi-yadkori2011improved} for linear bandits. 
1) For \LinPSRL: we deploy BLR to approximate the posterior distribution over the Q-function using conjugate Gaussian prior and likelihood. In tabular MDP, this approach turns out to be   similar to \citet{osband2014generalization}\footnote{We refer the readers to this work for an empirical study of BLR on tabular environment.}.
2) For \LinUCB: we deploy BLR to fit a Gaussian distribution to the frequentist upper confidence bound constructed in \OFU~(Fig 2 in \citet{abeille2017linear}). These two approximation procedures result in the same Gaussian distribution, and therefore, the same algorithm. Finally, we deploy Thompson sampling on this approximated distribution over the Q-functions. While it is clear that this approach is an approximation to \PS, \citet{abeille2017linear} show that this approach is also an approximation to \OFU~\footnote{An extra $\sqrt{d}$ expansion of the Gaussian approximation is required for the theoretical analysis.
}. For practical use, we extend this unified algorithm to deep RL, as described below. 

\noindent{\bf Contribution 3 -- Design of \bdqn: } 
We introduce Bayesian Deep Q-Network (\bdqn), a Thompson sampling based deep RL algorithm, as an extension of our theoretical development to deep neural networks. We follow the \DDQN~\citep{van2016deep} architecture and train the Q-network in the same way except for the last layer (the linear model) where we use BLR instead of linear regression. We deploy Thompson sampling on the approximated posterior of the Q-function to balance between exploration and exploitation. Thus, \bdqn requires a simple and a minimal modification to the standard \DDQN implementation.

We empirically study the behavior of \bdqn on a wide range of Atari games~\citep{bellemare2013arcade,machado2017revisiting}. Since \bdqn follows an efficient exploration-exploitation strategy, it reaches much higher cumulative rewards in fewer interactions, compared to its $\varepsilon$-greedy predecessor \DDQN. We empirically observed that \bdqn achieves \DDQN performance in less than 5M$\pm$1M interactions for almost half of the games while the cumulative reward improves by a  median of $300\%$ with a maximum of  $80K\%$ on all games. Also, \bdqn has $300\%\pm40\%$ (mean and standard deviation) improvement over these games on the area under the performance measure. Thus, \bdqn achieves better sample complexity due to a better exploration/exploitation trade-off.



%
\noindent{\bf Comparison: } Recently, many works have studied efficient exploration/exploitation in high dimensional environments.~\citep{lipton2016efficient} proposes a variational inference-based approach to help the exploration.~\citet{bellemare2016unifying} proposes a   surrogate for optimism.~\citet{osband2016deep} proposes an ensemble of many \DQN models. These approaches are significantly more expensive than \DDQN~\citep{osband2016deep} while also require a massive hyperparameter tuning effort~\citep{bellemare2016unifying}. 

In contrast, our approach has the following desirable properties: 
1) \textbf{Computation}: \bdqn nearly has a same computational complexity as \DDQN since there is no backpropagation in the last layer of \bdqn (therefore faster), but instead, there is a BLR update which requires inverting a small $512\times512$ matrix (order of less than a second), once in a while.
2) \textbf{Hyperparameters}: no exhaustive hyper-parameter tuning. We spent less than two days of academic level GPU time on hyperparameter tuning of BLR in \bdqn which is another evidence on its significance. 
3) \textbf{Reproducibility}: All the codes, with detailed comments and explanations, are publicly available.

\section{Linear Q-function}\label{sec:Theory}

\subsection{Preliminaries}
Consider an episodic MDP $M:=\langle \X, \A, P,P_0, R,\gamma,H\rangle$, with horizon length $H$, state space $\X$, closed action set $\A$, transition kernel $P$, initial state distribution $P_0$, reward distribution $R$, discount factor $0\leq\gamma\leq 1$. For any natural number $H$,  $[H]=\lbrace 1,2, \ldots,H\rbrace$. The time step within the episode, $h$, is encoded in the state, i.e., $\X^h$ and $\A^h,~\forall h\in[H]$. We drop $h$ in state-action definition for brevity.
$\|\cdot\|_2$ denotes the spectral norm and for any positive definite matrix $\chi$, $\|\cdot\|_{\chi}$ denotes the  $\chi$ matrix-weighted spectral norm. At state $x^h$ of time step $h$,  we define Q-function as agent's expected return after taking action $a^h$ and following policy $\pi$, a mapping from state to action.

\begin{align*}
    Q(x^h,a^h) = \E\left[R(x^h,a^h) + \gamma  Q(x^{h+1},\pi(x^{h+1}))\Big|x^h,a^h\right]
\end{align*}
Following the Bellman optimality in MDPs, we have that for the optimal Q-function
\begin{align*}
    Q^*(x^h,a^h) = \E\left[R(x^h,a^h) + \gamma  \max_{a\in\A} Q^*(x^{h+1},a)\Big|x^h,a^h\right]
\end{align*}

We consider MDPs where the optimal $Q$-function, similar to linear bandits, is linear in state-action representations $\phi(\cdot,\cdot):=\X\times\A\rightarrow \R^d$, i.e., $Q^{\opt}_{{\pi^*}}(x^h,a^h):=\phi(x^h,a^h)^\top\opt^h,~\forall x^h,a^h\in\X\times \A$. $\opt$ denotes the set of $\opt^h\in\Re^d~\forall h\in[H]$, representing the environment and $\pi^*$ the set of ${\pi^*}^h:\X\rightarrow \A$ with ${\pi^*}^h(x):=\arg\max_{a\in\A}Q^{{\opt}^h}_{{\pi^*}}(x^h,a^h)$. $V^{\opt}_{{\pi^*}}$ denotes the corresponding value function. 

\subsection{LinReL}
At each time step $h$ and state action pair $x^{h},a^{h}$, define $\nu^h$ a mean zero random variable that captures stochastic reward and transition at time step $h$:
\begin{equation*}
\phi(x^{h},a^{h})^\top\opt^{h} + \nu^h =R^h+\gamma	\phi(x^{h+1},{\pi^*}^{h+1}(x^{h+1}))^\top\opt^{h+1}
\end{equation*}
where $R^h$ is the reward at time step $h$. Definition of $\nu^h$ plays an important role since knowing $\opt^{h+1}$ and ${\pi^*}^{h+1}$ reduces the learning of $\opt^{h}$ to the standard Martingale based linear regression problem. Of course we neither have $\opt^{h+1}$ nor have ${\pi^*}^{h+1}$.

\textbf{\LinPSRL}(Algorithm~\ref{alg:psrl}): In this Bayesian approach, the agent maintains the prior over the vectors $\opt^h, \forall h$ and given the collected experiences, updates their posterior at the beginning of each episode. At the beginning of each episode $t$, the agent draws $\omega_t^h, \forall h,$ from the posterior, and follows their induced policy $\pi_t^h$, i.e.,  $a_t^h:=\arg\max_{a\in\A} \phi^\top(x^h,a)\omega_t^h, \forall x^h\in \X$. 

\textbf{\LinUCB}(Algorithm~\ref{alg:optimism}): In this frequentist approach, at the beginning of $t$'th episode, the agent exploits the so-far collected experiences and estimates $\opt^h$ up to a high probability confidence intervals $\C_{t-1}^h$ i.e., $\opt^h\in C_{t-1}^h,~\forall h$. At each time step $h$, given a state $x_t^h$, the agent follows the optimistic policy:  $\wt\pi_t^h(x_t^h)=\arg\max_{a\in\A}\max_{\omega\in\C^h_{t-1}}\phi^\top(X_t^h,a)\omega$. 

\begin{figure*}[t]
\vspace*{-0.1cm}
  \begin{minipage}{0.45\textwidth}
\begin{algorithm}[H]
\caption{\LinPSRL}
\begin{algorithmic}[1]
\STATE Input: the prior and likelihood
\FOR{episode:~t= 1,2,\ldots}
    \STATE $\omega_t^h\sim$ posterior distribution, $\forall h\in[H]$
    \FOR{$h=0$ to the end of episode}
	\STATE Follow $\pi_t$ induced by $\omega_t^h$
\ENDFOR
\STATE Update the posterior
\ENDFOR
\end{algorithmic}
\label{alg:psrl}
\end{algorithm}
  \end{minipage}\hfill
  \begin {minipage}{0.50\textwidth}
     \centering
\begin{algorithm}[H]
\caption{\LinUCB}
\begin{algorithmic}[1]
\STATE Input: $\sigma$, $\lambda$ and $\delta$
\FOR{episode:~t = 1,2,\ldots}
    \FOR{$h=1$ to the end of episode}
    \STATE choose optimistic $\wt\omega_t^h$ in $\C_{t-1}^h(\delta)$
	\STATE Follow $\wt\pi_t^h$ induced by $\wt\omega_t^h$
\ENDFOR
\STATE Compute the confidence $\C_{t}^h(\delta),~\forall h\in[H]$
\ENDFOR
\end{algorithmic}
\label{alg:optimism}
\end{algorithm}
  \end{minipage}\hfill
\end{figure*}

\textbf{Regret analysis:} For both of these approaches, we show that, as we get more samples, the confidence sets $\C_t^h,~\forall h,$ shrink with the rate of $\wt{\mathcal{O}}\left(1/\sqrt{t}\right)$, resulting in more confidence parameter estimation and therefore smaller per step regret (Lemma~\ref{lem:conf} in Appendix~\ref{apx:theory}).
For linear models, define the gram matrix $\chi_t^h$ and also ridge regularized matrix with $\wt\chi^h\in\Re^{d\times d}$ (we set it to $\lambda I$)
\begin{align*}
\chi_t^h := \sum_{i=1}^t \phi_i^h{\phi_i^h}^\top, ~~~~~ \wo\chi_t^h = \chi_t^h+\wt\chi^h
\end{align*}
Following the standard assumption in the self normalized analysis of linear regression and linear bandit~\citep{pena2009self,abbasi-yadkori2011improved}, we have: 
\begin{itemize}
    \item $\forall h\in[H]$: the noise vector $\nu^h~$ is a $\sigma$-sub-Gaussian vector. (refer to Assumption ~\ref{asm:subgaussian}) 
    \item $\forall h\in[H]$ we have $\|\opt^h\|_2\leq L_\omega$, $\|{\phi}(x^h,a^h){\phi}(x^h,a^h)^\top\|_2^2\leq L,~\forall x\in\X,a\in\A$, a.s. 
    \item Expected rewards and returns are in $[0,1]$.
\end{itemize}
 Then, $\forall h$ define $\rho_\lambda^h$ such that:
\begin{align*}
\left(\sum_i^t\|\phi(x_i^{h},\pi^*(x_i^{h}))\|_{{\wo\chi_t^{h}}^{-1}}^2\right)^{1/2}\leq \rho_\lambda^h,~\forall h,t, \textit{with}~ \rho_\lambda^{H+1}=0
\end{align*}
similar to the ridge linear regression analysis in~\citet{hsu2012random}, we require $\rho_\lambda^h<\infty$. This requirement is automatically satisfied if the optimal Q-function is bounded away from zero (all features have large component at least in one direction). Let $\wb\rho^H_\lambda(\gamma)$ denote the following combination of $\rho_\lambda^h$:
\begin{align*}
\wb\rho^H_\lambda(\gamma) := 
\sum_{i=1}^{H}\gamma^{2(H-i)} \left(1+\sum_{j=1}^i\I(j>1) \gamma^{j-1}\prod_{k=1}^{j-1} \rho_\lambda^{H-(i-k)+1}\right)^2
\end{align*}
For any prior and likelihood satisfying these assumptions, we have:


\begin{theorem}[Bayesian Regret]\label{Thm:BR}
For an episodic MDP with episode length $H$, discount factor $\gamma$, and feature map $\phi(x,a)\in\Re^d$,  
after $T$ episodes the Bayesian regret of \LinPSRL is upper bounded as:
\begin{align*}
\textbf{BayesReg}_T&=\E\left[\sum_t^T \left[V^{\omega^*}_{\pi^*}- V^{\omega^*}_{\pi_t}\right]  \right]   =  \mathcal{O} \left(  d\sqrt{\wb\rho^H_\lambda(\gamma)HT}\log(T) \right)
\end{align*}
\end{theorem}
\textit{Proof is given in the Appendix~\ref{sec:proof}}.

\begin{theorem}[Frequentist Regret]\label{Thm:frequentist}
For an episodic MDP with episode length $H$, discount factor $\gamma$, feature map $\phi(x,a)\in\Re^d$, after $T$ episodes the frequentist regret of \LinUCB is upper bounded as:
\begin{align*}
  \textbf{Reg}_T  :  &= \E \left[ \sum_t^T  \left[V^{\omega^*}_{\pi^*}  -  V^{\omega^*}_{\wt\pi_t}\right]  \Big|\omega^* \right]  =  \mathcal{O} \left(  d\sqrt{\wb\rho^H_\lambda (\gamma)HT}\log(T) \right)
\end{align*}
\end{theorem}
\textit{Proof is given in the Appendix~\ref{sub:optimism}}. 
%
These regret upper bounds are similar to those in linear bandits \citep{abbasi-yadkori2011improved,russo2014learning} and linear quadratic control ~\citep{abbasi-yadkori2011regret}, i.e. $\widetilde{\mathcal{O}}(d\sqrt{T})$.
Since linear bandits are special cases of episodic continuous MDPs, when horizon is equal to $H=1$, we observe that our Bayesian regret upper bound recovers~\citep{russo2014learning} 
and our frequentist regret upper bound recovers the bound in~\citep{abbasi-yadkori2011improved}. While our regret upper bounds are order optimal in $T$, and $d$, they have bad dependency in the horizon length $H$. In our future work, we plan to extensively study this problem and provide tight lower and upper bound in terms of $T,d$ and $H$.


\begin{wrapfigure}{r}{0.5\textwidth}
    \begin{center}
    \vspace*{-1.cm}
    \includegraphics[scale=0.3]{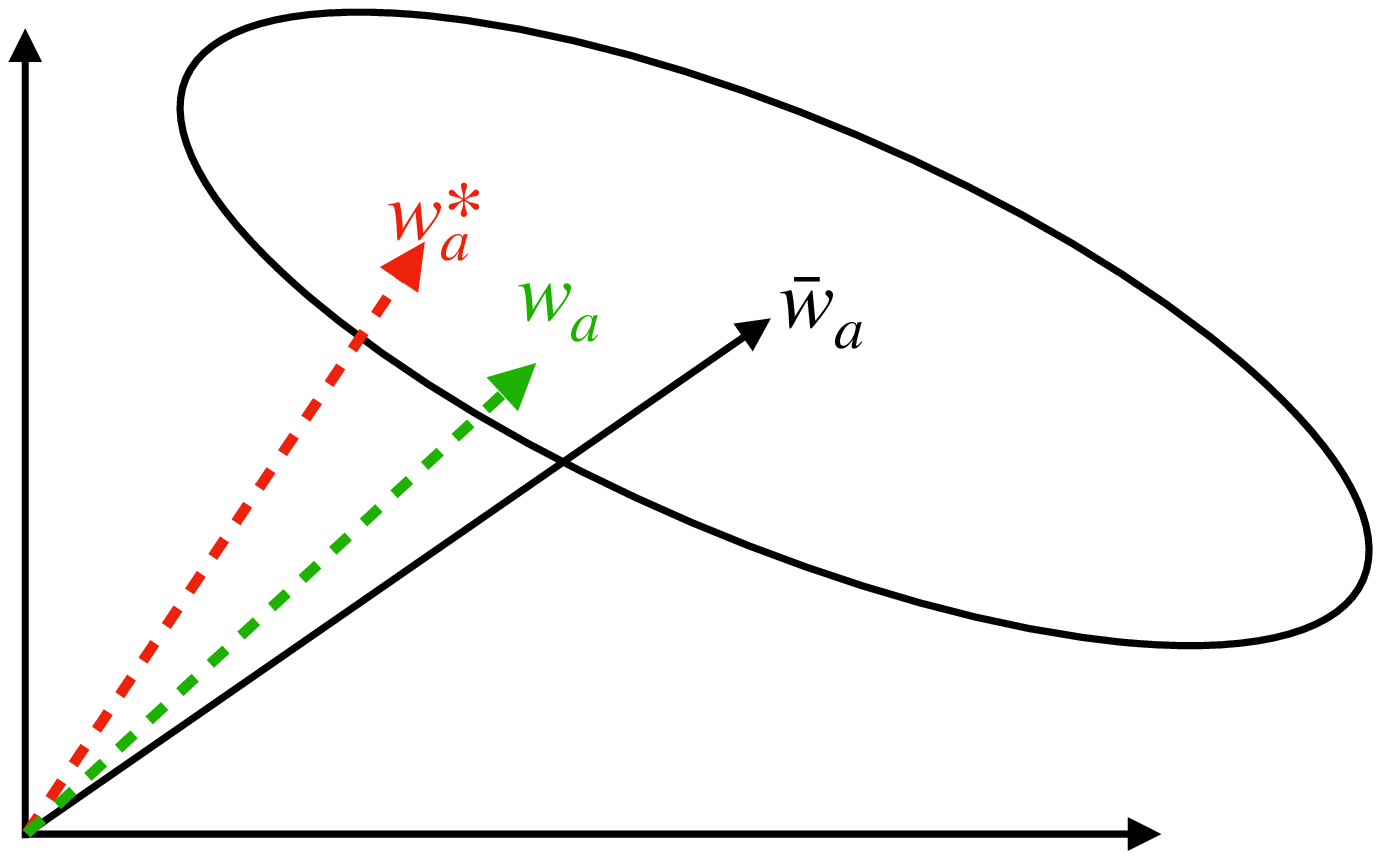}
    \caption{\bdqn deploys Thompson Sampling to $\forall a\in A$ sample $w_a$ (therefore a Q-function) around the empirical mean $\overline{w}_a$ and $w_a^*$ the underlying parameter of interest.}
    \label{fig:omega}
    \vspace*{-.5cm}
    \end{center}
\end{wrapfigure}

\section{Bayesian Deep Q-Networks}
We propose Bayesian deep Q-networks ($\bdqn$) an efficient Thompson sampling based method in high dimensional RL problems. In value based RL, the core of most prominent approaches is to learn the Q-function through minimizing a surrogate to Bellman residual~\citep{schweitzer1985generalized,lagoudakis2003least,antos2008learning} using temporal difference (TD) update \citep{tesauro1995temporal}. ~\citet{van2016deep} carries this idea, and propose \DDQN (similar to its predecessor \DQN~\citep{mnih2015human}) where the Q-function is parameterized by a deep network. \DDQN employ a target network $Q^{target}$, target value $y=r+\gamma Q^{target}(x',\hat{a})$, where the tuple $(x,a,r,x')$ are consecutive experiences, $\hat{a} = \argmax_{a'}Q(x',a')$. \DDQN learns the Q function by approaching the empirical estimates of the following regression problem:
\begin{align}\label{eq:reg}
\L(Q,Q^{target}) = \E\left[\left(Q(x,a)-y\right)^2\right]
\end{align}
The \DDQN agent,  once in a while, updates the $Q^{target}$ network by setting it to the Q network, and follows the regression in Eq.\ref{eq:reg} with the new target value. Since we aim to empirically study the effect of Thompson sampling, we directly mimic the \DDQN to design \bdqn.
%

\textbf{Linear Representation}: \DDQN architecture consists of a deep neural network where the Q-value is a linear function of the feature representation layer (output layer) of the Q-network, i.e., $\phi_{\theta}(x)\in\Re^d$ parameterized by $\theta$. Therefore, for any $x,a$, $Q(x,a) = {\phi_{\theta}(x)}^\top w_a$ with $w_a\in\R^d$, the parameter of the last linear layer. 
Similarly, the target model has the same architecture, and consists of $\phi_{\theta^{target}}(\cdot)$, the
feature representation of the target network, and ${w^{target}}_{{a}},~\forall {a}\in\A$ the target weight. Given a tuple $(x,a,r,x')$ and $\hat{a}=\mathrm{arg\,max}_{a'} {\phi_{\theta}}^\top w_{a'}$, \DDQN learns $w_a$'s and $\theta$ to match $y$:
\begin{align*}
Q(x,a) = {\phi_{\theta}(x)}^\top w_a\rightarrow y := r + \gamma \phi_{\theta^{target}}(x')^\top{w^{target}}_{\hat{a}}
\end{align*}

In \DDQN, we match ${\phi_{\theta}(x)}^\top w_a$ to $y$ using the regression in Eq.~\ref{eq:reg}. This regression problem results in a linear regression in the last layer, $w_a$'s. \bdqn follows all \DDQN steps except for the learning of the last layer $w_a$'s. \bdqn deploys Gaussian BLR instead of the plain linear regression, resulting in an approximated posterior on the $w_a$'s and consequently on the Q-function. As discussed before, BLR with Gaussian prior and likelihood is an approximation to \LinPSRL and \LinUCB~\citep{abeille2017linear}.
Through BLR, we efficiently approximate the distribution over the Q-values, capture the uncertainty over the Q estimates, and design an efficient exploration-exploitation strategy using Thompson Sampling.

Given a experience replay buffer $\D=\lbrace x_\tau,a_\tau,y_\tau\rbrace_{\tau=1}^D$, for each action $a$ we construct a data set $\D_a$ with $a_\tau =a$, then construct a matrix $\Phi_a^\theta\in\Re^{d\times |\D_a|}$, the concatenation of feature vectors $\lbrace\phi_{\theta}(x_i)\rbrace_{i=1}^{|\D_a|}$, and $\y_a\in\Re^{|\D_a|}$, the concatenation of target values in set $\D_a$. We then approximate the posterior distribution of $w_a$ as follows:
\begin{align}\label{eq:w}
&\wb w_a:=\frac{1}{\sigma^2_\epsilon}Cov_a\Phi_a^\theta\y_a,~~ Cov_a := \left(\frac{1}{\sigma_\epsilon^2}\Phi_a^\theta{\Phi_a^\theta}^\top+\frac{1}{\sigma^2}I \right)^{-1} \rightarrow \textit{sampling}~w_a\sim \N\left(\wb w_a,Cov_a\right)
\end{align}
which is the derivation of well-known BLR, with mean zero prior as well as $\sigma$ and $\sigma_\epsilon$, variance of prior and likelihood respectively.
Fig.~\ref{fig:omega} demonstrate the mean and covariance of the over $w_a$ for each action $a$. A \bdqn agent deploys Thompson sampling on the approximated posteriors every $T^{S}$ to balance exploration and exploitation while updating the posterior every $T^{BT}$.



\begin{algorithm}[t]
\caption{\bdqn}
\begin{algorithmic}[1]
\STATE Initialize $\theta$, $\theta^{target}$, and $\forall a$, $w_a$, $w^{target}_a$, $Cov_a$ 
\STATE Set the replay buffer $RB=\lbrace\rbrace$
\FOR{$t$ = 1,2,3\ldots}
\IF{$t$ mod $T^{BT}=0$  }
\STATE $\forall a$, update $\wb w_a$ and $Cov_a$, $\forall a$
\ENDIF
\IF{$t$ mod $T^{S}=0$ }
\STATE Draw $w_a\sim\mathcal{N}\left(w_a^{target},Cov_a\right)~\forall a$
\ENDIF
	\STATE Set $\theta^{target}\leftarrow\theta$ every $T^{T}$
\STATE Execute $a_t = \mathrm{arg\,max}_{a'} w_{a'}^{\top}\phi_{\theta}(x_{t})$
	\STATE Store $(x_t, a_t, r_t, x_{t+1})$ in  the $RB$
  \STATE Sample a minibatch ($x_{\tau}$ , $a_{\tau}$, $r_{\tau}$, $x_{\tau+1}$) from the $RB$
\STATE $y_{\tau} \gets  \left\{\begin{array}{ll}
    r_{\tau}&\text{terminal  } x_{\tau+1}\\
  r_\tau  +  {w_{\hat{a}}^{target}}^\top\phi_{\theta^{target}}(x_{\tau+1}),~\hat{a}:=\mathrm{arg\,max}_{a'} w_{a'}^{\top}\phi_{\theta}(x_{\tau+1}) &\text{non-terminal } x_{\tau+1}  \\
\end{array}   \right\} $ 
\STATE Update $\theta \gets \theta - \alpha \cdot \nabla_{\theta} (y_\tau -  w_{a_\tau} ^{\top}\phi_{\theta}(x_{\tau}) )^2$ 
\ENDFOR
\end{algorithmic}
\label{alg:bdqn}
\end{algorithm}

 \section{Experiments} \label{sec:Exp}
We empirically study \bdqn behaviour on a variety of Atari games in the Arcade Learning Environment \citep{bellemare2013arcade} using OpenAI Gym~\citep{1606.01540}, on the measures of sample complexity and score against \DDQN,~Fig~\ref{fig:exp}, Table~\ref{table:scores}.  All the codes\footnote{https://github.com/kazizzad/BDQN-MxNet-Gluon, Double-DQN-MxNet-Gluon, DQN-MxNet-Gluon}, with detailed comments and explanations are publicly available and programmed in MxNet \citep{chen2015mxnet}.

We implemented DDQN and \bdqn following \citet{van2016deep}. We also attempted to implement a few other deep RL methods that employ strategic exploration (with advice from their authors), e.g., \citep{osband2016deep,bellemare2016unifying}. Unfortunately  we encountered several implementation challenges that we could not address since neither the codes nor the implementation details are publicly available (we were not able to reproduce their results beyond the performance of random policy). Along with \bdqn and \DDQN codes, we also made our implementation of \citet{osband2016deep} publicly available. In order to illustrate the \bdqn performance we report its scores along with a number of state-of-the-art deep RL methods \ref{table:scores}. For some games, e.g., Pong, we ran the experiment for a longer period but just plotted the beginning of it in order to observe the difference. Due to huge cost of deep RL methods, for some games, we run the experiment until a plateau is reached. The \bdqn and \DDQN columns are scores after running them for number steps reported in \textit{Step} column without Since the regret is considered, no evaluation phase designed for them. $\DDQN^+$ is the reported scores of \DDQN in~\citet{van2016deep} at evaluation time where the $\varepsilon=0.001$. We also report scores of Bootstrap \DQN \citep{osband2016deep}, NoisyNet \citep{fortunato2017noisy}, CTS, Pixel, Reactor \citep{ostrovski2017count}. For NoisyNet, the scores of NoisyDQN are reported. To illustrate the sample complexity behavior of \bdqn we report \textit{SC}: the number of interactions \bdqn requires to beat the human score \citep{mnih2015human}($``-"$ means \bdqn could not beat human score), and $\textit{SC}^+$: the number of interactions the \bdqn requires to beat the score of $\DDQN^+$. Note that Table~\ref{table:scores} does not aim to compare different methods. Additionally, there are many additional details that are not included in the mentioned papers which can significantly change the algorithms behaviors~\citep{henderson2017deep}), e.g., the reported scores of \DDQN in \citet{osband2016deep} are significantly higher than the reported scores in the original \DDQN paper, indicating many existing non-addressed advancements (Appendix~\ref{sub:reproducability}). 

We also implemented \DDQN drop-out a Thomson Sampling based algorithm motivated by \citet{gal2016dropout}. We observed that it is not capable of capturing the statistical uncertainty in the $Q$ function and falls short in outperforming a random(uniform) policy. 
\citet{osband2016deep} investigates the sufficiency of the estimated uncertainty and hardness in driving suitable exploitation out of it. It has been observed that drop-out results in the ensemble of infinitely many models but all models almost the same~\citep{dhillon2018stochastic,osband2016deep} Appendix~\ref{sub:dropout}.

As mentioned before, due to an efficient exploration-exploitation strategy, not only \bdqn improves the regret and enhance the sample complexity, but also reaches significantly higher scores. In contrast to naive exploration, \bdqn assigns less priority to explore actions that are already observed to be not worthy, resulting in \textit{better sample complexity}. Moreover, since \bdqn does not commit to adverse actions, it does not waste the model capacity to estimate the value of unnecessary actions in unnecessary states as good as the important ones, resulting in \textit{saving the model capacity} and better policies.

For the game \textit{Atlantis}, $\DDQN^+$ reaches score of $64.67k$ during the evaluation phase, while \bdqn reaches score of $3.24M$ after $20M$ time steps. After multiple run of \bdqn, we constantly observed that its performance suddenly improves to around $3M$ in the vicinity of $20M$ time steps. We closely investigate this behaviour and realized that \bdqn saturates the \textit{Atlantis} game and reaches the internal \textit{OpenAIGym} limit of $max\_episode$. After removing this limit, \bdqn  reaches score $62M$ after $15M$. Please refer to Appendix~\ref{apx:empirical} for the extensive empirical study.

\setlength{\tabcolsep}{1pt}
\begin{table*}[t]
  \centering
  \caption{
    Comparison of scores and sample complexities (scores in the first two columns are average of 100 consecutive episodes). The scores of $\DDQN^+$ are the reported scores of \DDQN in~\citet{van2016deep} after running it for 200M interactions at evaluation time where the $\varepsilon=0.001$. Bootstrap \DQN \citep{osband2016deep}, CTS, Pixel, Reactor \citep{ostrovski2017count} are borrowed from the original papers. For NoisyNet \citep{fortunato2017noisy}, the scores of NoisyDQN are reported. Sample complexity, \textit{SC}: the number of samples the \bdqn requires to beat the human score \citep{mnih2015human}($``-"$ means \bdqn could not beat human score). $\textit{SC}^+$: the number of interactions the \bdqn requires to beat the score of $\DDQN^+$. }
  \footnotesize
  \hspace*{-.2cm}
  \begin{tabular}{l|ccccccccc|c|c|c}
%
\toprule
 Game  & \bdqn & \DDQN & $\DDQN^+$ & Bootstrap  & NoisyNet & CTS & Pixel  & Reactor & Human & SC & $SC^+$&Step  \\
\midrule
Amidar & \textbf{5.52k}& 0.99k & 0.7k & 1.27k & 1.5k  & 1.03k & 0.62k  & 1.18k & 1.7k& 22.9M &4.4M & 100M\\ 
Alien & 3k & 2.9k& 2.9k & 2.44k & 2.9k  & 1.9k &   1.7k  & \textbf{3.5k}&6.9k& - &36.27M & 100M\\ 
Assault & \textbf{8.84k} &2.23k&  5.02k &8.05k & 3.1k & 2.88k &  1.25k & 3.5k&1.5k &1.6M & 24.3M & 100M\\ 
Asteroids  & \textbf{14.1k} &0.56k& 0.93k & 1.03k &  2.1k & 3.95k & 0.9k & 1.75k&13.1k &58.2M  & 9.7M & 100M\\ 
Asterix  & \textbf{58.4k} & 11k& 15.15k &19.7k &  11.0 & 9.55k & 1.4k &  6.2k &8.5k& 3.6M  & 5.7M & 100M\\ 
BeamRider  & 8.7k & 4.2k& 7.6k & \textbf{23.4k} & 14.7k & 7.0k & 3k & 3.8k& 5.8k& 4.0M  & 8.1M  & 70M\\ 
BattleZone  & \textbf{65.2k} & 23.2k& 24.7k & 36.7k & 11.9k & 7.97k & 10k &  45k& 38k &25.1M  & 14.9M  & 50M\\ 
Atlantis  & 3.24M &39.7k& 64.76k& 99.4k & 7.9k & 1.8M & 40k &   \textbf{9.5M}& 29k& 3.3M  & 5.1M  & 40M\\ 
DemonAttack  & 11.1k & 3.8k& 9.7k& 82.6k &26.7k & \textbf{39.3k} &  1.3k & 7k& 3.4k & 2.0M  & 19.9M &40M\\ 
Centipede  & \textbf{7.3k} & 6.4k& 4.1k & 4.55k & 3.35k & 5.4k &  1.8k  & 3.5k &12k& -  & 4.2M & 40M\\ 
BankHeist  & 0.72k &0.34k&  0.72k &\textbf{1.21k} & 0.64k & 1.3k &  0.42k & 1.1k & 0.72k &2.1M  & 10.1M & 40M\\ 
CrazyClimber  & 124k & 84k& 102k &\textbf{138k} & 121k & 112.9k & 75k & 119k & 35.4k&0.12M  & 2.1M &40M\\
ChopperCmd & \textbf{72.5k} & 0.5k&  4.6k & 4.1k &5.3k & 5.1k & 2.5k  & 4.8k&9.9k& 4.4M  &2.2M  &40M\\ 
Enduro  & 1.12k &0.38k& 0.32k &1.59k &  0.91k & 0.69k & 0.19k & \textbf{2.49k} & 0.31k& 0.82M  & 0.8M& 30M\\ 
Pong  & \textbf{21} &18.82&  \textbf{21} &20.9 & \textbf{21}  & 20.8  & 17 & 20& 9.3 &1.2M  &2.4M  & 5M\\ 
    \bottomrule
  \end{tabular}
  \label{table:scores}
\end{table*}

 \section{Related Work}

The complexity of the exploration-exploitation trade-off has been deeply investigated in RL literature for both continuous and discrete MDPs~\citep{kearns2002near,brafman2003r,asmuth2009bayesian,kakade2003exploration,ortner2012online,osband2014model,osband2014near}.~\citet{jaksch2010near} investigate the regret analysis of MDPs with finite state and action and deploy \OFU~\citep{auer2003using} to guarantee a regret upper bound, while \citet{ortner2012online} relaxes it to a continuous state space and propose a sub-linear regret bound. 
\citet{azizzadenesheli2016reinforcement} deploys \OFU~and propose a regret upper bound for Partially Observable MDPs (POMDPs) using spectral methods \citep{anandkumar2014tensor}. Furthermore, \citet{bartok2014partial} tackles a general case of partial monitoring games and provides minimax regret guarantee. For linear quadratic models \OFU~ is deployed to provide an optimal regret bound \citep{abbasi-yadkori2011regret}. In multi-arm bandit, Thompson sampling have been studied both from empirical and theoretical point of views~\citep{chapelle2011empirical,agrawal2012analysis,russo2014learninginfo}. A natural adaptation of this algorithm to RL, posterior sampling RL (PSRL)~\citet{strens2000bayesian} also shown to have good frequentist and Bayesian performance guarantees ~\citep{osband2013more,abbasi2015bayesian}. Inevitably  for PSRL, these methods also have hard time to become scalable to high dimensional problems, \citep{ghavamzadeh2015bayesian,engel2003bayes,dearden1998bayesian,tziortziotis2013linear}.

Exploration-exploitation trade-offs has been theoretically studied in
RL but a prominent problem in high dimensional environments~\citep{mnih2015human,abel2016exploratory,azizzadenesheli2016rich}. Recent success of Deep RL  on Atari games~\citep{mnih2015human}, the board game Go \citep{silver2017mastering}, robotics \citep{levine2016end}, self-driving cars \citep{shalev2016safe}, and safety in RL \citep{lipton2016combating} propose promises on deploying deep RL in high dimensional problem.

To extend the exploration-exploitation efficient methods to high dimensional RL problems,~\citet{osband2016deep} suggest bootstrapped-ensemble approach that trains several models in parallel to approximate the posterior distribution. \citet{bellemare2016unifying} propose a way to come up with a surrogate to optimism in high dimensional RL. Other works suggest using a variational approximation to the Q-networks~\citep{lipton2016efficient} or a concurrent work on noisy network \citep{fortunato2017noisy} suggest to randomize the Q-network. 
However, most of these approaches significantly increase the computational cost of DQN,  e.g., the bootstrapped-ensemble incurs a  computation overhead that is linear in the number of bootstrap models. 

Concurrently, \citet{levine2017shallow} proposes least-squares temporal difference which learns a linear model on the feature representation in order to estimate the Q-function. They use $\varepsilon$-greedy approach and provide results on five Atari games. Out of these five games, one is common with our set of 15 games which \bdqn outperforms it by a factor of $360\%$ (w.r.t. the score reported in their paper). As also suggested by our theoretical derivation, our empirical study illustrates that performing Bayesian regression instead, and sampling from the result yields a substantial benefit. This indicates that it is not just the higher data efficiency at the last layer, but that leveraging an explicit uncertainty representation over the value function is of substantial benefit.

\section{Conclusion}
In this work, we proposed \LinPSRL and \LinUCB, two LinReL algorithms for continuous MDPs. We then proposed \bdqn, a deep RL extension of these methods to high dimensional environments. \bdqn deploys Thompson sampling and provides an efficient exploration/exploitation in a computationally efficient manner. It involved making simple modifications to the \DDQN architecture by replacing the linear regression learning of the last layer with Bayesian linear regression. We demonstrated significantly improvement training, convergence, and regret along with much better performance in many games.

While our current regret upper bounds seem to be sub-optimal in terms of $H$ (we are not aware of any tight lower bound), in the future, we plan to deploy the analysis in \citep{antos2008learning,lazaric2010finite} and develop a tighter regret upper bounds as well as an information theoretic lower bound. We also plan to extend the analysis in~\citet{abeille2017linear} and develop Thompson sampling methods with a performance guarantee and finally go beyond the linear models~\citep{jiang2016contextual}. While finding optimal continuous action given a Q function can be computationally intractable, we aim to study the relaxation of these approaches in continuous control tasks in the future.



\newpage
 \section*{Acknowledgments}
 The authors would like to thank Emma Brunskill, Zachary C. Lipton, Marlos C. Machado,  Ian Osband, Gergely Neu, Kristy Choi, particularly Akshay Krishnamurthy, and Nan Jiang during  ICLR2019 openreview, for their feedbacks, suggestions, and helps. K. Azizzadenesheli is supported in part by NSF Career Award CCF-1254106 and AFOSR YIP FA9550-15-1-0221. This research has been conducted when the first author was a visiting researcher at Stanford University and Caltech. A. Anandkumar is supported in part by Bren endowed chair, Darpa PAI, and Microsoft, Google, Adobe faculty fellowships, NSF Career Award CCF-1254106, and AFOSR YIP FA9550-15-1-0221. All the experimental study have been done using Caltech AWS credits grant.

\bibliographystyle{apalike}
\bibliography{Master_arxiv}
\newpage
\appendix

\textit{Note}: In the main text, for the sake of brevity, we use "BLR" for both i.i.d. non i.i.d. samples.

\section{Empirical Study}\label{apx:empirical}
We empirically study the behavior of \bdqn along with \DDQN as reported in Fig.~\ref{fig:exp} and Table~\ref{table:scores}. We run each of these algorithms for the number steps mentioned in the \textit{Step} column of Table~\ref{table:scores}. We observe that \bdqn significantly improves the sample complexity over \DDQN and reaches the highest scores of \DDQN in a much fewer number of interactions required by \DDQN. Due to \bdqn{}'s better exploration-exploitation strategy, we expected \bdqn to improve the convergence and enhance the sample complexity, but we further observed a significant improvement in the scores as well. It is worth noting that since the purpose of this study is sample complexity and regret analysis, design of evaluation phase (as it is common for $\varepsilon$-greedy methods) is not relevant. In our experiments, e.g., in the game Pong, \DDQN reaches the score of 18.82 during the learning phase. When we set $\varepsilon$ to a quantity close to zero, \DDQN reaches the score of $21$ while \bdqn just converges to 21 as uncertainty decays.

In addition to the Table~\ref{table:scores}, we also provided the score ratio as well as the area under the performance plot ratio comparisons in Table~\ref{table:percentage}.
\newcommand{\widplot}{1.1}
\newcommand{\widmini}{0.19}
\newcommand{\widmi}{0.1cm}
\begin{figure*}[ht]
\vspace*{-0.0cm}
  \begin{minipage}{\widmini\textwidth}
	 \centering
\includegraphics[width=\widplot\linewidth]{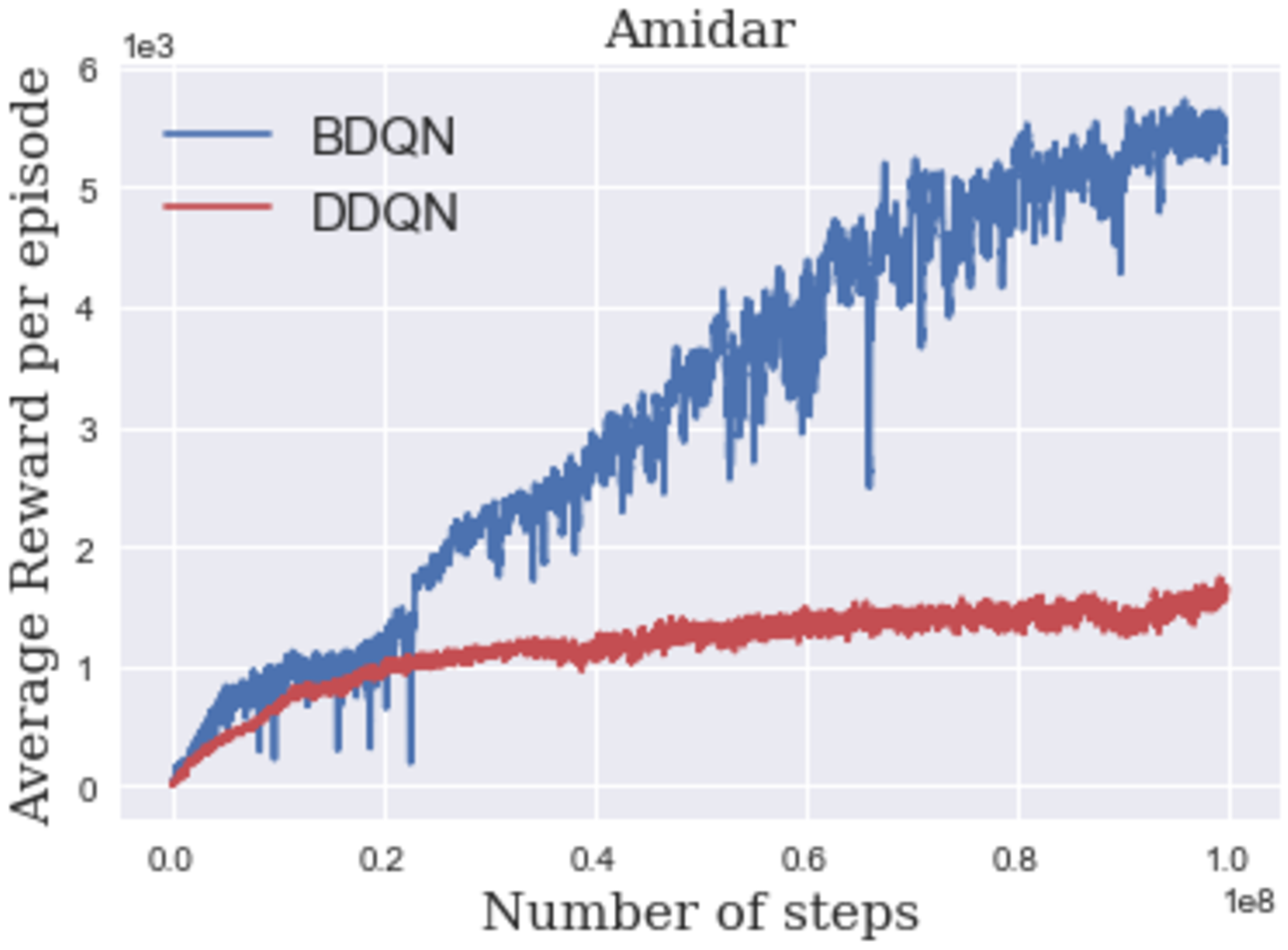}
  \end{minipage}\hfill
  \begin {minipage}{\widmini\textwidth}
     \centering
\includegraphics[width=\widplot\linewidth]{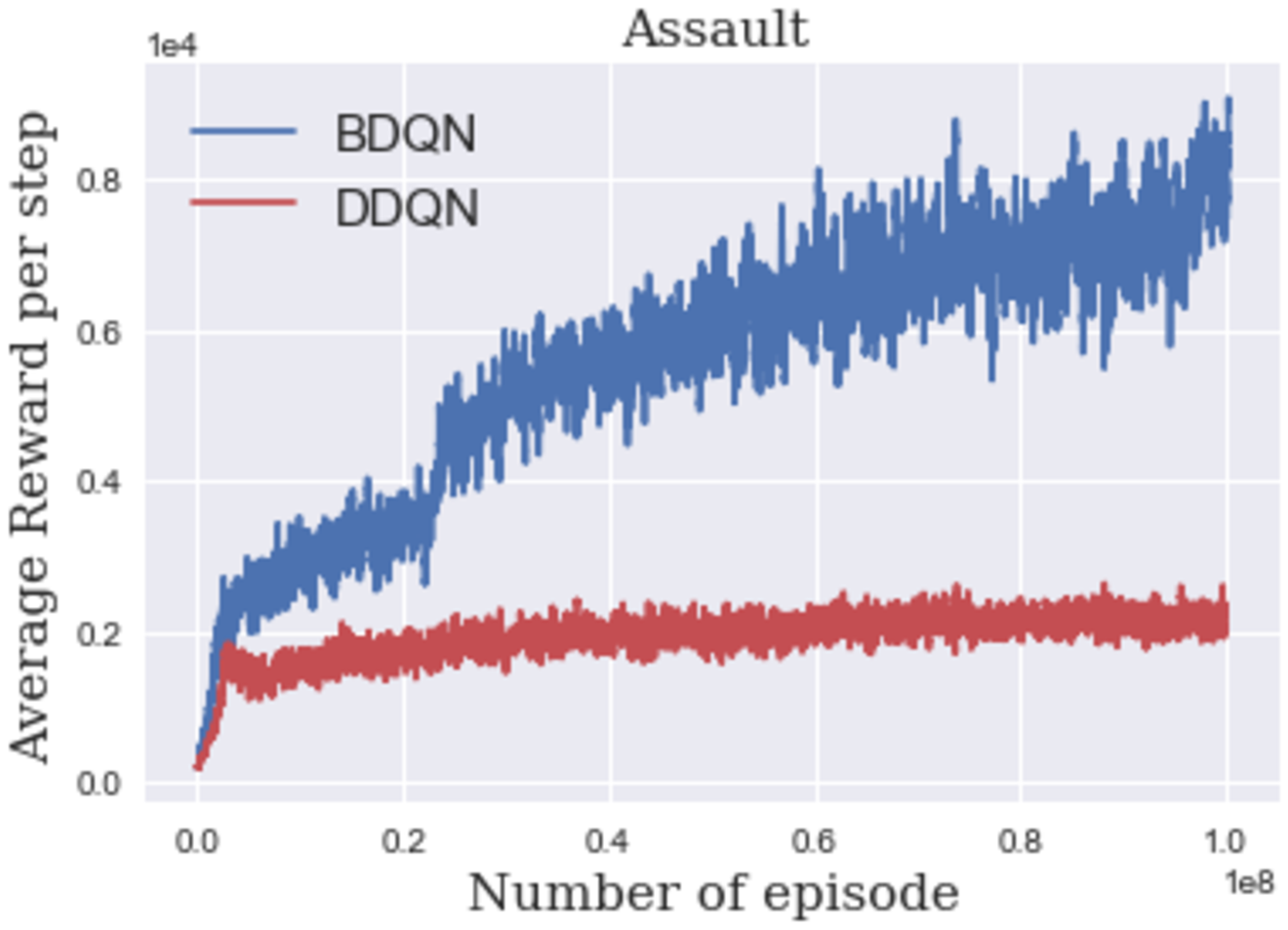}
  \end{minipage}\hfill
  \begin {minipage}{\widmini\textwidth}
     \centering
	 \includegraphics[width=\widplot\linewidth]{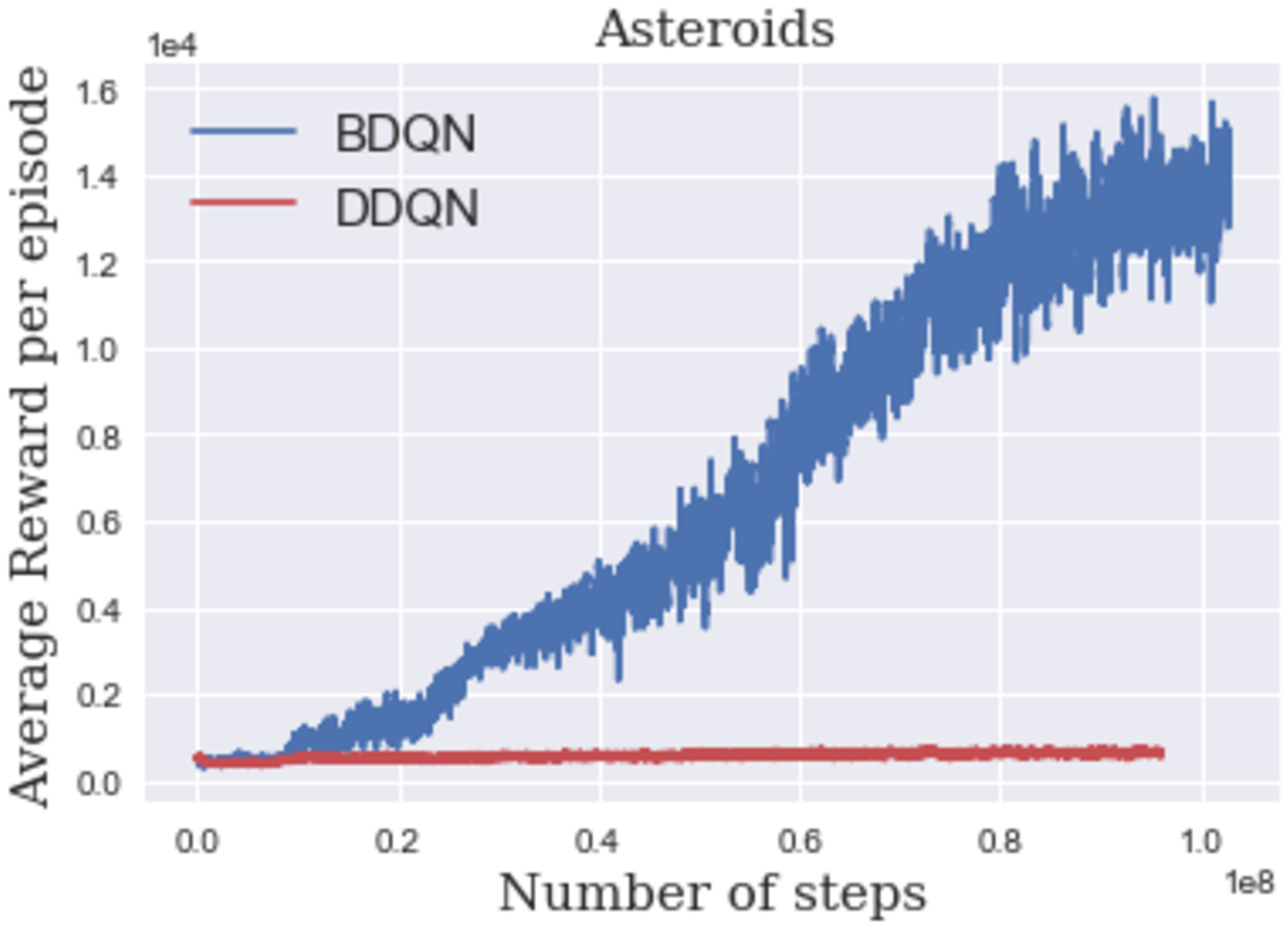}
  \end{minipage}\hfill
  \begin {minipage}{\widmini\textwidth}
     \centering
\includegraphics[width=\widplot\linewidth]{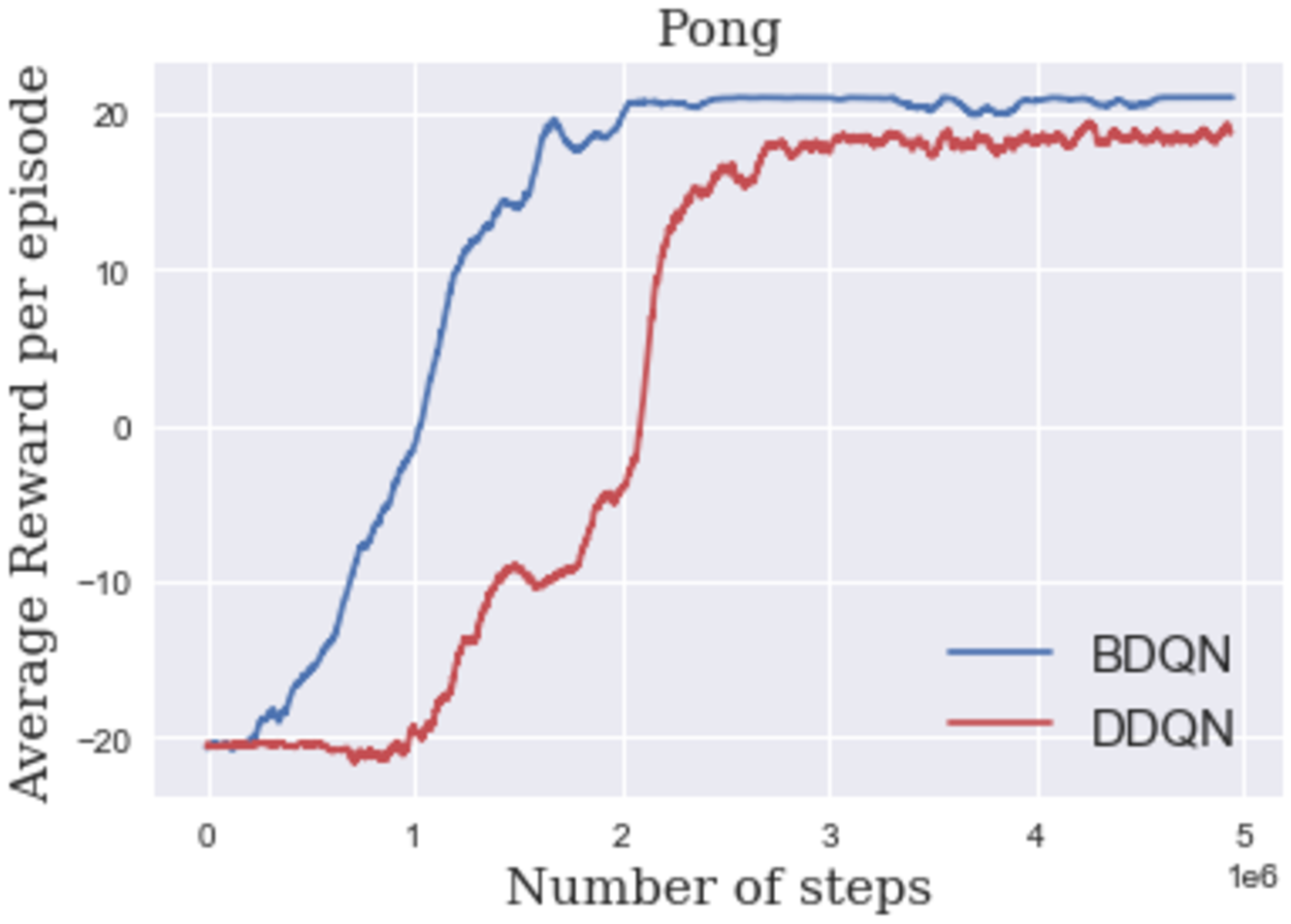}
  \end{minipage}\hfill
  \begin {minipage}{\widmini\textwidth}
     \centering
\includegraphics[width=\widplot\linewidth]{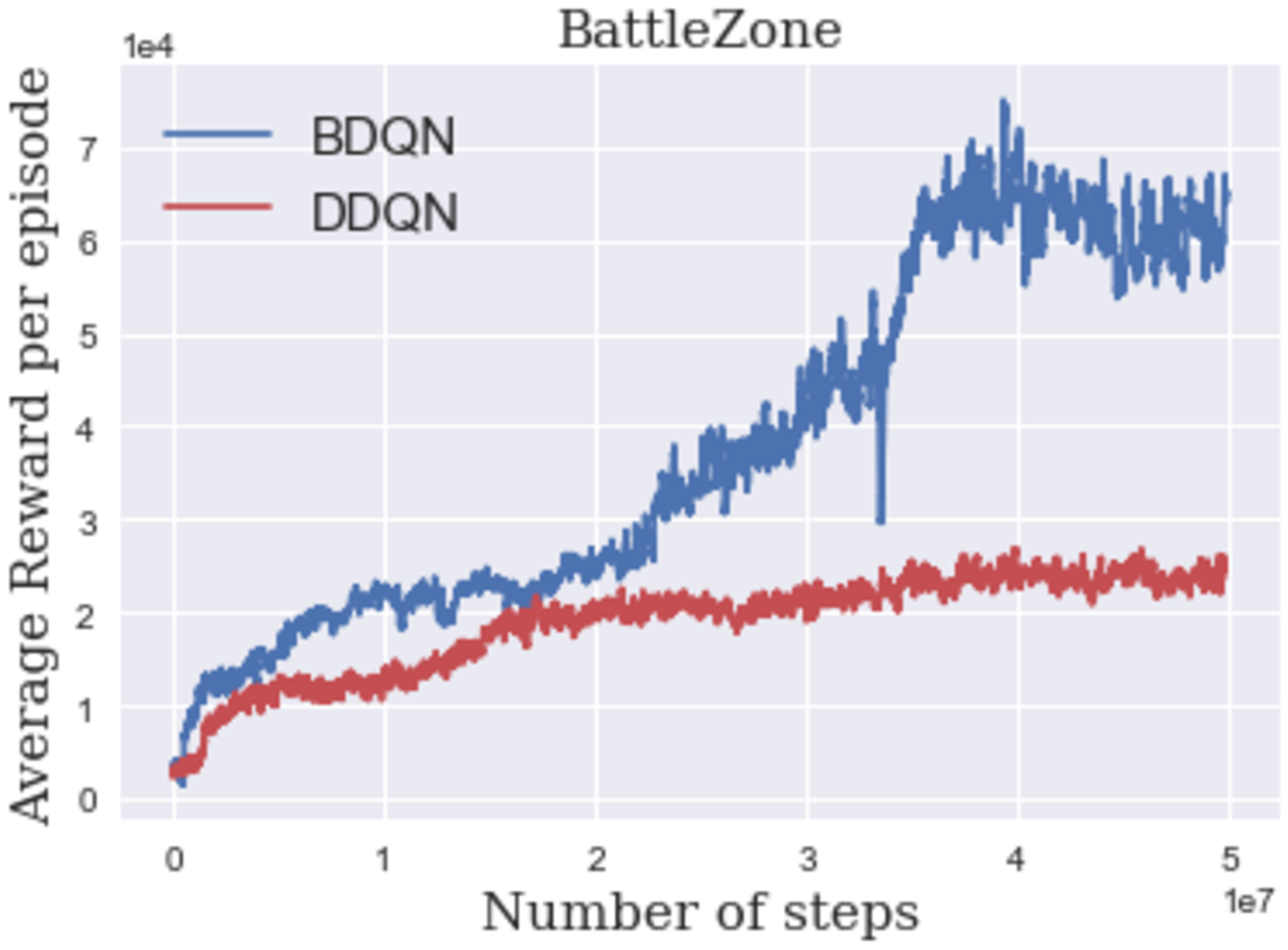}
  \end{minipage}\hfill
  \begin {minipage}{\widmini\textwidth}
     \centering
\includegraphics[width=\widplot\linewidth]{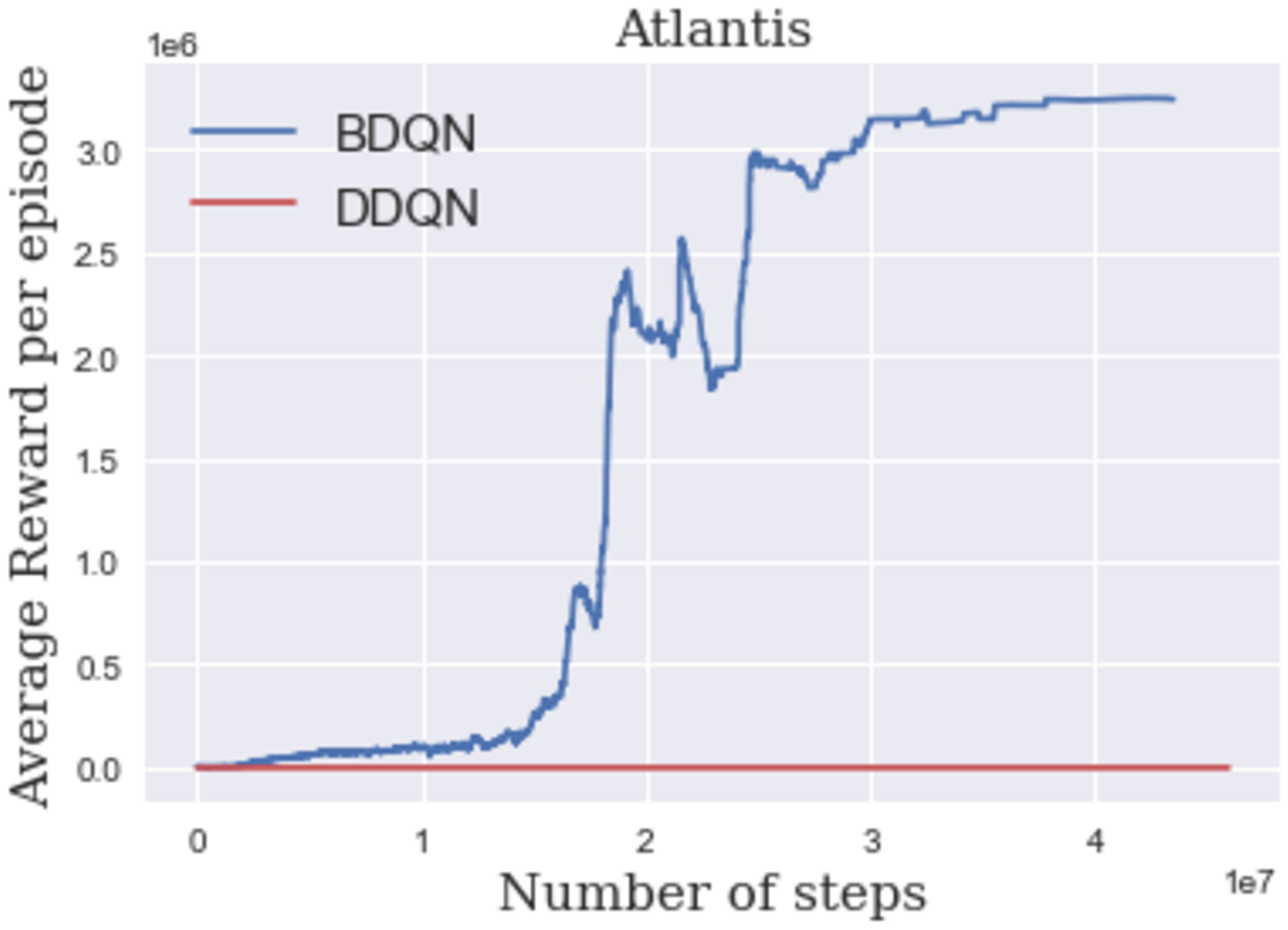}
  \end{minipage}\hfill
  \begin {minipage}{\widmini\textwidth}
     \centering
\includegraphics[width=\widplot\linewidth]{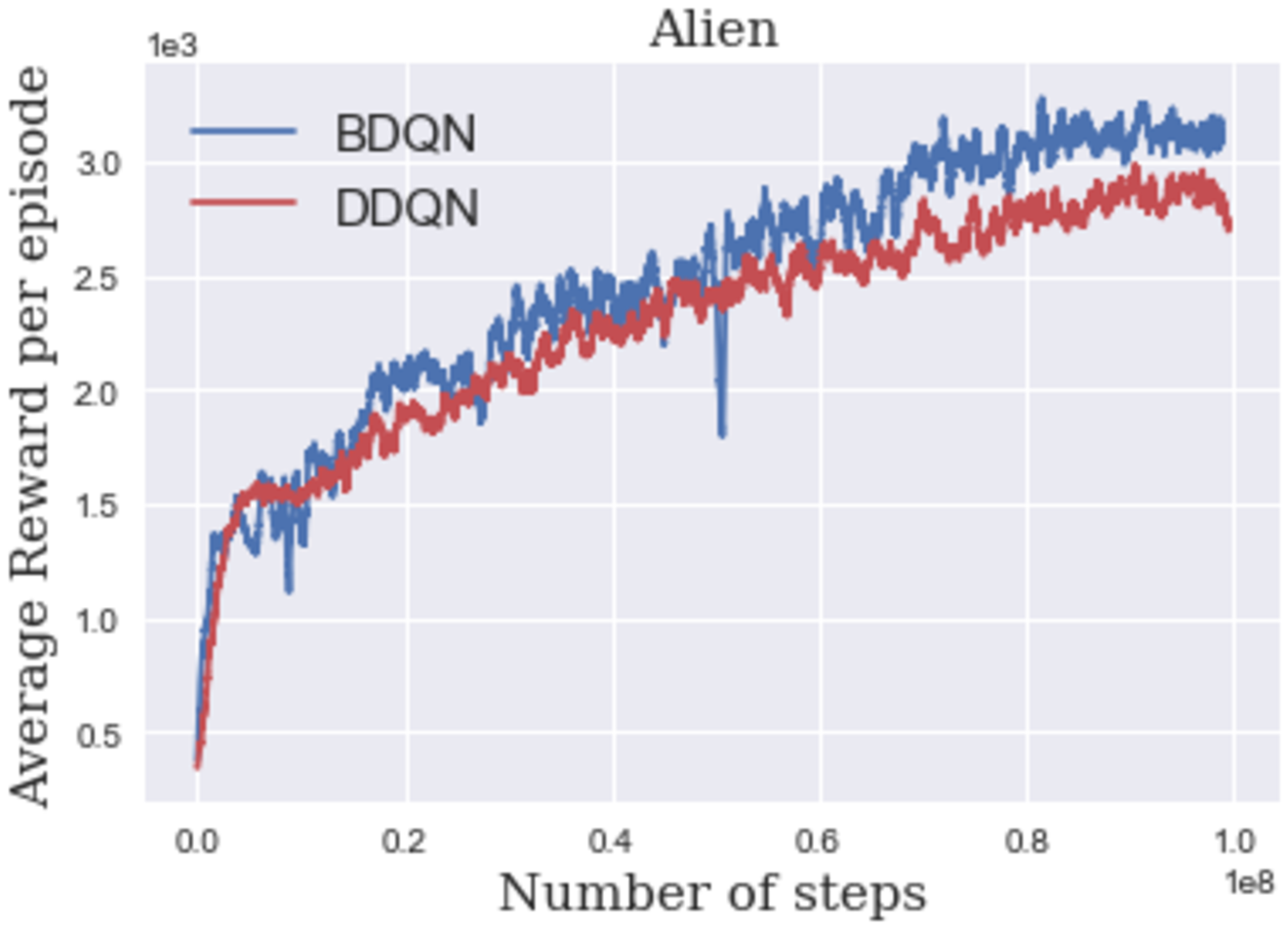}
  \end{minipage}\hfill
  \begin {minipage}{\widmini\textwidth}
     \centering
\includegraphics[width=\widplot\linewidth]{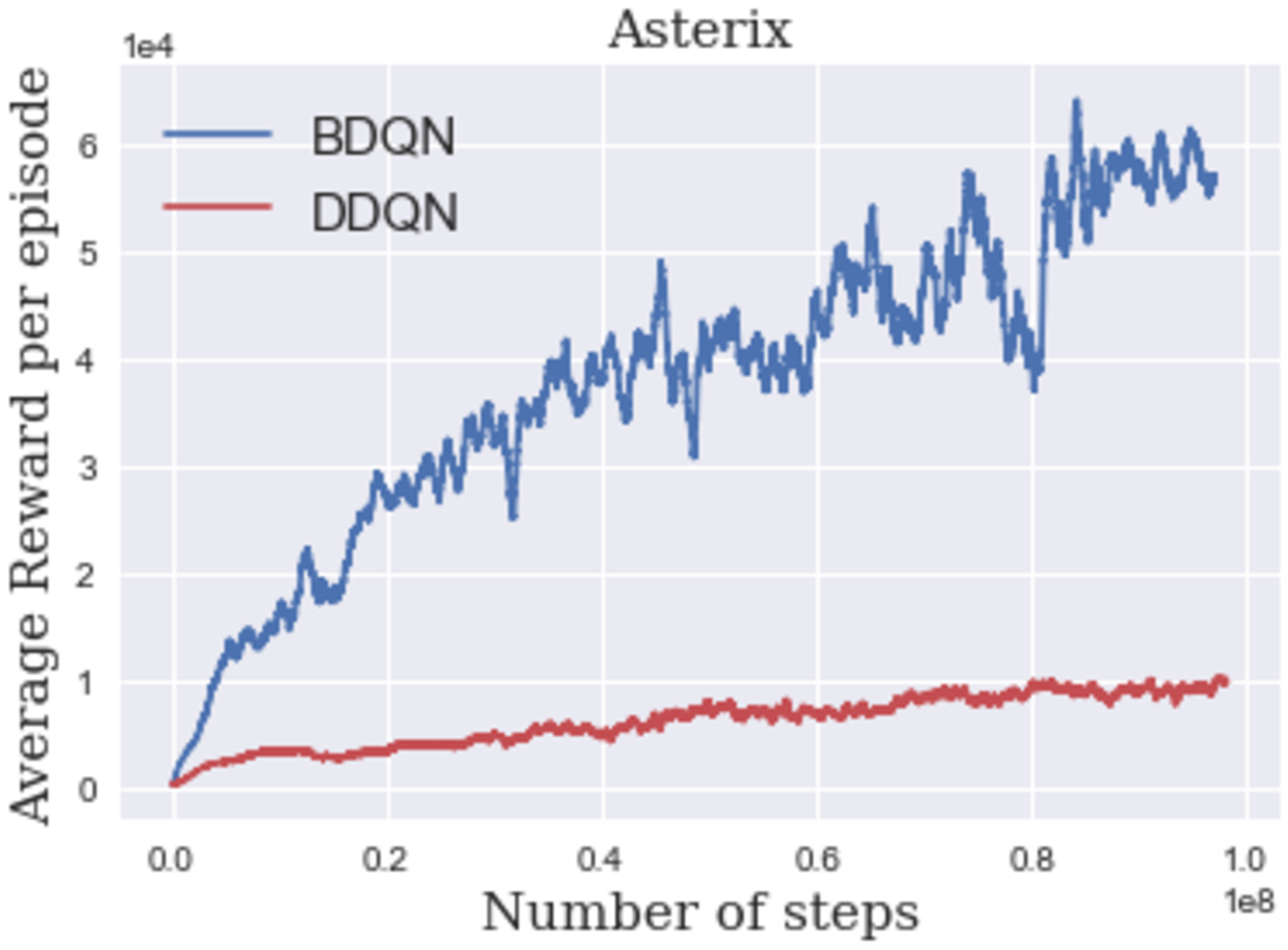}
  \end{minipage}\hfill   
  \begin {minipage}{\widmini\textwidth}
     \centering
\includegraphics[width=\widplot\linewidth]{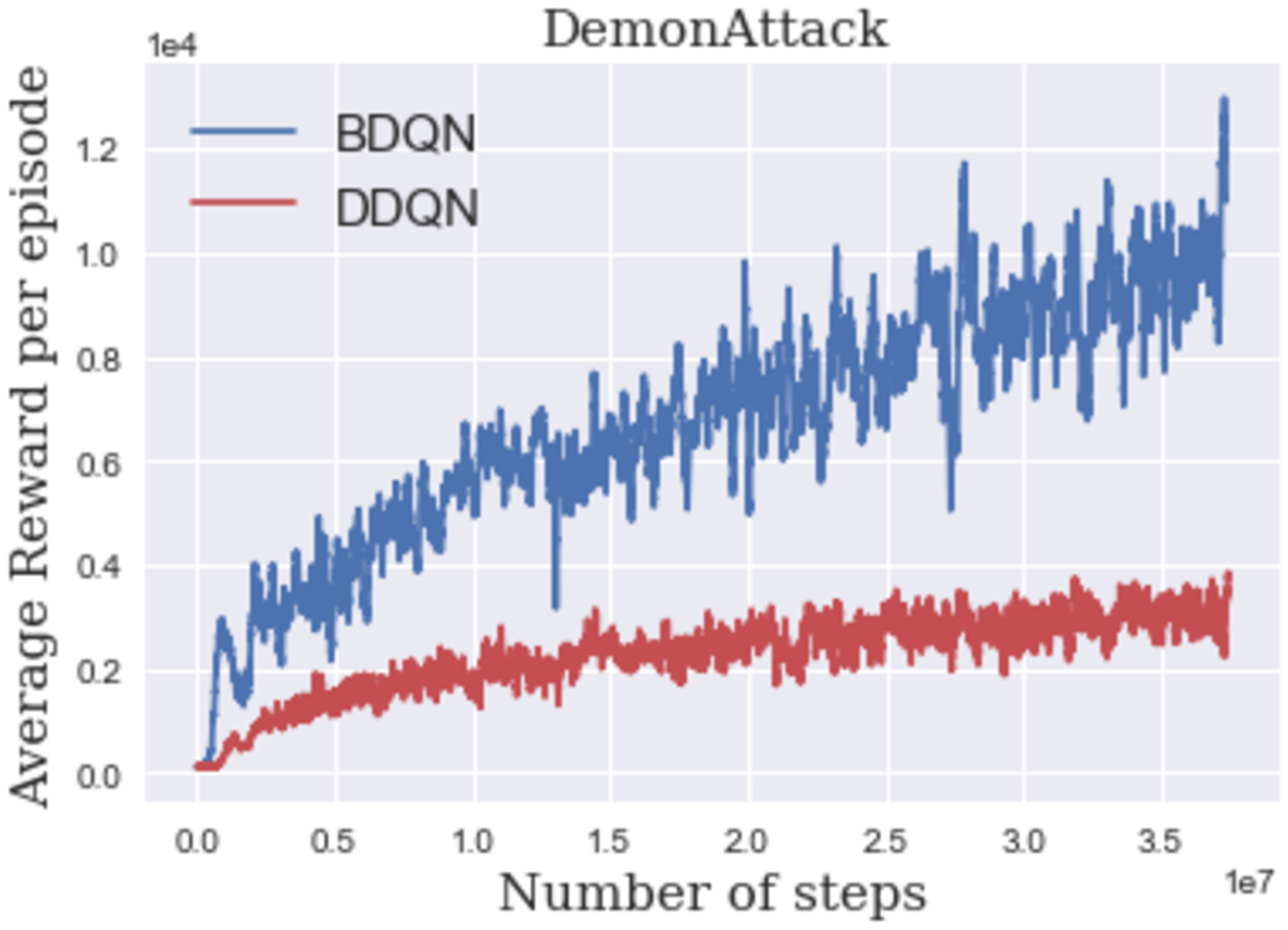}
  \end{minipage}\hfill 
  \begin {minipage}{\widmini\textwidth}
     \centering
\includegraphics[width=\widplot\linewidth]{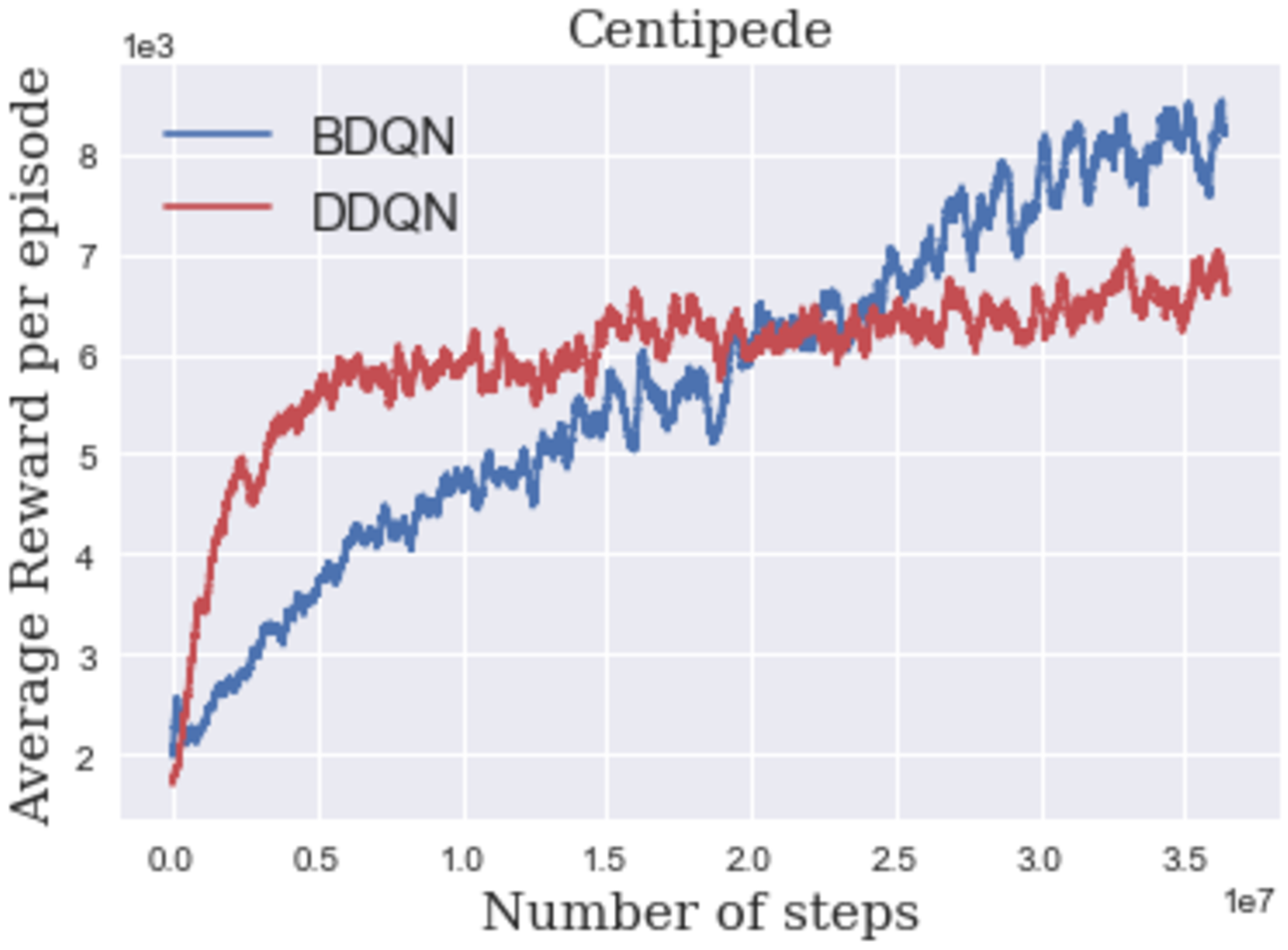}
  \end{minipage}\hfill
  \begin {minipage}{\widmini\textwidth}
     \centering
\includegraphics[width=\widplot\linewidth]{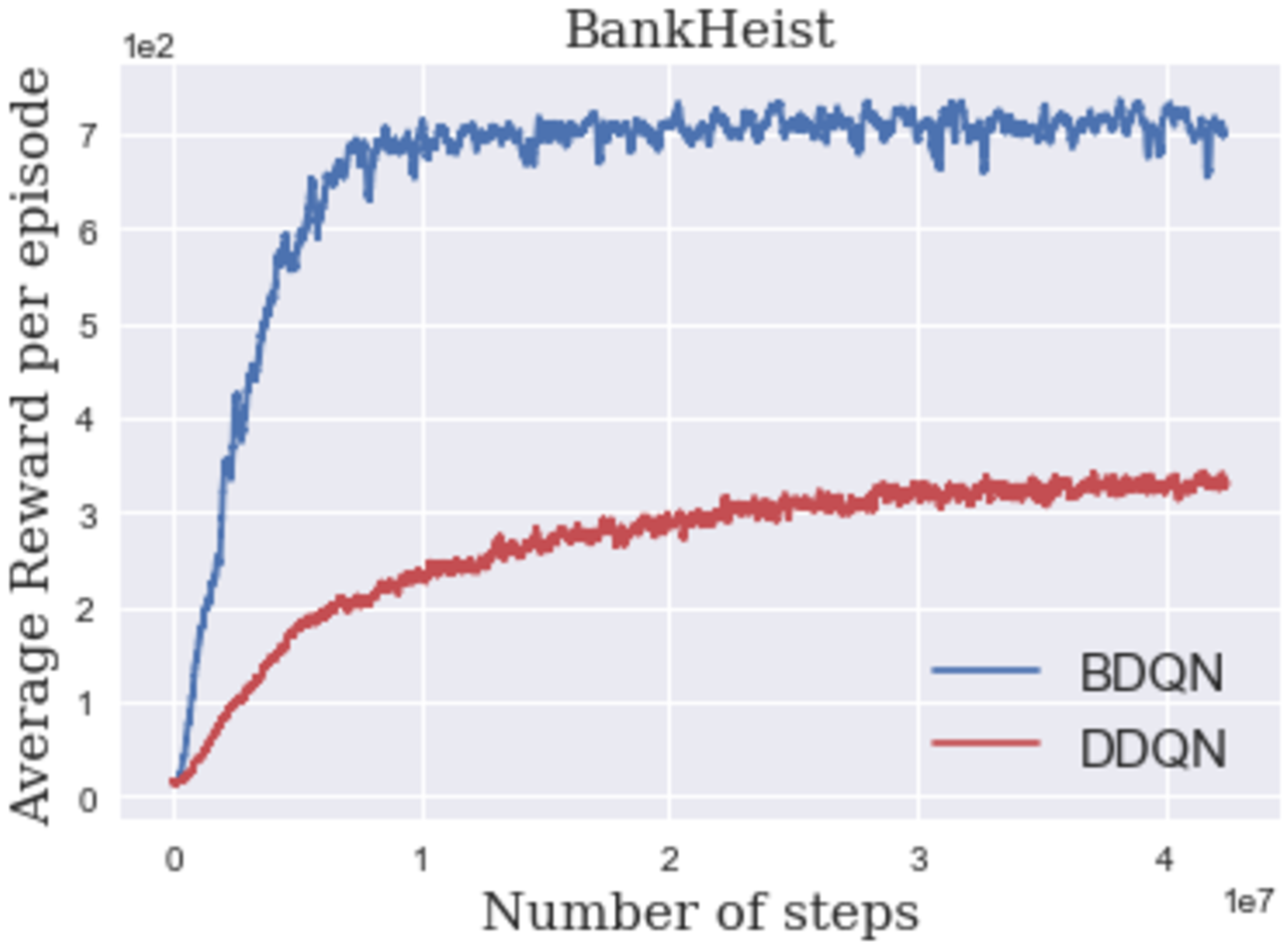}
  \end{minipage}
  \hspace*{\widmi}
  \begin {minipage}{\widmini\textwidth}
     \centering
\includegraphics[width=\widplot\linewidth]{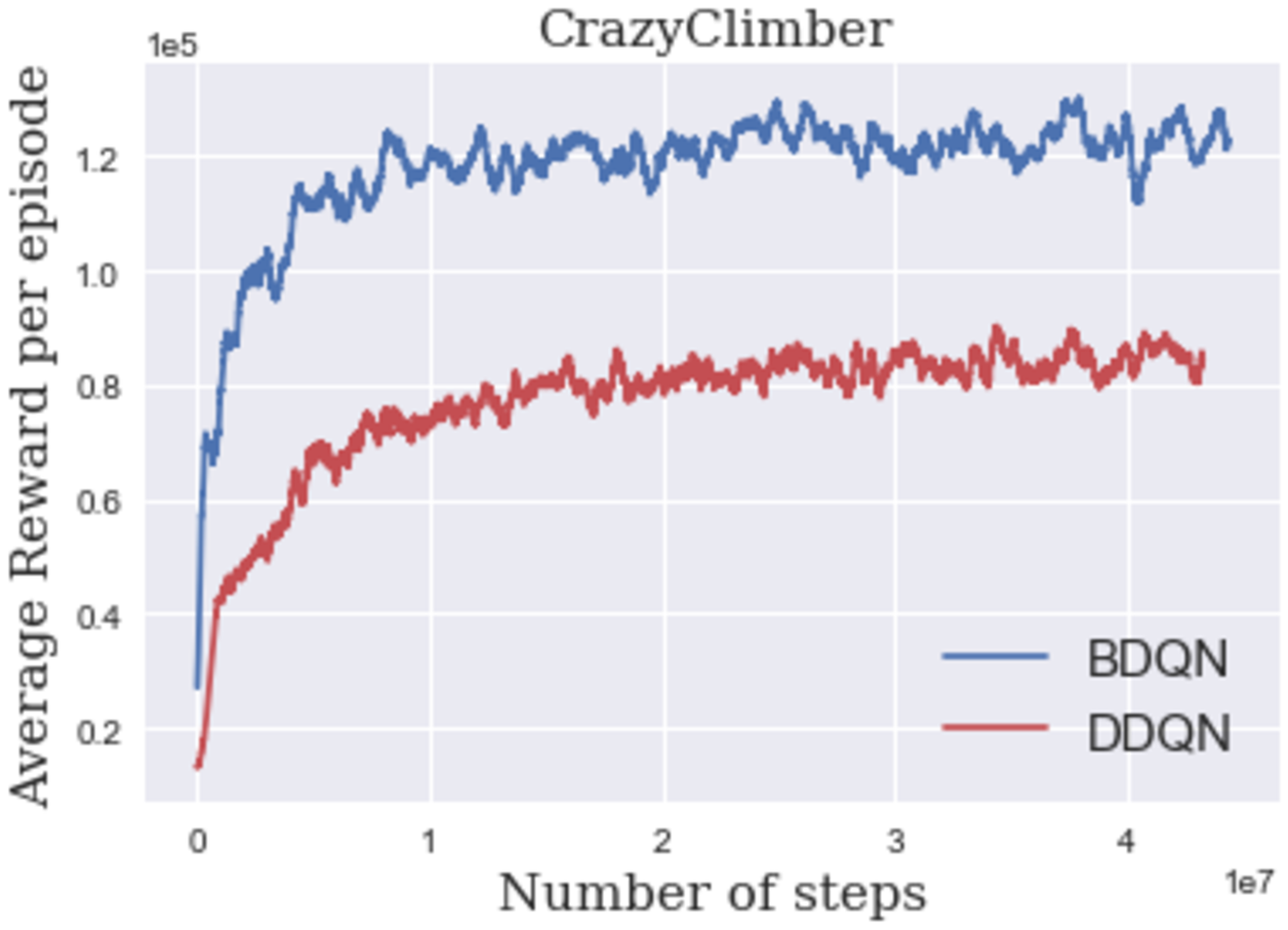}
  \end{minipage}
  \hspace*{\widmi}
  \begin {minipage}{\widmini\textwidth}
     \centering
\includegraphics[width=\widplot\linewidth]{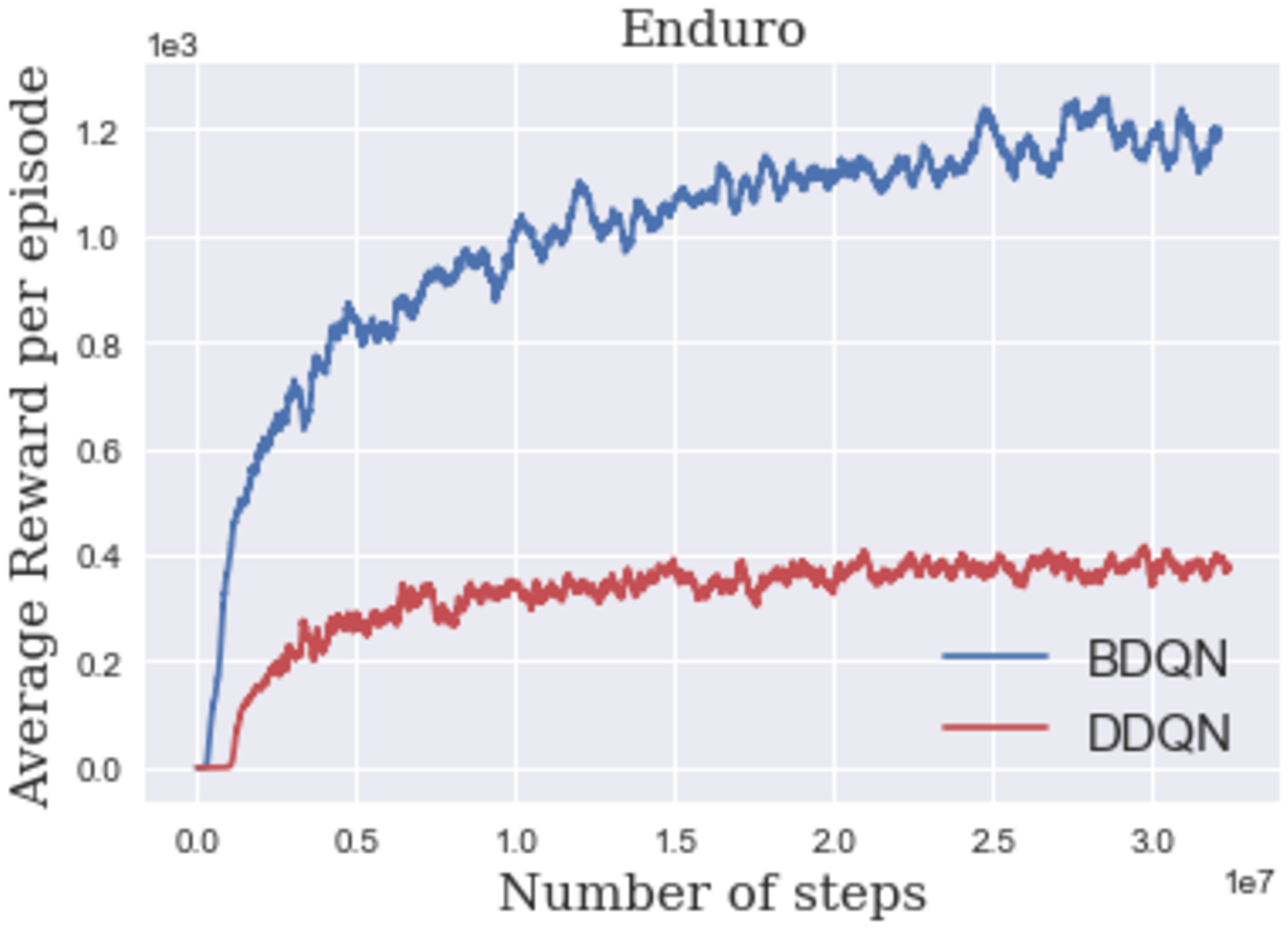}
  \end{minipage}
  \hspace*{\widmi}
  \begin {minipage}{\widmini\textwidth}
     \centering
\includegraphics[width=\widplot\linewidth]{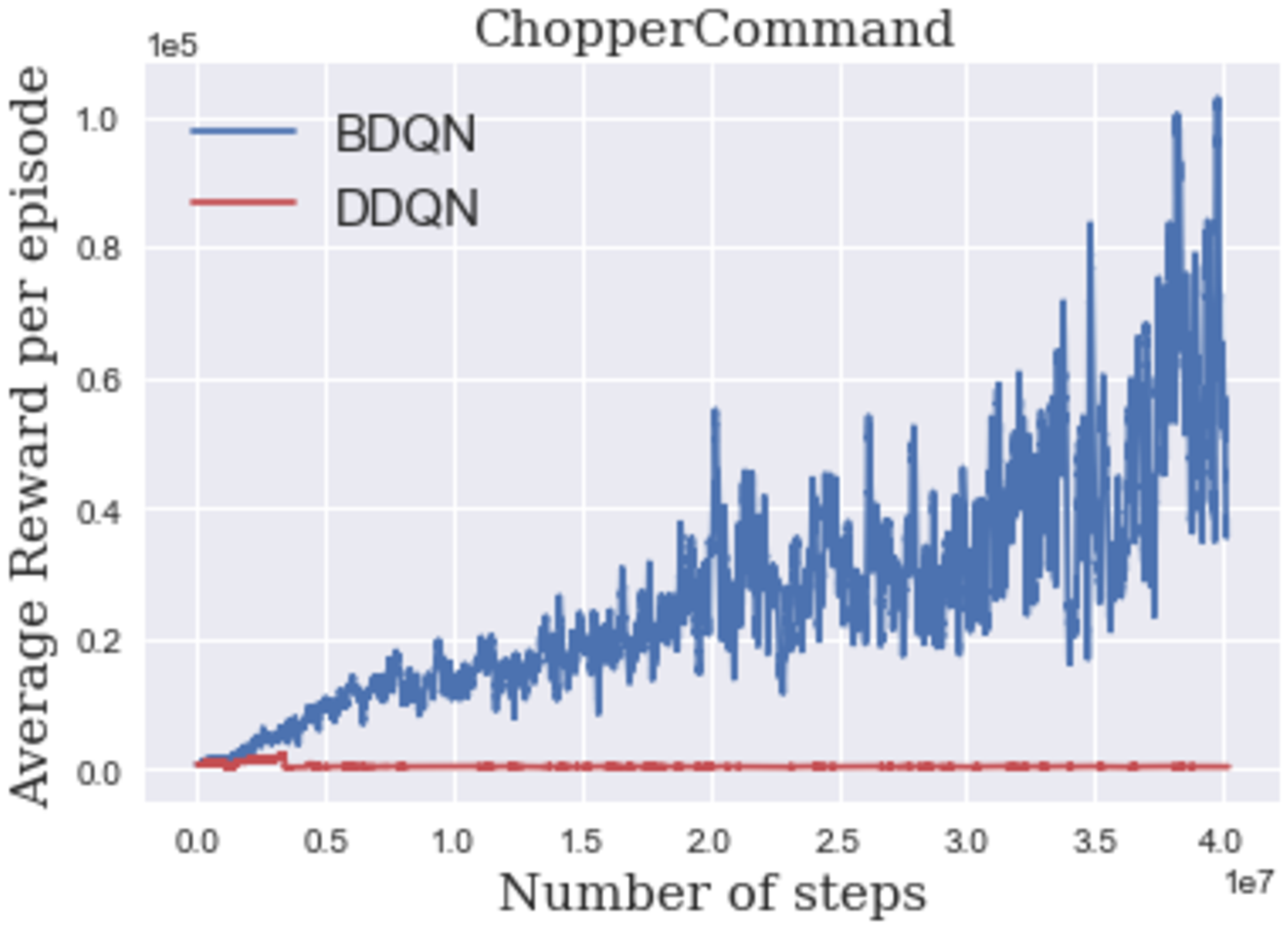}
  \end{minipage}
  \hspace*{\widmi}
  \begin {minipage}{\widmini\textwidth}
     \centering
\includegraphics[width=\widplot\linewidth]{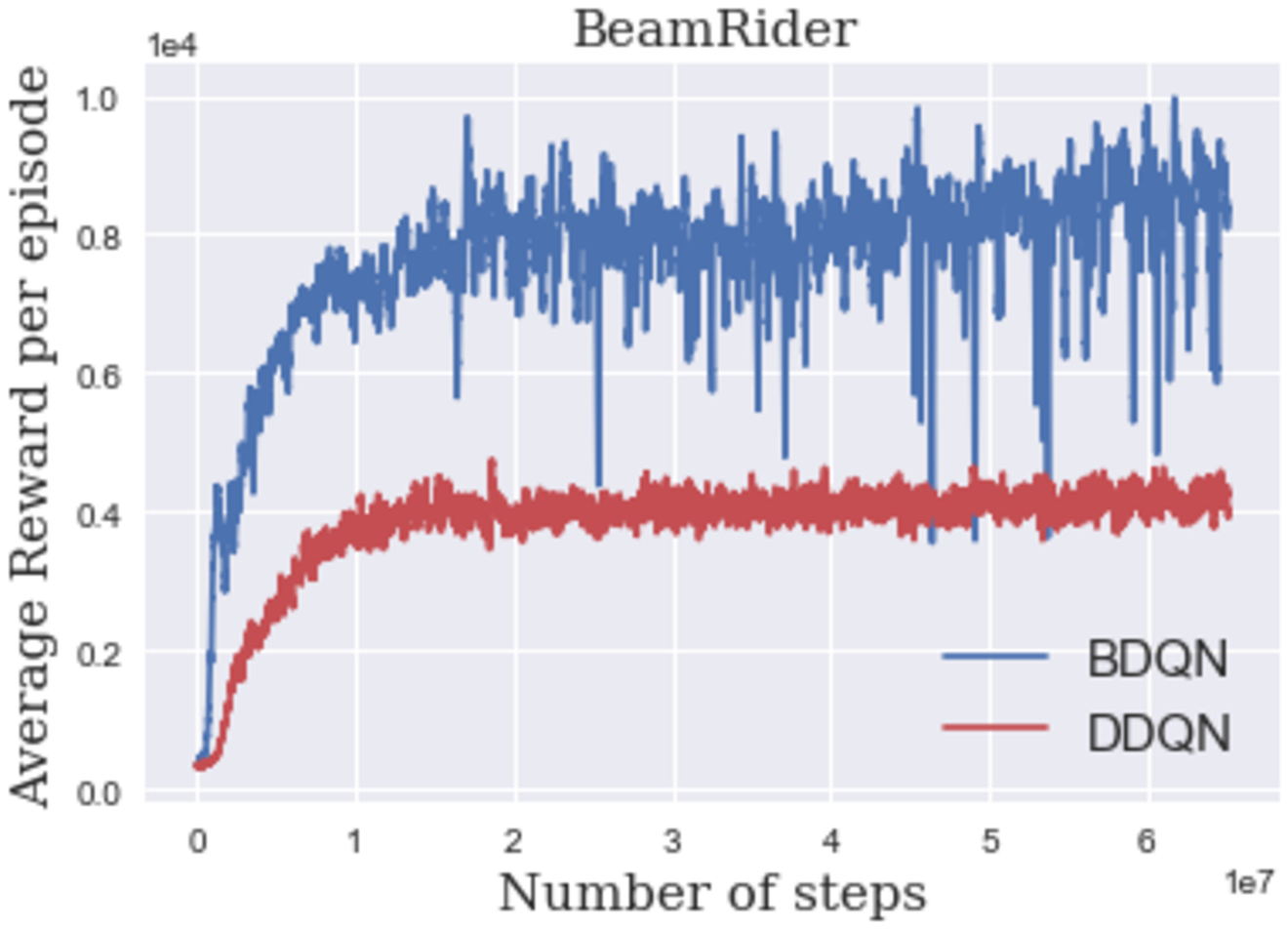}
  \end{minipage}
\caption{The comparison between \DDQN and \bdqn}
   \label{fig:exp}
   \vspace*{-0.4cm}
\end{figure*}

We observe that \bdqn learns significantly better policies due to its efficient explore/exploit in a much shorter period of time. Since \bdqn on game \textit{Atlantis} promise a big jump around time step $20M$, we ran it five more times in order to make sure it was not just a coincidence Fig.~\ref{fig:expatlantis}. For the game Pong, we ran the experiment for a longer period but just plotted the beginning of it in order to observe the difference. Due to cost of deep RL methods, for some games, we run the experiment until a plateau is reached.

\textbf{Interesting point about game \textit{Atlantis}}

For the game \textit{Atlantis}, $\DDQN^+$ reaches the score of $64.67k$ during the evaluation phase, while \bdqn reaches score of $3.24M$ after $20M$ interactions. As it is been shown in Fig.~\ref{fig:exp}, \bdqn saturates for \textit{Atlantis} after $20M$ interactions. We realized that \bdqn reaches the internal \textit{OpenAIGym} limit of $max\_episode$, where relaxing it improves score after $15M$ steps to $62M$.

\begin{table}[h]
  \centering
  \caption{
$1st$ column: score ratio of \bdqn to \DDQN run for same number of time steps. $2nd$ column: score ratio of \bdqn to $\DDQN^+$. $3rd$ column: score ratio of \bdqn to human scores reported at \citet{mnih2015human}. $4th$ column: Area under the performance plot ration (AuPPr) of \bdqn to \DDQN. AuPPr is the integral of area under the performance plot ration. For Pong, since the scores start form $-21$, we shift it up by 21.  $5th$ column: Sample complexity, \textit{SC}: the number of samples the \bdqn requires to beat the human score \citep{mnih2015human}($``-"$ means \bdqn could not beat human score). $6th$ column: $\textit{SC}^+$: the number of samples the \bdqn requires to beat the score of $\DDQN^+$. We run both \bdqn and \DDQN for the same number of times steps, stated in the last column.
}
  \footnotesize
  \begin{tabular}{l|cccccc|c}
    \multicolumn{1}{c}{}&\multicolumn{7}{c}{} \\
    \toprule
Game  & $\dfrac{\bdqn}{\DDQN}$ & $\dfrac{\bdqn}{\DDQN^+}$& $\dfrac{\bdqn}{\textsc{Human}}$& AuPPr &SC & $SC^+$& Steps \\
\midrule
Amidar & 558\% & 788\% & 325\%&280\% &  22.9M &4.4M & 100M\\
Alien & 103\% & 103\%  & 43\%&110\% &  - &36.27M &100M\\ 
Assault &396\% & 176\% & 589\%&290\% & 1.6M & 24.3M &100M\\ 
Asteroids  & 2517\% & 1516\% & 108\%&680\% &58.2M  & 9.7M &  100M\\ 
Asterix  & 531\% & 385\%  & 687\%&590\% &3.6M  & 5.7M &100M\\ 
BeamRider  & 207\% & 114\% &150\%&210\% & 4.0M  & 8.1M  &   70M\\ 
BattleZone  & 281\%  & 253\% & 172\%&180\% &25.1M  & 14.9M  & 50M\\ 
Atlantis  & 80604\% & 49413\% & 11172\%&380\% & 3.3M  & 5.1M  & 40M\\ 
DemonAttack  & 292\% & 114\% & 326\%&310\% & 2.0M  & 19.9M & 40M\\ 
Centipede  & 114\% & 178\% & 61\%&105\% &-  & 4.2M & 40M\\ 
BankHeist  & 211\% & 100\% & 100\%&250\% &2.1M  & 10.1M & 40M\\ 
CrazyClimber  & 148\% & 122\% & 350\%&150\% &0.12M  & 2.1M &  40M\\ 
ChopperCommand & 14500\% & 1576\% & 732\%&270\% & 4.4M  &2.2M  &  40M\\ 
Enduro  & 295\% & 350\% &361\% &300\%&0.82M  & 0.8M& 30M\\ 
Pong  & 112\% & 100\% & 226\%&130\% &1.2M  &2.4M  & 5M\\ 
    \bottomrule
  \end{tabular}
  \label{table:percentage}
\end{table}

After removing the maximum episode length limit for the game Atlantis, \bdqn gets the score of 62M. This episode is long enough to fill half of the replay buffer and make the model perfect for the later part of the game but losing the crafted skill for the beginning of the game. We observe in Fig.~\ref{fig:atlantislimit} that after losing the game in a long episode, the agent forgets a bit of its skill and loses few games but wraps up immediately and gets to score of $30M$. To overcome this issue, one can expand the replay buffer size, stochastically store samples in the reply buffer where the later samples get stored with lowest chance, or train new models for the later parts of the episode. There are many possible cures for this interesting observation and while we are comparing against DDQN, we do not want to advance \bdqn structure-wise.
\begin{figure}[ht]
\centering    \includegraphics[width=0.25\linewidth]{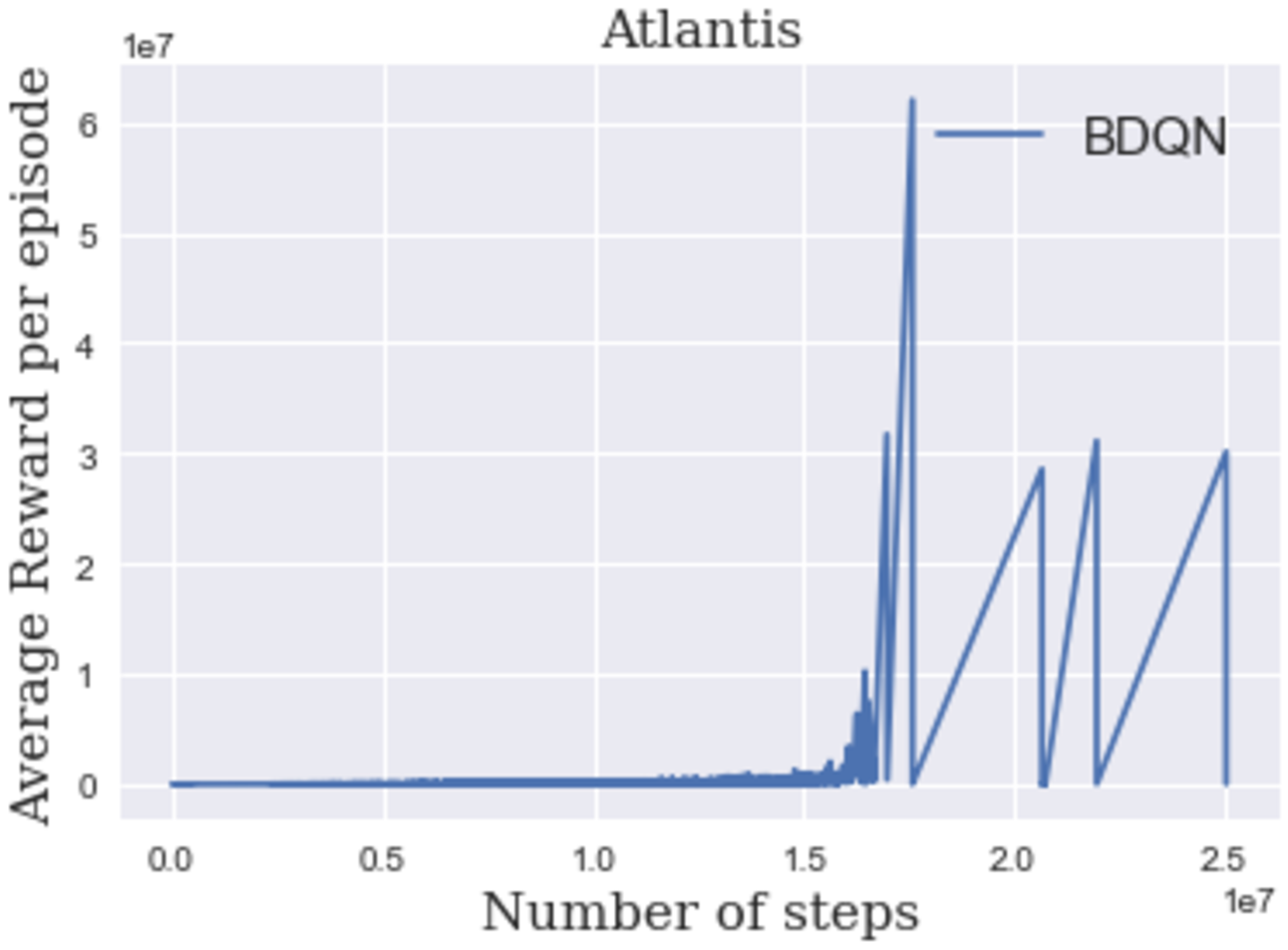}
   \caption{\bdqn on Atlantis after removing the limit on max of episode length hits the score of $62M$ in $16M$ samples.}
   \label{fig:atlantislimit}
\end{figure}

\newcommand{\widplott}{1.1}
\newcommand{\widminii}{0.25}
\begin{figure*}[ht]
   \begin{minipage}{0.28\textwidth}
	 \centering
     \vspace*{-0.2cm}
\includegraphics[width=\widplott\linewidth]{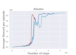}
   \end{minipage}\hfill
   \begin{minipage}{\widminii\textwidth}
     \centering
\includegraphics[width=\widplott\linewidth]{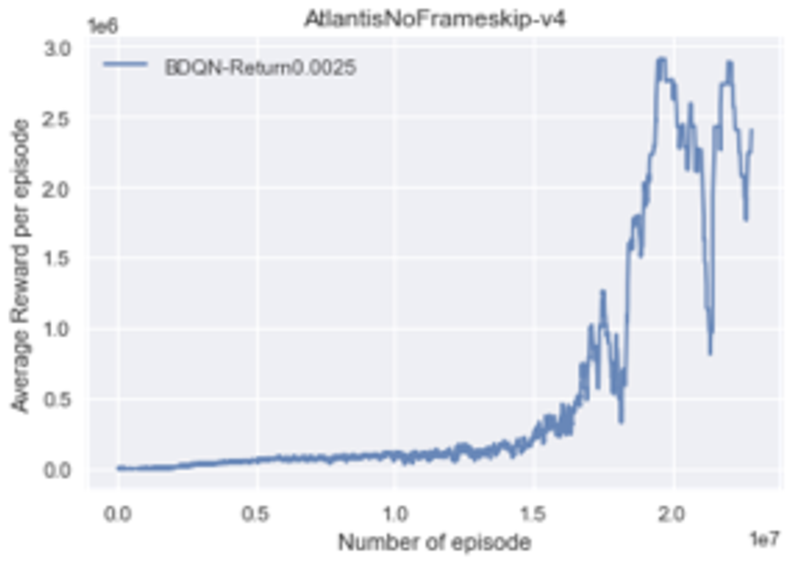}
   \end{minipage}\hfill
   \begin{minipage}{\widminii\textwidth}
     \centering
\includegraphics[width=\widplott\linewidth]{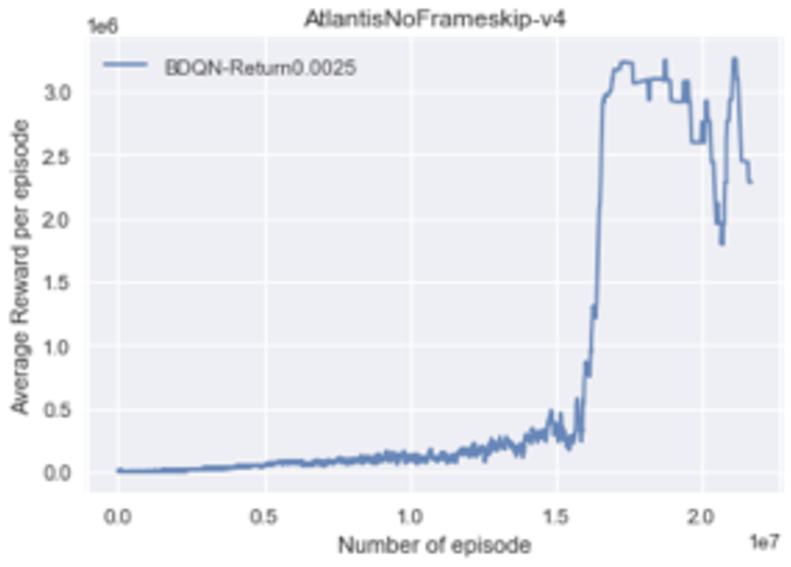}
   \end{minipage}\hfill
\caption{A couple of more runs of \bdqn where the jump around $15M$ constantly happens}
   \label{fig:expatlantis}
   \vspace*{-0.4cm}
\end{figure*}

\subsection{Dropout as a randomized exploration strategy}\label{sub:dropout}

We would like to acknowledge that we provide this study based on request by our reviewer during the last submission. 

Dropout, as another randomized exploration method, is proposed by \citet{gal2016dropout}, but \citet{osband2016deep} argue about the deficiency of the estimated uncertainty and hardness in driving a suitable exploration and exploitation trade-off from it (Appendix A in \citep{osband2016deep}). They argue that \citet{gal2016dropout} does not address the fundamental issue that for large networks trained to convergence all dropout samples may converge to every single datapoint. As also observed by \citep{dhillon2018stochastic}, dropout might results in a ensemble of many models, but all almost the same (converge to the very same model behavior). We also implemented the dropout version of \DDQN, Dropout-\DDQN, and ran it on four randomly chosen Atari games (among those we ran for less than $50M$ time steps). We observed that the randomization in Dropout-\DDQN is deficient and results in performances worse than DDQN on these four Atari games, Fig.~\ref{fig:dropout}. In Table~\ref{table:dropout} we compare the performance of \bdqn, \DDQN, $\DDQN^+$, and Dropout-\DDQN, as well as the performance of the random policy, borrowed from \citet{mnih2015human}. We observe that the Dropout-\DDQN not only does not outperform the plain $\varepsilon$-greedy \DDQN, it also sometimes underperforms the random policy. For the game Pong, we also ran Dropout-\DDQN for $50M$ time steps but its average performance did not get any better than -17. For the experimental study we used the default dropout rate of $0.5$ to mitigate its collapsing issue.

\setlength{\tabcolsep}{1pt}
\begin{table*}[h]
  \centering
  \caption{The comparison of \bdqn, \DDQN, Dropout-\DDQN and random policy. Dropout-\DDQN as another randomization strategy provides a deficient estimation of uncertainty and results in poor exploration/exploitation trade-off. }
  \footnotesize
  \begin{tabular}{l|c|c|c|c|c|c}
%
\toprule
 Game  & \bdqn & \DDQN & $\DDQN^+$ & Dropout-\DDQN{}& Random Policy&Step  \\
\midrule
CrazyClimber & 124k& 84k & 102k & 19k & 11k & 40M \\ 
Atlantis & 3.24M& 39.7k & 64.76k & 7.7k & 12.85k & 40M \\ 
Enduro & 1.12k& 0.38k & 0.32k & 0.27k & 0 & 30M \\ 
Pong & 21& 18.82 & 21 & -18 & -20.7 & 5M \\
    \bottomrule
  \end{tabular}
  \label{table:dropout}
\end{table*}

\newcommand{\widplottt}{1.1}
\newcommand{\widminitt}{0.25}
\newcommand{\widmitt}{0.1cm}
\begin{figure*}[ht]
\vspace*{-0.0cm}
  \begin{minipage}{\widminitt\textwidth}
	 \centering
\includegraphics[width=\widplottt\linewidth]{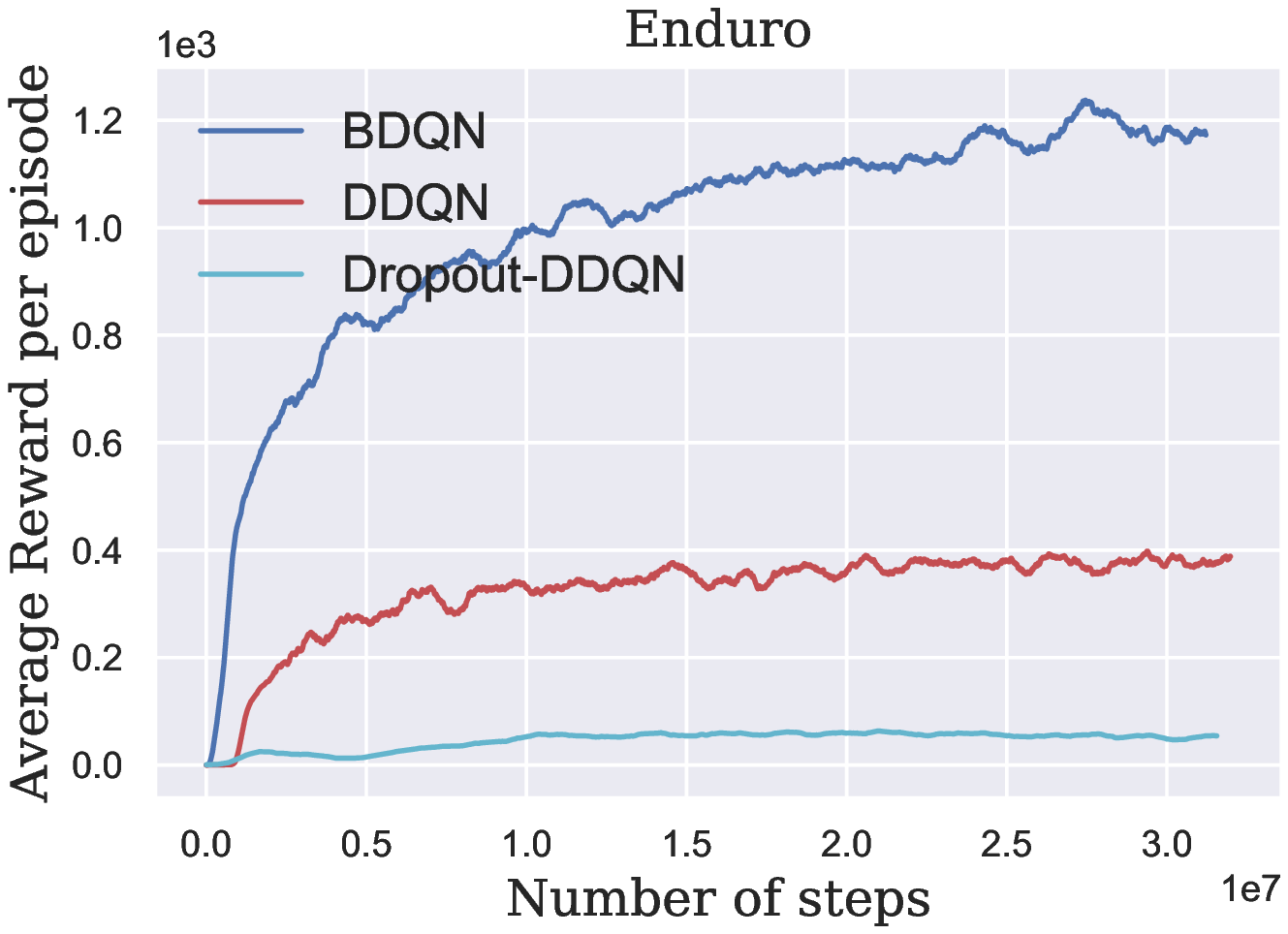}
  \end{minipage}\hfill
  \begin {minipage}{\widminitt\textwidth}
     \centering
\includegraphics[width=\widplottt\linewidth]{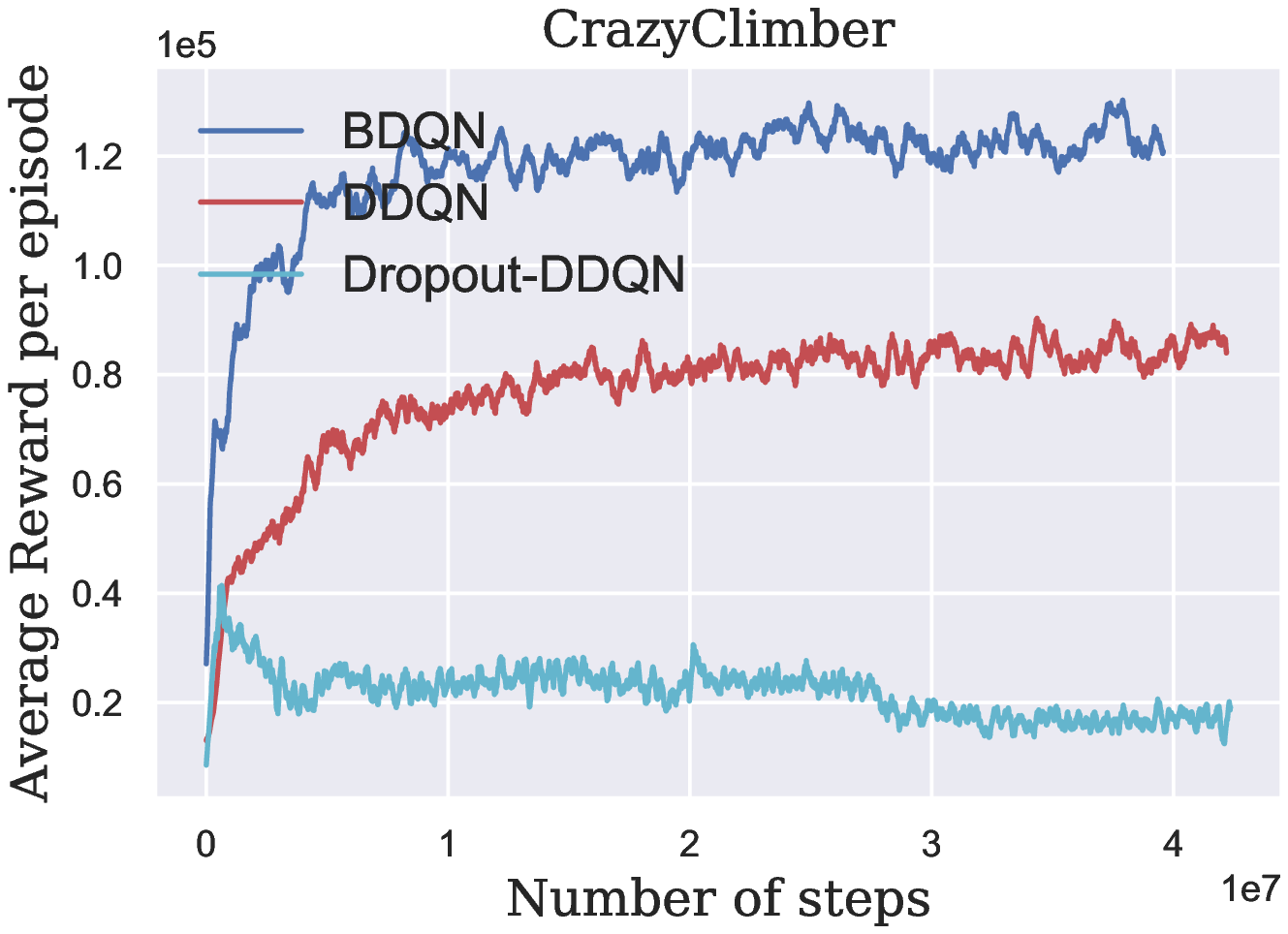}
  \end{minipage}\hfill
  \begin {minipage}{\widminitt\textwidth}
     \centering
	 \includegraphics[width=\widplottt\linewidth]{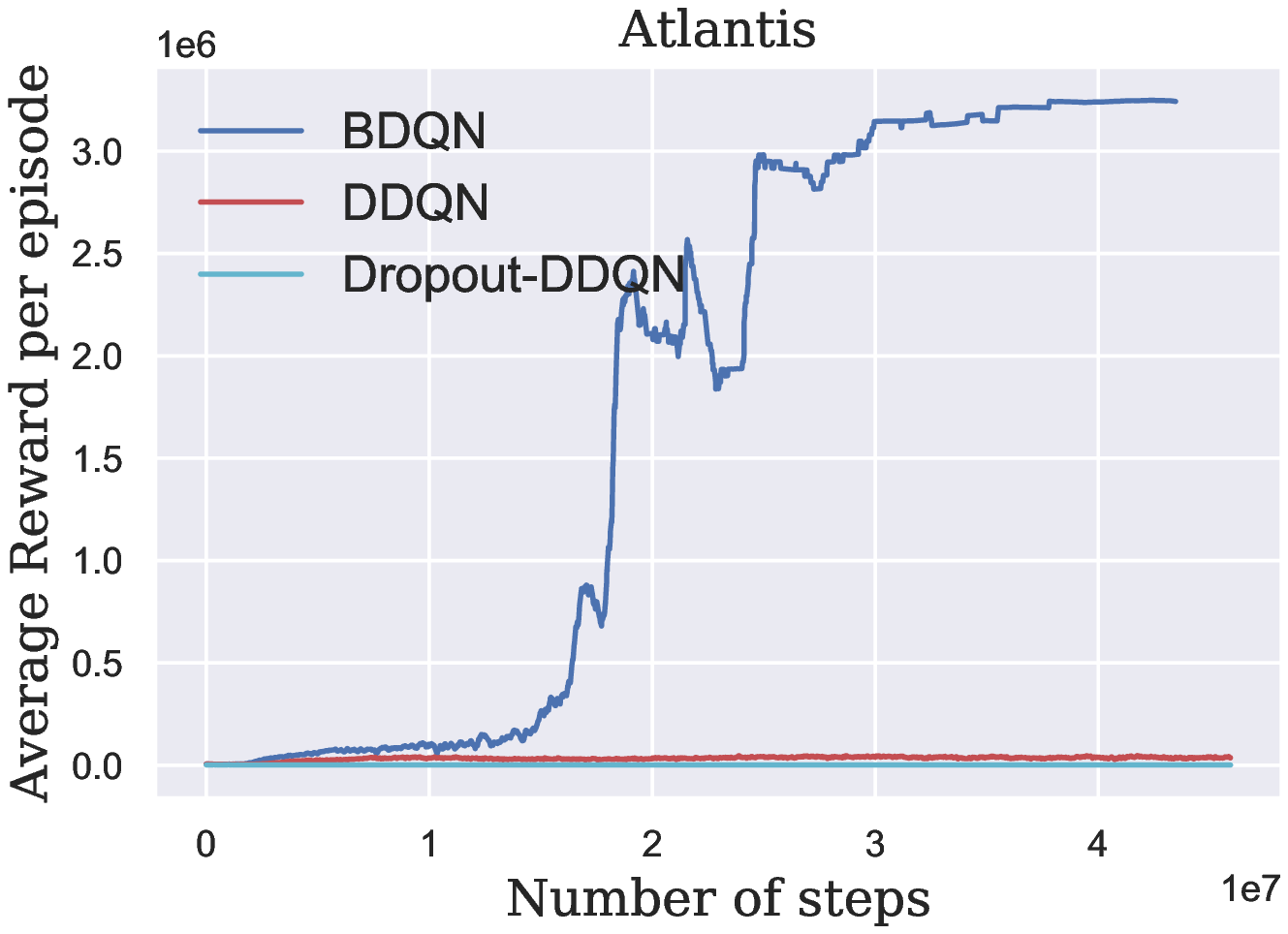}
  \end{minipage}\hfill
  \begin {minipage}{\widminitt\textwidth}
     \centering
\includegraphics[width=\widplottt\linewidth]{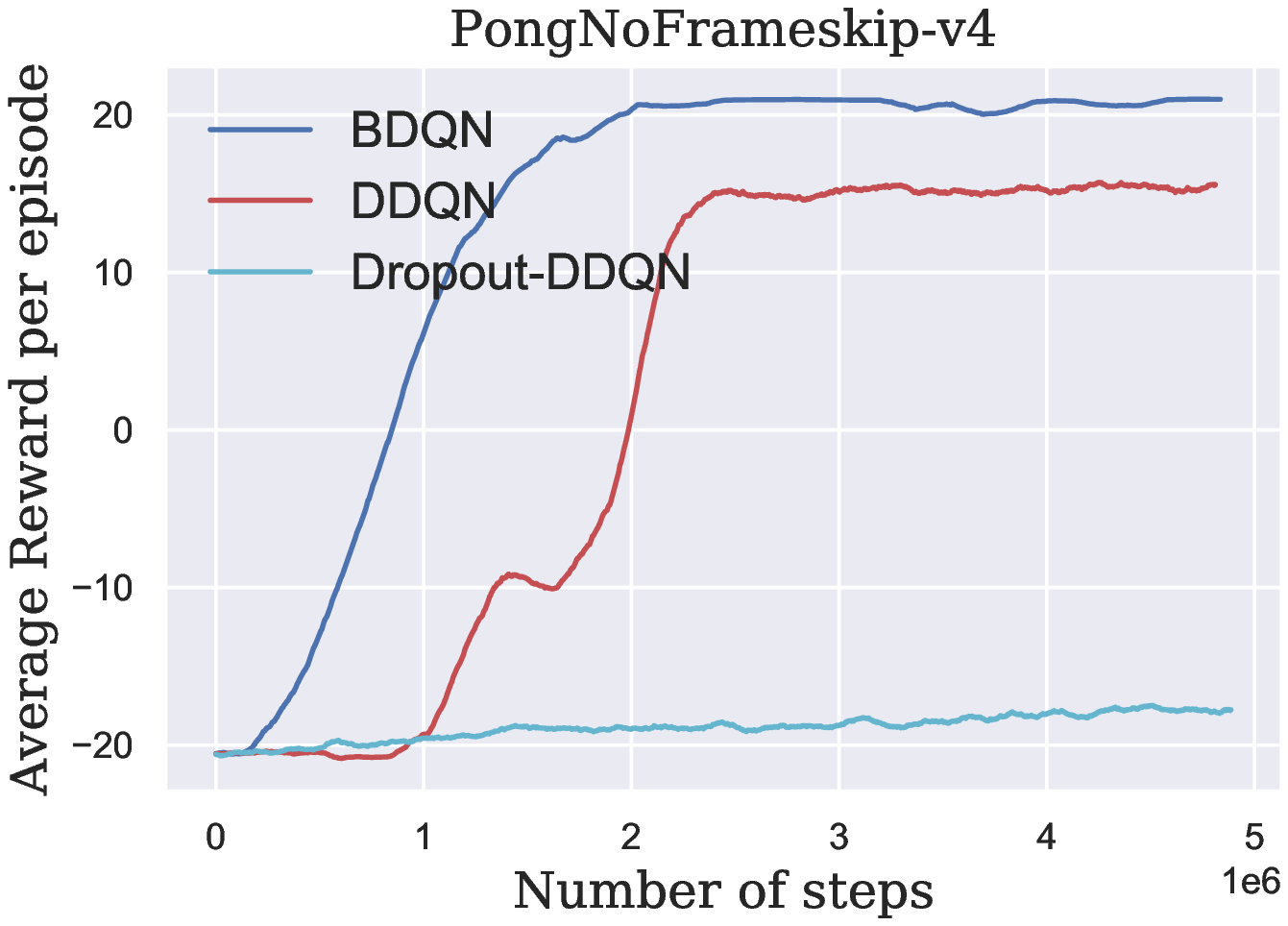}
  \end{minipage}\hfill
\caption{The comparison between \DDQN, \bdqn and Dropout-\DDQN}
   \label{fig:dropout}
\end{figure*}

\subsection{Further discussion on Reproducibility}\label{sub:reproducability}
In Table~\ref{table:scores}, we provide the scores of bootstrap \DQN \citep{osband2016deep} and NoisyNet \footnote{This work does not have scores of Noisy-net with \DDQN objective function but it has Noisy-net with \DQN objective which are the scores reported in Table~\ref{table:scores}}, and \cite{fortunato2017noisy} along with \bdqn. These score are directly copied from their original papers and we did not make any change to them. We also  report the scores of count-based method \citep{ostrovski2017count} despite its challenges. \citet{ostrovski2017count} does not provide the scores table.

Table~\ref{table:scores} shows, despite the simplicity of \bdqn, it provides a significant improvement over these baselines. It is important to note that we can not scientifically claim that \bdqn outperforms these baselines by just looking at the scores in Table\ref{table:scores} since we are not aware of their detailed implementation as well as environment details. For example, in this work, we directly implemented \DDQN by following the implementation details mentioned in the original \DDQN paper and the scores of our \DDQN implementation during the evaluation time almost matches the scores of \DDQN reported in the original paper. But the reported scores of implemented \DDQN in \citet{osband2016deep} are significantly different from the reported score in the original \DDQN paper.


\section{Why Thompson Sampling and not \texorpdfstring{$\varepsilon$}{}-greedy or Boltzmann exploration}\label{apx:TSvsE}

In  value approximation RL algorithms,  there are different ways to manage the exploration-exploitation trade-off.  \DQN uses a naive $\varepsilon$-greedy for exploration, where with $\varepsilon$ probability it chooses a random action and with $1-\varepsilon$ probability it chooses the greedy action based on the estimated $Q$ function. Note that there are only point estimates of the $Q$ function in \DQN. In contrast, our proposed Bayesian approach \bdqn maintains uncertainties over the estimated $Q$ function, and employs it to carry out Thompson Sampling based exploration-exploitation. 
Here, we demonstrate  the fundamental benefits of Thompson Sampling over $\varepsilon$-greedy and Boltzmann exploration strategies  using simplified examples. In Table~\ref{table:reasons}, we list the three strategies and their properties. 

$\varepsilon$-greedy is among the simplest exploration-exploitation strategies and it is uniformly random over all the non-greedy actions. Boltzmann exploration is an intermediate strategy since it uses the estimated $Q$ function to sample from action space. However, it does not maintain uncertainties over the $Q$ function estimation. In contrast, Thompson sampling incorporates the $Q$ estimate as well as the uncertainties in the estimation and utilizes the most information for exploration-exploitation strategy.  

Consider the example in Figure~\ref{fig:uncertainty}(a) with our current estimates and uncertainties of the $Q$ function over different actions.   $\varepsilon$-greedy is not compelling since it assigns uniform probability to explore over $5$ and $6$, which are sub-optimal when the uncertainty estimates are available. In this setting, a possible remedy is Boltzmann exploration since it assigns lower probability to actions $5$ and $6$ but randomizes with almost the same probabilities over the remaining actions.

\begin{figure}[h]
  \begin{minipage}[]{0.27\textwidth}
	 \centering
      \hspace*{2.cm}
\includegraphics[height=.9\linewidth]{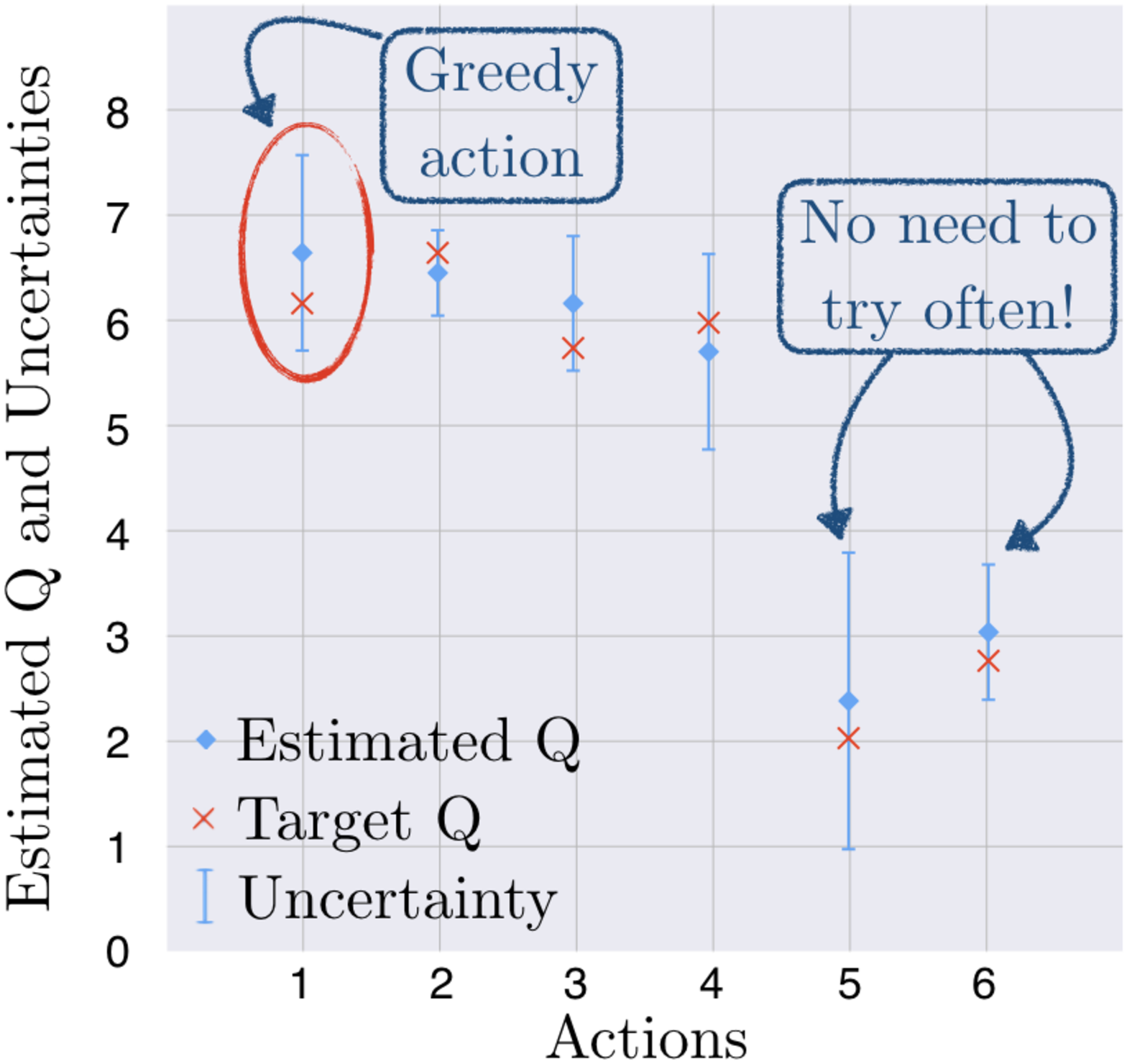}\\
     \hspace*{4.0cm}(a)
  \end{minipage}\hfill
  \begin{minipage}{0.27\textwidth}
     \centering
	 \includegraphics[height=.9\linewidth]{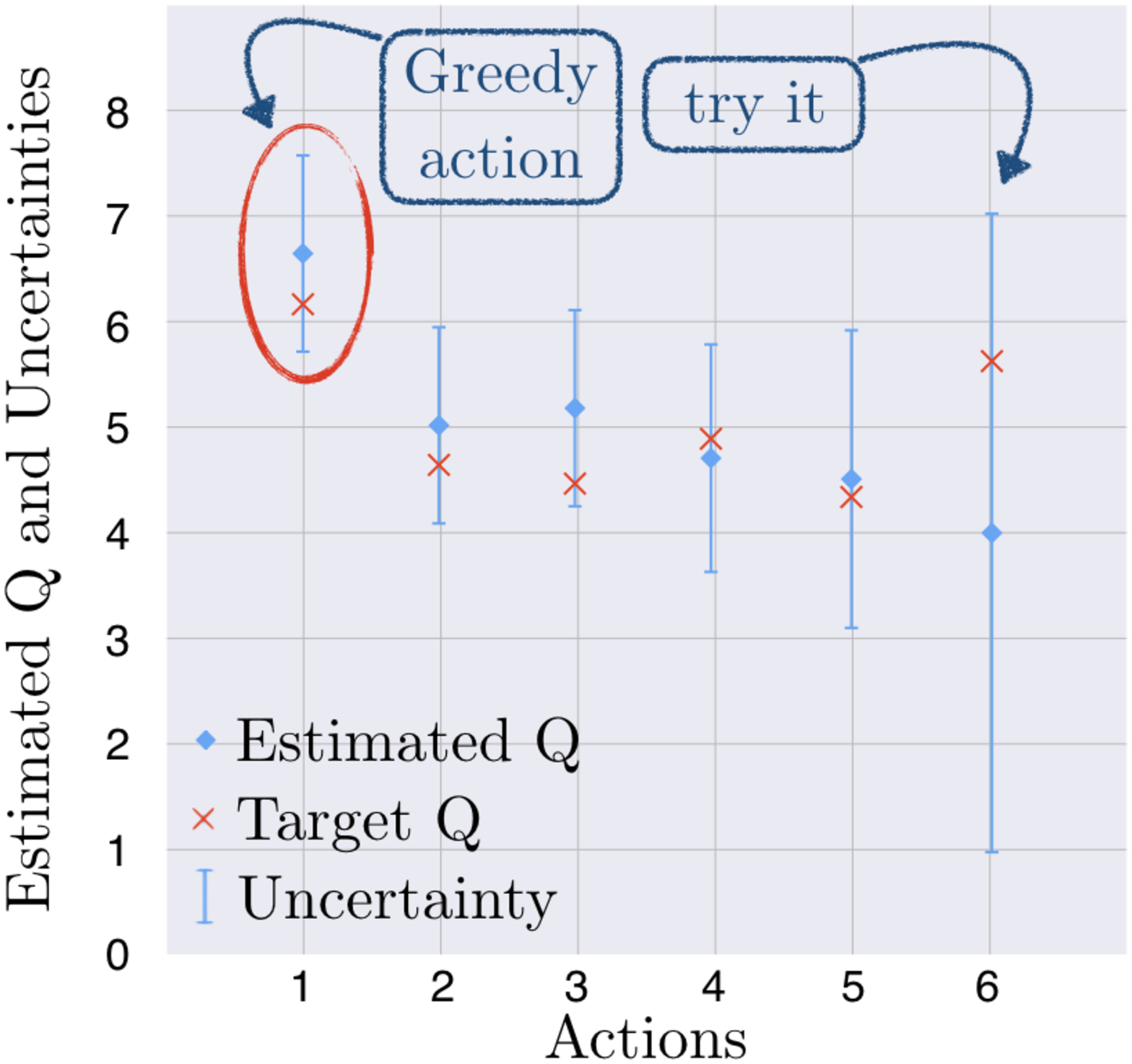}\\
     \hspace*{0.5cm}     
     (b)
  \end{minipage}\hfill
   \caption{Thompson Sampling vs $\varepsilon$-greedy and Boltzmann exploration. $(a)$ $\varepsilon$-greedy is wasteful since it assigns uniform probability to explore over $5$ and $6$, which are obviously sub-optimal when the uncertainty estimates are available. Boltzmann exploration randomizes over actions even if the optimal action is identifies. (b) Boltzmann exploration does not incorporate uncertainties over the estimated action-values and chooses actions 5 and 6 with similar probabilities while action 6 is significantly more uncertain. Thomson Sampling is a simple remedy to all these issues. 
}
\label{fig:uncertainty}
\end{figure}

However, Boltzmann exploration is sub-optimal in settings where there is high uncertainty. For example if the current $Q$ estimate is according to Figure ~\ref{fig:uncertainty}(b), then Boltzmann exploration assigns almost equal probability to actions $5$ and $6$, even though action $6$ has much higher uncertainty and needs to be explored more. 

Thus, both  $\varepsilon$-greedy and Boltzmann exploration strategies are sub-optimal since they do not maintain an uncertainty estimate over the $Q$ estimation. In contrast, Thompson sampling uses both estimated $Q$ function and its uncertainty estimates to carry out a more efficient exploration.


\section{\bdqn Implementation}\label{apx:lr}
In this section, we provide the details of \bdqn algorithm, mentioned in Algorithm~\ref{alg:bdqn}.

As mentioned, \bdqn agent deploys Thompson sampling on approximated posteriors over the weights of the last layer $w_a$'s. At every $T^{S}$ time step, \bdqn samples a new model to balance exploration and exploitation while updating the posterior every $T^{BT}$ time steps using $B$ experience in the replay buffer, chosen uniformly at random.

\bdqn agent samples a new Q function, (i.e., $w_a,~\forall{a}$), every $T^{S}:=\mathcal{O}(episodes~length)$ time steps and act accordingly $a_{\TS}:=\max_a w_a^\top\phi_{\theta}(x)$. As suggested by derivations in Section~\ref{sec:Theory}, $T^{S}$ is chosen to be in order of episode~length for the Atari games. As mentioned before, the feature network is trained similar to \DDQN using $\theta \gets \theta - \alpha \cdot \nabla_{\theta} (y_\tau -  w_{a_\tau}^{\top}\phi_{\theta}(x_{\tau}) )^2$.

Similar to \DDQN, we update the target network every $T^{T}$ steps and set $\theta^{target}$ to $\theta$. We update the posterior distribution using a minibatch of $B$ randomly chosen experiences in the replay buffer every $T^{BT}$, and set the $w_a^{target}=\wb w_a,~\forall a\in\A$, the mean of the posterior distribution  Alg.~\ref{alg:bdqn}.

\subsection{Network architecture:}\label{sub:architecture}
The input to the network part of \bdqn is $4\times 84\times 84$ tensor with a rescaled and averaged over channels of the last four observations. The first convolution layer has $32$ filters of size $8$ with a stride of $4$. The second convolution layer has $64$ filters of size $4$ with stride $2$. The last convolution layer has $64$ filters of size $3$ followed by a fully connected layer with size $512$. We add a BLR layer on top of this.

\subsection{Choice of hyper-parameters:}\label{sub:hyper-choice}
For \bdqn, we set the values of $W^{target}$ to the mean of the posterior distribution over the weights of BLR with covariances $Cov$ and draw $w_a$'s from this posterior. For the fixed set of $w_a$'s and $w_a^{target}$, we randomly initialize the parameters of network part of \bdqn, $\theta$, and train it using RMSProp, with learning rate of $0.0025$, and a momentum of $0.95$, inspired by \citep{mnih2015human} where the discount factor is $\gamma = 0.99$, the number of steps between target updates $T^{T} = 10k$ steps, and weights $W$ are re-sampled from their posterior distribution  every $T^{S}$ steps. We update the network part of \bdqn every $4$ steps by uniformly at random sampling a mini-batch of size $32$ samples from the replay buffer. We update the posterior distribution of the weight set $W$ every $T^{BT}$ using mini-batch of size $B$, with entries sampled uniformly form replay buffer. The experience replay contains the $1M$ most recent transitions. Further hyper-parameters are equivalent to ones in DQN setting.

For the BLR, we have noise variance $\sigma_\epsilon$, variance of prior over weights $\sigma$, sample size $B$, posterior update period $T^{BT}$, and the posterior sampling period $T^{S}$. To optimize for this set of hyper-parameters we set up a very simple, fast, and cheap hyper-parameter tuning procedure which proves the robustness of \bdqn. To find the first three, we set up a simple hyper-parameter search. We used a pretrained \DQN model for the game of \textit{Assault}, and removed the last fully connected layer in order to have access to its already trained feature representation. Then we tried combination of $B=\lbrace T^{T},10\cdot T^{T}\rbrace$, $\sigma =\lbrace 1,0.1,0.001 \rbrace$, and $\sigma_\epsilon =\lbrace 1,10 \rbrace $ and test for $1000$ episode of the game. We set these parameters to their best $B=10\cdot T^{T},\sigma = 0.001,\sigma_\epsilon =1 $. 

The above hyper-parameter tuning is cheap and fast since it requires only a few times the $B$ number of forwarding passes. Moreover, we did not try all the combination, and rather each separately. We observed that $\sigma = 0.1,0.001$ perform similarly with $\sigma = 0.001$ slightly better. But both perform better than $\sigma = 1$. For $B=\lbrace T^{T},10\cdot T^{T}\rbrace$. For the remaining parameters, we ran \bdqn ( with weights randomly initialized) on the same game, \textit{Assault}, for $5M$ time steps, with a set of $T^{BT}=\lbrace T^{T},10\cdot T^{T}\rbrace$ and $T^{S}=\lbrace \frac{T^{T}}{10}, \frac{T^{T}}{100} \rbrace$, where \bdqn performed better with choice of $T^{BT}=10\cdot T^{T}$. For both choices of $T^{S}$, it performs almost equal and we choose the higher one to reduce the computation cost. We started off with the learning rate of $0.0025$ and did not tune for that. Thanks to the efficient  Thompson sampling exploration and closed form BLR, \bdqn can learn a better policy in an even shorter period of time. In contrast, it is well known for \DQN based methods that changing the learning rate causes a major degradation in the performance (Fig.~\ref{fig:lr}). The proposed hyper-parameter search is very simple and an exhaustive hyper-parameter search is likely to provide even better performance.

\subsection{Learning rate:}\label{sub:lr}
It is well known that \DQN and \DDQN are sensitive to the learning rate and change of learning rate can degrade the performance to even worse than random policy. We tried the same learning rate as \bdqn, 0.0025, for \DDQN and observed that its performance drops. Fig.~\ref{fig:lr} shows that the \DDQN with higher learning rates learns as good as \bdqn at the very beginning but it can not maintain the rate of improvement and degrade even worse than the original \DDQN with learning rate of 0.00025.
\begin{figure}[ht]
\centering
\includegraphics[width=0.5\linewidth]{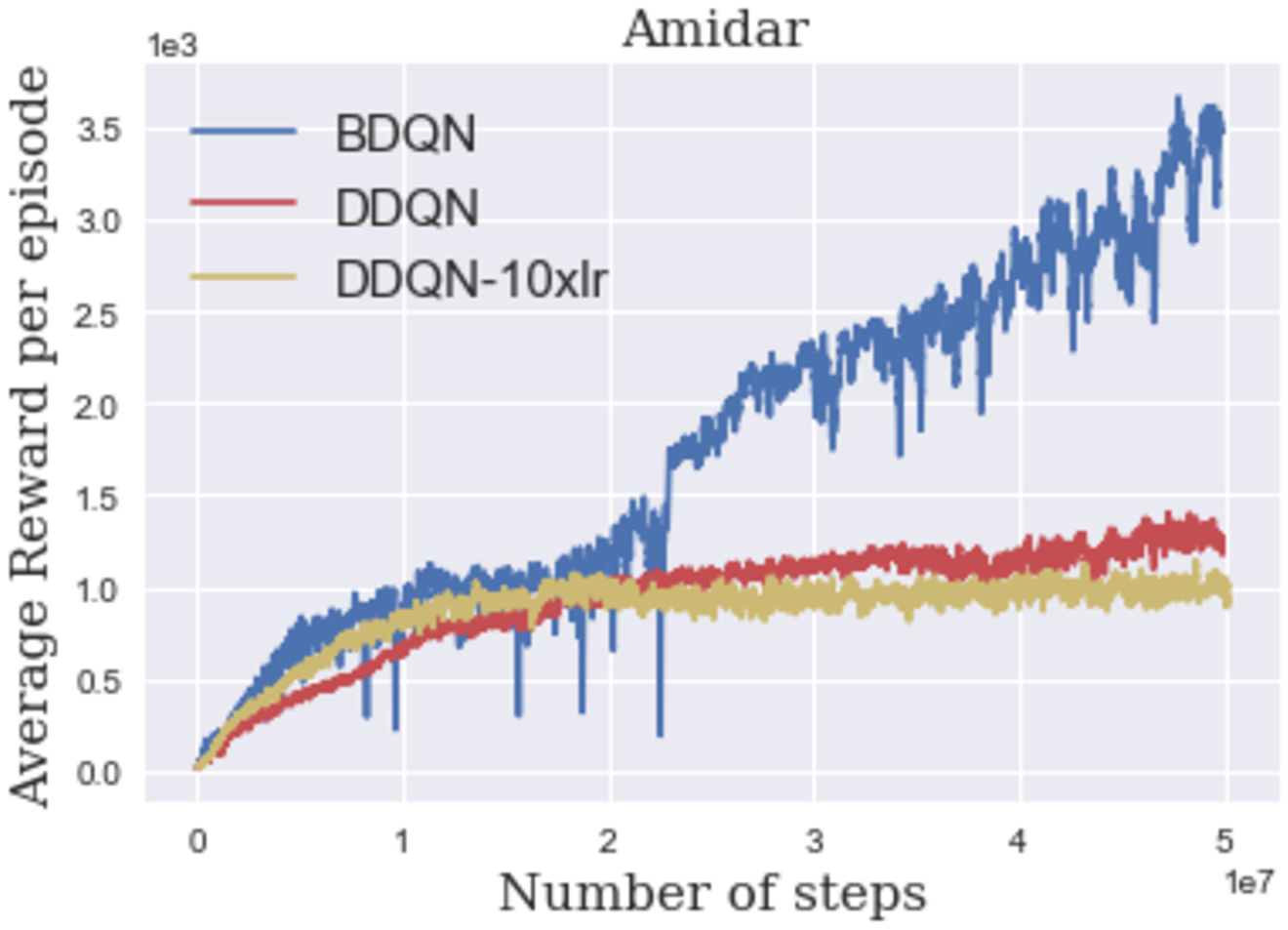}
   \caption{Effect of learning rate on \DDQN }
   \label{fig:lr}
\end{figure}

\subsection{Computational and sample cost comparison: }\label{sub:computation}
For a given period of game time, the number of the backward pass in both \bdqn and \DQN are the same where for \bdqn it is cheaper since it has one layer (the last layer) less than \DQN. In the sense of fairness in sample usage, for example in duration of $10\cdot T^{BT} = 100k$, all the layers of both \bdqn and \DQN, except the last layer, sees the same number of samples, but the last layer of \bdqn sees $16$ times fewer samples compared to the last layer of \DQN. The last layer of \DQN for a duration of $100k$, observes $25k = 100k/4$ (4 is back prob period) mini batches of size $32$, which is $16\cdot 100k$, where the last layer of \bdqn just observes samples size of $B=100k$. As it is mentioned in Alg.~\ref{alg:bdqn}, to update the posterior distribution, \bdqn draws $B$ samples from the replay buffer and needs to compute the feature vector of them. Therefore, during the $100k$ interactions for the learning procedure, \DDQN does $32*25k$ of forward passes and $32*25k$ of backward passes, while \bdqn does same number of backward passes (cheaper since there is no backward pass for the final layer) and $36*25k$ of forward passes.
One can easily relax it by parallelizing this step along the main body of \bdqn or deploying on-line posterior update methods.

\subsection{Thompson sampling frequency:}\label{sub:TS-frequency}
The choice of Thompson sampling update frequency can be crucial from domain to domain. Theoretically, we show that for episodic learning, the choice of sampling at the beginning of each episode, or a bounded number of episodes is desired. If one chooses $T^{S}$ too short, then computed gradient for backpropagation of the feature representation is not going to be useful since the gradient get noisier and the loss function is changing too frequently. On the other hand, the network tries to find a feature representation which is suitable for a wide range of different weights of the last layer, results in improper waste of model capacity. If the Thompson sampling update frequency is too low, then it is far from being Thompson sampling and losses the randomized exploration property. We are interested in a choice of $T^{S}$ which is in the order of upper bound on the average length of each episode of the Atari games.  The current choice of $T^{S}$ is suitable for a variety of Atari games since the length of each episode is in range of $\mathcal{O}(T^{S})$ and is infrequent enough to make the feature representation robust to big changes. 

For the RL problems with shorter a horizon we suggest to introduce two more parameters, $\tilde{T}^{S}$ and each $\tilde{w}_a$ where $\tilde{T}^{S}$, the period that of each $\tilde{w}_a$ is sampled out of posterior, is much smaller than $T^{S}$ and $\tilde{w}_a,\forall a$ are used for Thompson sampling where $w_a~,\forall a$ are used for backpropagation of feature representation. For game Assault, we tried using $\tilde{T}^{S}$ and each $\tilde{w}_a$ but did not observe much a difference, and set them to ${T}^{S}$ and each $w_a$. But for RL setting with a shorter horizon, we suggest using them.

\begin{table*}[ht]
  \centering
  \caption{After a very elementary hyper parameter search using two days of GPU (P2.xlarge) time, we observed that this set of parameters are proper. Of course a typical exhaustive search can provide higher performance, but this effort is not properly aligned with the purpose of this paper.}
  \footnotesize
  \begin{tabular}{c|c|c|c|c|c|c}
    \toprule
 $T^T=10k~$  & $~T^S=T^T/10~$  & $~T^{BT}=10T^T~$  &  $~B=10T^T~$ & $~lr=0.0025~$ & $~\sigma = 0.001~$ & $\sigma_\epsilon =1$\\
    \bottomrule
  \end{tabular}
  \label{table:hpo}
\end{table*}


\section{A short discussion on safety}
In \bdqn, as mentioned in Eq.~\ref{eq:w}, the prior and likelihood are conjugate of each others. Therefore, we have a closed form posterior distribution of the discounted return, $\sum_{t=0}^N \gamma^tr_t|x_0 = x,a_0=a,\D_a$, approximated as 

\begin{align*}
\N\left(\frac{1}{\sigma^2_\epsilon}\phi^{\theta}(x)^\top\Xi_a\Phi_a^\theta\y_a ,\phi^{\theta}(x)^\top\Xi_a\phi^{\theta}(x)\right)
\end{align*}
One can use this distribution and come up with a safe RL criterion for the agent ~\citep{garcia2015comprehensive}. Consider the following example; for two actions with the same mean, if the estimated variance over the return increases, then the action becomes more unsafe Fig.~\ref{fig:safe}. By just looking at the low and high probability events of returns under different actions we can approximate whether an action is safe to take.

\begin{figure*}[ht]
\vspace*{-0.0cm}
	 \centering
\includegraphics[width=1.\linewidth]{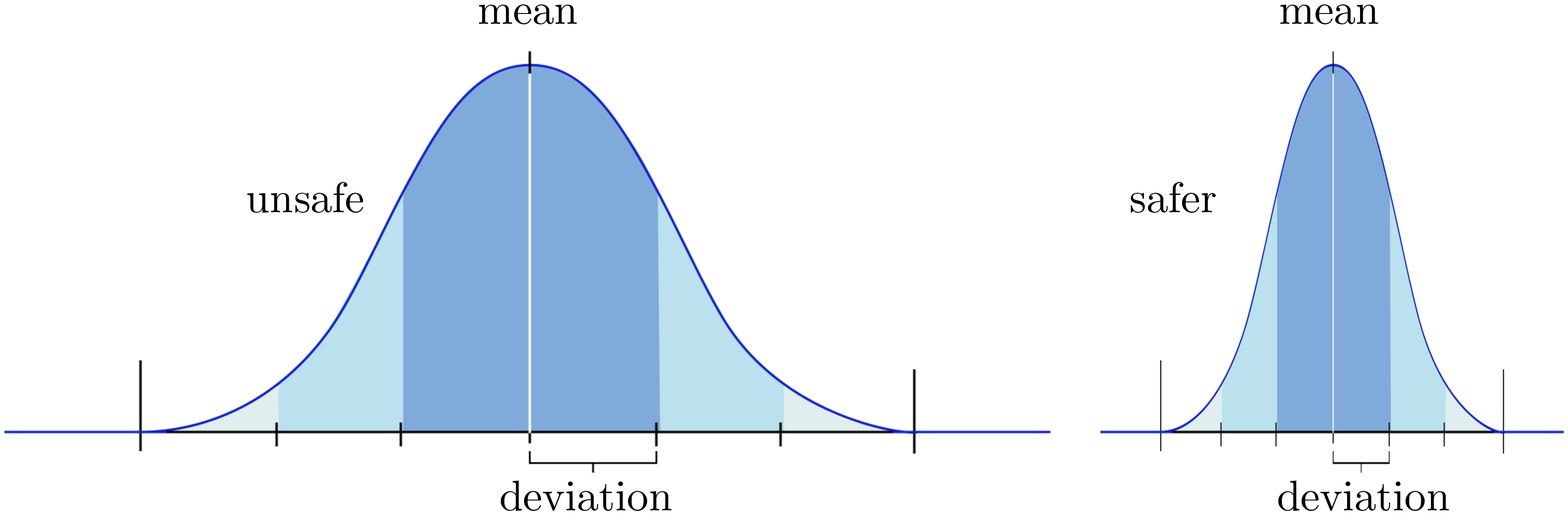}
\caption{Two actions, with the same mean, but the one with higher variance on the return might be less safe than the one with narrower variance on the return.}
   \label{fig:safe}
   \vspace*{-0.4cm}
\end{figure*}

\section{Bayesian and frequentist regrets, Proof of Theorems \ref{Thm:BR} and \ref{Thm:frequentist}}\label{apx:theory}

\subsection{Modeling}

We consider an episodic MDP with episode length of $H$, accompanied with discount factor $0\leq \gamma\leq 1$.
The optimal Q function $Q^*(\cdot,\cdot)\rightarrow \Re$ is the state action conditional expected return under the optimal policy. For any time step $h$, and any state and action pair $x^{h},a^{h}$ we have;

\begin{align*}
Q^*(x^{h},a^{h})=\E\left[\sum_{h'=h}^H \gamma^{h'-h}R_{h'}\Big|X^{h}=x^{h},A^{h}=a^{h},\pi^*\right]
\end{align*}
where $X^{h}$ and $A^{h}$ denote the state and action random variables. Since we assume the environment is an MDP, following the Bellman optimality we have for $h<H$
\begin{align*}
Q^*(x^{h},a^{h})=\E\left[R_{h}+\gamma Q^*(X^{h+1},\pi^*(X^{h+1}))\Big|X^{h}=x^{h},A^{h}=a^{h},\pi^*\right],
\end{align*}
and for $h=H$ 
\begin{align*}
Q^*(x^{H},a^{H})=\E\left[R_{H}\Big|X^{H}=x^{H},A^{H}=a^{H},\pi^*\right]
\end{align*}
where the optimal policy $\pi^*(\cdot)$ denotes a deterministic mapping from states to actions. Following the Bellman optimality, conditioned on $X^{h}=x^{h},A^{h}=a^{h}$, one can rewrite the reward at time step $h$, $R^h$, as follows;

\begin{align*}
R^{h} = Q^*(x^{h},a^{h})-\gamma\E\left[ Q^*(X^{h+1},\pi^*(X^{h+1}))\Big|X^{h}=x^{h},A^{h}=a^{h},\pi^*\right]+R_\nu^h
\end{align*}
Where $R_\nu^h$ is the noise in the reward, and it is a mean zero random variable due to bellman optimality. In other word, condition on $X^{h}=x^{h},A^{h}=a^{h}$, the distribution of one step reward can be described as:
\begin{align*}
R^{h}+\gamma\E\left[ Q^*(X^{h+1},\pi^*(X^{h+1}))\Big|X^{h}=x^{h},A^{h}=a^{h},\pi^*\right]=Q^*(x^{h},a^{h})+R_\nu^h
\end{align*}
where the equality is in distribution. We also could extend the randomness in the noise of the return one step further and include the randomness in the transition. It means, condition on $X^{h}=x^{h},A^{h}=a^{h}$ and following $\pi^*$ at a time step $h+1$, for the distribution of one step return we have:
\begin{align}\label{eq:reward-unbiased}
 R^{h} + \gamma Q^*(X^{h+1},\pi^*(X^{h+1}))=Q^*(x^{h},a^{h}) +\nu^h 
\end{align}
where here the $\nu^h$ encodes the randomness in the reward as well as the transition kernel, and it is a mean zero random variable. The equality is in distribution. If instead of following $\pi^*$ after time step $h+1$, we follow policies other than $\pi^*$, e.g., $\pi$ then condition on $X^{h}=x^{h},A^{h}=a^{h}$ and following $\pi$ at a time step $h+1$, and $\pi^*$ afterward, for the distribution of one step return we have:
\begin{align}\label{eq:reward-biased}
R^{h}+\gamma Q^*(X^{h+1},\pi(X^{h+1})) = Q^*(x^{h},a^{h})+\nu^h
\end{align}
where then noise process $\nu^h$ is not mean zero anymore (it is biased), except for final time step $H$.  We can quantify this bias as follows:

\begin{align*}
\!\!\!\!\!R^{h}+&\gamma Q^*(X^{h+1},\pi(X^{h+1})) = Q^*(x^{h},a^{h})+\nu^h+\gamma Q^*(X^{h+1},\pi(X^{h+1}))-\gamma Q^*(X^{h+1},\pi^*(X^{h+1}))
\end{align*}
where $\nu^h$ is now an unbiased zero mean random variable due to equality of two reward distributions in Eq.~\ref{eq:reward-unbiased} and Eq.~\ref{eq:reward-biased}. 

We use this intuition and definition of noises in order to later turn the learning of model parameters to unbiased Martingale linear regression.

\subsection{Linear Q-function}

In the following, we consider the case when the optimal Q function is a linear transformation of a feature representation vector $\phi(x^h,a^h)\in \R^d$ for any pair of state and actions $x^h,a^h$ at any time step $h$. In other words, at any time step $h$ we have 
\begin{align*}
Q^*(x^h,a^h)=\phi(x^h,a^h)^\top\opt^h,~x^h\in\X^h,~a^h\in\A^h
\end{align*}
To keep the notation simple, while all the policies, Q functions, and feature presentations are functions of time step $h$, we encode the $h$ into the state $x$. 

As mentioned in the section \ref{sec:Theory}, we consider the feature represent and the weight vectors of $\opt^1,\opt^2,\ldots,\opt^H$ satisfy the following conditions;
\begin{align*}
 \|\phi(x^h,a^h)\phi(x^h,a^h)^\top\|_2^2\leq L~,~\textit{and}~\|\opt^1\|_2,\ldots,\|\opt^H\|_2\leq L_\omega,~\forall h\in[H],~x^h\in\X^h,~a^h\in\A^h
\end{align*}

In the following, we denote 
the agent policy at $h$'th time step of $t$'th episode as $\pi_t^h$ . We show how to estimate $\opt^h~,\forall h\in[H]$ using data collected  under $\pi_t^h$. We denote $\hat{\omega}_t^h$ as the estimation of $\opt^h$ for any $h$ at an episode $t$. We show over time for any $h$ our estimations of$\hat{\omega}_t^h$'s concentrate around their true parameters $\opt^h$. We further show how to deploy this concentrations and construct two algorithms, one based on \PSRL, \LinPSRL, and another based on Optimism \OFU~, \LinUCB,  to guaranteed Bayesian and frequentist regret upper bounds respectively. The main body of the following analyses on concentration of measure are based on the  contextual linear bandit literature~\citep{abbasi-yadkori2011improved,chu2011contextual,li2010contextual,rusmevichientong2010linearly,dani2008stochastic,russo2014learning} and self normalized processes~\citep{de2004self,pena2009self}. 

For the following we consider the case where $\gamma=1$ and then show how to extend the results to any discounted case.

\paragraph{Abstract Notation:} We use subscript, e.g., $t$ to represent an event in $t^{th}$ episode and superscript, e.g., $h$ to represent an event at $h^{th}$ time step. For example, $\phi(X_t^h,A_t^h)$ represents the feature vector observed at $h^{th}$ time step of $t^{th}$ episode, in short $\phi_t^h$. Moreover, we denote $\pi_t$ as the agent policy during episode $t$ as the concatenation of $\phi_t^h,~\forall h\in[H]$. Similarly, for the optimal policy, we have $\pi^*$ as the concatenation of ${\phi^*}^h,~\forall h\in[H]$. In addition, we have $\omega_t$ and $\opt$ as the concatenation of some model parameters in episode $t$ and the true parameters.

We define the \textit{modified target value}, condition on $X_t^h=x_t^h,A_t^h=a_t^h$ as follows:
\begin{align*}
\wt\nu_t^h := \phi(x_t^{h},a_t^{h})^\top \opt^{h}+\nu_t^h=
r_t^h+\phi(x_t^{h+1},\pi^*(x_t^{h+1}))^\top \opt^{h+1}
\end{align*}
where we restate the mean zero noise $\nu_t^h$ as:
\begin{align*}
\nu_t^h = r_t^h+\phi(x_t^{h+1},\pi^*(x_t^{h+1}))^\top \opt^{h+1}-\phi(x_t^{h},a_t^{h})^\top \opt^{h}
\end{align*}

\begin{assumption}[Sub-Gaussian random variable]\label{asm:subgaussian}
At time step $h$ of $t^{th}$ episode, conditioned on the event up to time step $h$ of the episode $t$ $\nu_t^h$ is a sub-Gaussian random variable, i.e. there exists a parameter $\sigma\geq 0$ such that $\forall \alpha\in\R^d$
\begin{align*}
\E\left[\exp\left(\alpha^\top\phi_t^h\nu_t^h/\sigma-(\alpha^\top\phi_t^h)^2/2\right)\Big|\F_{t-1}^{h-1}\right]\leq 1
\end{align*}
where $\nu_t^h$ is $\F_{t-1}^h$-measurable, and $\phi(X_t^h,A_t^h)$ is $\F_{t-1}^{h-1}$-measurable.
\end{assumption}

A similar assumption on the noise model is considered in the prior works in linear bandits \citep{abbasi-yadkori2011improved}.

\paragraph{Expected Reward and Return:}The expected reward and expected return are in $[0,1]$.

\paragraph{Regret Definition} Let $V^{\omega}_{\pi}$ denote the value of policy $\pi$ under a model parameter $\omega$. The regret definition (pseudo regret) for the frequentist regret is as follows;
\begin{align*}
\textbf{Reg}_T:&=\E\left[\sum_t^T \left[V^{\opt}_{{\pi^*}}-V^{\opt}_{\pi_t}\right]|\opt\right]
\end{align*}
Where  $\pi_t$ is the agent policy during episode $t$.

When there is a prior over the $\omega^*$  the expected Bayesian regret might be the target of the study. The Bayesian regret is as follows;
\begin{align*}
\textbf{BayesReg}_T:&=\E\left[\sum_t^T \left[V^{\opt}_{{\pi^*}}-V^{\opt}_{\pi_t}\right]\right]
%
%
\end{align*}

\subsection{Optimism: Regret bound of \LinUCB, Algorithm.~\ref{alg:optimism}}\label{sub:optimism}

In optimism we approximate the desired model parameters $\opt^1,\opt^2,\ldots,\opt^H$ up to their high probability confidence set $\C_t^1(\delta),\C_t^2(\delta),\ldots,\C_t^H(\delta)$ where $\opt^h\in\C_t^h(\delta),~\forall h\in[H]$ with probability at least $1-\delta$. In optimism we choose the most optimistic models from these plausible sets. Let $\wt\omega_t^1,\wt\omega_t^2,\ldots,\wt\omega_t^H$ denote the chosen most optimistic parameters during an episode $t$. The most optimistic models for each $h$ is;
\begin{align*}
\wt\omega_t^h = \arg\max_{\omega\in \C_{t-1}^h(\delta)}\max_{a\in\A^h}\phi(x_t^h,a^h_t)\omega 
\end{align*}
and also the most pessimistic models $\check\omega_t^1,\check\omega_t^2,\ldots,\check\omega_t^H$;
\begin{align*}
\check\omega_t^h = \arg\min_{\omega\in \C_{t-1}^h(\delta)}\max_{a\in\A^h}\phi(x_t^h,a^h_t)\omega
\end{align*}

Since in RL, we do not have access to the true model parameters, we do not observe the modified target value. Therefore for eath $t$ and $h$ we define $\wb\nu_t^h$ as the observable \textit{target value}:
\begin{align*}
\wb\nu_t^h = r_t^h+\phi(x_t^{h+1},\check\pi_t(x_t^{h+1}))^\top \check\omega^h_{t-1}
\end{align*}
where $\check\pi_t$ is the policy corresponding to $\check\omega^h_{t},~\forall h$. For $h=H$ we have $\wb\nu_t^H = \wt\nu_t^H=r_t^H$.

The modified target values and the target value have the following relationship 

\begin{align*}
\wb\nu_t^h &= \wb\nu_t^h - \wt\nu_t^h + \wt\nu_t^h\\
&=
r_t^h+\phi(x_t^{h+1},\check\pi_{t-1}(x_t^{h+1}))^\top \check\omega^h_{t-1}-r_t^h-\phi(x_t^{h+1},\pi^*(x_t^{h+1}))^\top \opt^h+\wt\nu_t^h\\
&=
\phi(x_t^{h+1},\check\pi_t(x_t^{h+1}))^\top \check\omega^h_{t-1}-\phi(x_t^{h+1},\pi^*(x_t^{h+1}))^\top \opt^h+\wt\nu_t^h\\
\end{align*}
These definitions of $\wb\nu_t^h$, $\wt\nu_t^h$, and $\nu_t^h$ and their relationship allows us to efficiently construct a self normalized process to estimate the true model parameters.

Let for each $h\in[H]$, $\mathbf{\Phi}_t^h\in\Re^{t\times d}$ denote the row-wised concatenation of $\lbrace\phi_i^h\rbrace_{i=1}^t$, $\bm{\wb\nu}_t^h\in\Re^t$ a column of target values $\lbrace\wb\nu_i^h\rbrace_{i=1}^t$ (observed), $\bm{\wt\nu}_t^h\in\Re^t$ a column of modified target values $\lbrace\wt\nu_i^h\rbrace_{i=1}^t$ (unobserved), $\bm{\nu}_t^h\in\Re^t$ a column of noise in the modified target values $\lbrace\nu_i^h\rbrace_{i=1}^t$ (unobserved), and $\bm{R}_t^h\in\Re^t$ a column of $\lbrace r_i^h\rbrace_{i=1}^t$. Let us restate the following quantities for the self-normalized processes;
\begin{align*}
S_t^h := \sum_i^t \nu_i^h\phi_i^h={\mathbf{\Phi}_t^h}^\top \bm{\nu}^h_t, ~~~~~ \chi_t^h := \sum_{i=1}^t \phi_i^h{\phi_i^h}^\top= {\mathbf{\Phi}_t^h}^\top {\mathbf{\Phi}_t^h}, ~~~~~ \wo\chi_t^h = \chi_t^h+\wt\chi^h
\end{align*}
where $\wt\chi^h$ is a ridge regularization matrix and is usually set to $\lambda I$. 

For the optimal actions and time step $h$, $\rho_\lambda^h$ denotes the spectral bound on;
\begin{align*}
\sqrt{\sum_i^t\|\phi(x_i^{h},\pi^*(x_i^{h}))\|_{{\wo\chi_t^{h}}^{-1}}^2}\leq \rho_\lambda^h,~\forall h,t
\end{align*}
Let $\wb\rho^H_\lambda(\gamma)$ denote the following combination of $\rho_\lambda^h$;
\begin{align*}
\wb\rho^H_\lambda(\gamma) :=  &\gamma^{2(H-1)}\\
&+\gamma^{2(H-2)}\left(1+\gamma \rho_\lambda^H\right)^2\\
&+\gamma^{2(H-3)}\left(1+\gamma \rho_\lambda^{H-1}+\gamma^2 \rho_\lambda^{H-1}\rho_\lambda^{H}\right)^2\\
&+\ldots	\\
&+\gamma^{2(0)}\left(1+\gamma \rho_\lambda^{2}+\ldots+\gamma^{H-1}\rho_\lambda^{2} \ldots\rho_\lambda^{H-1}\rho_\lambda^{H}\right)^2\\
&=\sum_{i=1}^{H}\gamma^{2(H-i)} \left(1+\sum_{j=1}^i\I(j>1) \gamma^{j-1}\prod_{k=1}^{j-1} \rho_\lambda^{H-(i-k)+1}\right)^2
\end{align*}
It is worth noting that a lose upper bound on $\wb\rho^H_\lambda$  when $\lambda=1$ is $\mathcal{O}\left(\left(\max_h\lbrace\rho_\lambda^h\rbrace\right)^H\right)$.

\begin{lemma}[Confidence Intervals]\label{lem:conf}
For all $h\in[H]$, let $\wh\omega_t^h$ denote the estimation of $\opt^h$ given $\mathbf{\Phi}_t^1,\ldots,\mathbf{\Phi}_t^H$ and $\bm{\wb\nu}_t^1,\ldots,\bm{\wb\nu}_t^H$; 
\begin{align*}
\wh\omega_t^h:= \left({\mathbf{\Phi}_t^h}^\top\mathbf{\Phi}_t^h+\lambda I\right)^{-1}{\mathbf{\Phi}_t^h}^\top\bm{\wb\nu}_t^h
\end{align*}
then, with probability at least $1-\delta$
\begin{align*}
\|\wh\omega^h_t-\opt^h\|_{\wo\chi_t^h}\leq\theta_t^h(\delta)&:\sigma\sqrt{2
\log\left(H/\delta\right)
+d\log\left(1+tL^2/\lambda\right)}+\lambda^{1/2}L_\omega+\theta_t^{h+1}(\delta)\rho_\lambda^{h+1}
\end{align*}
and 
\begin{align*}
 \C_t^h(\delta):=\lbrace\omega\in\Re^d:\|\wh\omega_t^h-\omega^h\|_{\wo\chi_t^h}\leq \theta_t^h(\delta)\rbrace
\end{align*}
for all $h\in[H]$. For $h=H$, we set $\theta_t^{H+1}=0,\forall t$ and have
\begin{align*}
\|\wh\omega^H_t-\opt^H\|_{\wo\chi_t^H}\leq\theta_t^H(\delta)&:\sigma\sqrt{2
\log\left(H/\delta\right)
+d\log\left(1+tL^2/\lambda\right)}+\lambda^{1/2}L_\omega
\end{align*}

for all $t$ up to a stopping time.
\end{lemma}

Let $\Theta_T$ denote the event that the confidence bounds in Lemma~\ref{lem:conf} holds at least until $T^{th}$ episode.
\begin{lemma}[Determinant Lemma(Lemma 11 in \citep{abbasi-yadkori2011improved})]\label{lemma:DeterminantLemma}
For a sequence $\phi_t^h$ we have
\begin{align}\label{eq:xx}
\sum_t^T\log\left(1+\|\phi_t^h\|^2_{\wo\chi_{t-1}^{-1}}\right)\leq d\log(\lambda + TL^2/d)
\end{align}
\end{lemma}

Lemma~\ref{lem:conf} states that under event $\Theta_T$, $\|\wh\omega^h_t-\opt^h\|_{\wo\chi^h_t}\leq\theta_t^h(\delta)$. Furthermore, we define state and policy dependent optimistic parameter  $\wt\omega_t(\pi)$ as follows;
\begin{align*}
\wt\omega_t^h(\pi):=\arg\max_{\omega\in \C^h_{t-1}(\delta)}\phi(X^h_t,\pi(X_t^h))^\top\omega
\end{align*}
Following \OFU, Algorithm~\ref{alg:optimism}, we set $\pi_t^h=\wt\pi_t^h,~\forall h$, the optimistic policies. By the definition we have
\begin{align*}
    V^{\wt\omega_t^h(\wt\pi_t^h)}_{\wt\pi_t^h}(X_t^h):=\phi(X_t^h,\wt\pi_t^h(X_t^h))^\top  \wt\omega_t^h(\wt\pi_t^h)\geq V^{\wt\omega_t({\pi^*}^h)}_{{\pi^*}^h}(X_t^h)
\end{align*}
We use this inequality to derive an upper bound for the regret;

\begin{align*}
\textbf{Reg}_T:&=\E\left[\sum_t^T\left[\underbrace{V^{\opt}_{{\pi^*}}(X_t^1)- V^{\opt}_{\wt\pi_t}(X_t^1)}_{\Delta_t^{h=1}}\right]|\opt\right]\\
&\leq\E \left[\sum_t^T\left[V^{\wt\omega_t^1(\wt\pi_t^1)}_{\wt\pi_t^1}(X_t^1)-V^{\wt\omega_t^1({\pi^*}^1)}_{{\pi^*}^1}(X_t^1)+V^{\opt}_{{\pi^*}}(X_t^1)- V^{\opt}_{\wt\pi_t}(X_t^1)\right]|\opt\right]\\
&=\E\left[\sum_t^T \left[V^{\wt\omega_t^1(\wt\pi_t^1)}_{\wt\pi_t^1}(X_t^1)-V^{\opt}_{\wt\pi_t}(X_t^1)+\underbrace{V^{\opt}_{{\pi^*}}(X_t^1)-V^{\wt\omega_t^1({\pi^*}^1)}_{{\pi^*}^1}(X_t^1)}_\text{$\leq 0$}\right]|\opt\right]\\
\end{align*}
Resulting in 
\begin{align*}
\textbf{Reg}_T&\leq\E\left[\sum_t^T\left[ V^{\wt\omega_t^1(\wt\pi_t^1)}_{\wt\pi_t^1}(X_t^1)-V^{\opt}_{\wt\pi_t}(X_t^1)\right]|\opt\right]
\end{align*}
Let us defined $V^{\opt}_{\pi}(X_t^{h'};h)$ as the value function at $X_t^{h'}$, following policies in $\pi$ for the first $h$ time steps then switching to the optimal policies.
\begin{align*}
\textbf{Reg}_T&\leq\E\left[\sum_t^T \left[V^{\wt\omega^1_t(\wt\pi^1_t)}_{\wt\pi^1_t}(X_t^1)-V^{\opt}_{\wt\pi_t}(X_t^1)\right]|\opt\right]\\
&=\E\left[\sum_t^T \left[V^{\wt\omega_t^1(\wt\pi_t^1)}_{\wt\pi_t^1}(X_t^1)-V^{\opt}_{\wt\pi_t}(X_t^1;1)+V^{\opt}_{\wt\pi_t}(X_t^1;1)-V^{\opt^1}_{\wt\pi_t^1}(X_t^1)\right]|\opt\right]
\end{align*}
Given the linear model of the $Q$ function we have;
\begin{align*}
\textbf{Reg}_T&\leq\E\left[\sum_t^T\left[ V^{\wt\omega_t^1(\wt\pi_t^1)}_{\wt\pi_t^1}(X_t^1)-V^{\opt}_{\wt\pi_t}(X_t^1;1)+V^{\opt}_{\wt\pi_t}(X_t^1;1)-V^{\opt}_{\wt\pi_t}(X_t^1)\right]|\opt\right]\\
&=\E\left[\sum_t^T \left[\phi(X_t^1,\wt\pi_t^1(X_t^1))^\top\wt\omega_t^1(\wt\pi_t^1)-\phi(X_t^1,\wt\pi_t^1(X_t^1))^\top\opt^1+V^{\opt}_{\wt\pi_t}(X_t^1;1)-V^{\opt}_{\wt\pi_t}(X_t^1)\right]|\opt\right]\\
&=\E\left[\sum_t^T \left[\phi(X_t^1,\wt\pi_t^1(X_t^1))^\top\left(\wt\omega_t^1(\wt\pi_t)-\opt^1\right)+\underbrace{V^{\opt}_{\wt\pi_t}(X_t^1;1)-V^{\opt}_{\wt\pi_t}(X_t^1)}_{(\Delta_t^{h=2})}\right]|\opt\right]
\end{align*}
For $\Delta_t^{h=2}$ we deploy the similar decomposition and upper bound as $\Delta_t^{h=1}$
\begin{align*}
\Delta_t^2:=V^{\opt}_{\wt\pi_t}(X_t^1;1)-V^{\opt}_{\wt\pi_t}(X_t^1)
\end{align*}
Since for both of $V^{\opt}_{\wt\pi_t}(X_t^1;1)$ and $V^{\opt}_{\wt\pi_t}(X_t^1)$ we follow the same policy on the same model for 1 time step, the reward at the first time step has the same distribution, therefore we have;
\begin{align*}
\Delta_t^{h=2}&=E\left[V^{\opt}_{\wt\pi_t}(X_t^2;1)-V^{\opt}_{\wt\pi_t}(X_t^2)|X_t^1,A_t^1=\wt\pi_t(X_t^1),\opt\right]
\end{align*}
resulting in 
\begin{align*}
\Delta_t^{h=2}&=E\left[V^{\opt}_{\wt\pi_t}(X_t^2;1)-V^{\opt}_{\wt\pi_t}(X_t^2)|X_t^1,A_t^1=\wt\pi_t(X_t^1),\opt\right]\\
&\leq\E\left[\phi(X_t^2,\wt\pi_t^2(X_t^2))^\top\left(\wt\omega_t^2(\wt\pi^2_t)-\opt^2\right)+\underbrace{V^{\opt}_{\wt\pi_t}(X_t^2;2)-V^{\opt}_{\wt\pi_t}(X_t^2)}_{(\Delta_t^{h=3})}|X_t^1,A_t^1=\wt\pi_t(X_t^1),\opt\right]
\end{align*}
Similarly we can defined $\Delta_t^{h=3}, \ldots \Delta_t^{h=H}$. Therefore;

\begin{align}\label{eq:onesnorm}
\textbf{Reg}_T&\leq\E\left[\sum_t^T\sum_h^H \phi(X_t^h,\wt\pi_t(X_t^h))^\top\left(\wt\omega^h_t(\wt\pi^h_t)-\opt^h\right)|\opt\right]\nonumber\\
&=\E\left[\sum_t^T\sum_h^H \phi(X_t^h,\wt\pi_t(X_t^h))^\top{\wo\chi_{t-1}^h}^{-1/2}{\wo\chi_{t-1}^h}^{1/2}\left(\wt\omega^h_t(\wt\pi^h_t)-\opt^h\right)|\opt\right]\nonumber\\
&\leq\E\left[\sum_t^T\sum_h^H \|\phi(X_t^h,\wt\pi_t(X_t^h))\|_{{\wo\chi_{t-1}^h}^{-1}}\|\wt\omega_t^h(\wt\pi_t^h)-\opt^h\|_{\wo\chi_{t-1}^h}|\opt\right]\nonumber\\
&\leq\E\left[\sum_t^T\sum_h^H \|\phi(X_t^h,\wt\pi_t(X_t^h))\|_{{\wo\chi_{t-1}^h}^{-1}}2\theta^h_{t-1}(\delta)|\opt\right]\nonumber\\
\end{align}

Since the maximum expected cumulative reward, condition on states of a episode is at most $1$, we have;
\begin{align*}
\textbf{Reg}_T&\leq\E\left[\sum_t^T\sum_h^H \min\lbrace\|\phi(X_t^h,\wt\pi_t(X_t^h))\|_{{\wo\chi_{t-1}^h}^{-1}}2\theta^h_{t-1}(\delta),1\rbrace|\opt\right]\nonumber\\
&\leq\E\left[\sum_t^T\sum_h^H \min\lbrace\|\phi(X_t^h,\wt\pi_t(X_t^h))\|_{{\wo\chi_{t-1}^h}^{-1}}2\theta^h_{T}(\delta),1\rbrace|\opt\right]\nonumber\\
&\leq\E\left[\sum_t^T\sum_h^H2\theta^h_{T}(\delta) \min\lbrace\|\phi(X_t^h,\wt\pi_t(X_t^h))\|_{{\wo\chi_{t-1}^h}^{-1}},1\rbrace|\opt\right]\nonumber\\
\end{align*}
where we exploit the fact that $\theta_t^h(\delta)$ is an increasing function of $t$ and consider the harder case where $\theta^h_{T}(\delta)\geq 1$.

Moreover, at time $T$, we can use Jensen's inequality

\begin{align}\label{eq:OR1}
\textbf{Reg}_T&\leq2\E\left[\sqrt{TH\sum_h^H2\theta^h_{T}(\delta)^2\sum_t^T \min\lbrace\|\phi(X_t^h,\wt\pi_t(X_t^h))\|^2_{{\wo\chi_{t-1}^h}^{-1}},1\rbrace|\opt}\right]\nonumber\\
\end{align}

Now, using the fact that for any scalar $\alpha$ such that $0\leq \alpha\leq 1$, then $\alpha\leq 2\log(1+\alpha)$ we can rewrite the latter part of Eq.~\ref{eq:OR1}
\begin{align*}
\sum_t^T\min\lbrace\|\phi(X_t^h,\wt\pi_t(X_t^h))\|^2_{{\wo\chi_{t-1}^h}^{-1}},1\rbrace\leq 2\sum_t^T\log\left(1+\|\phi(X_t^h,\wt\pi_t(X_t^h))\|^2_{{\wo\chi_{t-1}^h}^{-1}}\right)
\end{align*}
By applying the Lemma~\ref{lemma:DeterminantLemma} and substituting the RHS of Eq.~\ref{eq:xx} into Eq.~\ref{eq:OR1}, we get 
\begin{align}\label{eq:OR2}
\textbf{Reg}_T&\leq2\E\left[\sqrt{TH\sum_h^H2\theta^h_{T}(\delta)^2d\log(\lambda + TL^2/d)|\opt}\right]\nonumber\\
&\leq2\E\left[\sqrt{THd\log(\lambda + TL^2/d)\sum_h^H2\left(\sigma\sqrt{2
\log\left(H/\delta\right)
+d\log\left(1+TL^2/\lambda\right)}+\lambda^{1/2}L_\omega\theta^{h+1}_{T}(\delta){\rho_\lambda^{h+1}}\right)^2|\opt}\right]\nonumber\\
\end{align}
and deriving the upper bounds
\begin{align}\label{eq:regOPT}
\textbf{Reg}_T
&\leq
2\left(\sigma\sqrt{2
\log\left(1/\delta\right)
+d\log\left(1+TL^2/\lambda\right)}+\lambda^{1/2}L_\omega\right)\sqrt{2\wb\rho^H_\lambda(1) TH d\log(\lambda + TL^2/d)}
\end{align}
with probability at least $1-\delta$. If we set $\delta=1/T$ then the probability that the event $\Theta_T$ holds is $1-1/T$ and we get regret of at most the RHS of Eq.~\ref{eq:regOPT}, otherwise with probability at most $1/T$ we get maximum regret of $T$, therefore 
\begin{align*}
\textbf{Reg}_T
&\leq 1+
2\left(\sigma\sqrt{2
\log\left(1/\delta\right)
+d\log\left(1+TL^2/\lambda\right)}+\lambda^{1/2}L_\omega\right)\sqrt{2\wb\rho^H_\lambda(1) T Hd\log(\lambda + TL^2/d)}
\end{align*}
For the case of discounted reward, substituting $\wb\rho^H_\lambda(\gamma)$ instead of $\wb\rho^H_\lambda(1)$ results in the theorem statement.

\subsection{Bayesian Regret of Alg.~\ref{alg:psrl}}\label{sec:proof}
The analysis developed in the previous section, up to some minor modification, e.g., change of strategy to \PSRL,  directly applies to Bayesian regret bound, with a farther expectation over models.

When there is a prior over the $\opt^h,~\forall h$  the expected Bayesian regret might be the target of the study. 
\begin{align*}
\textbf{BayesReg}_T:&=\E\left[\sum_t^T \left[V^{\opt}_{{{\pi^*}}}-V^{\opt}_{\pi_t}\right]\right]\\
&=\E\left[\sum_t^T\left[ V^{\opt}_{{\pi^*}}-V^{\opt}_{\pi_t}\big|\mathcal{H}_t\right]\right]
\end{align*}
Here $\mathcal{H}_t$ is a multivariate random sequence which indicates history at the beginning of episode $t$ and  $\pi_t^h,\forall h$ are the policies following \PSRL. For the remaining $\pi_t^h$ denotes the \PSRL policy at each time step $h$. As it mentioned in the Alg.~\ref{alg:psrl}, at the beginning of an episode, we draw $\omega_t^h~\forall h$ from the posterior and the corresponding policies are;
\begin{align*}
    \pi_t^h(X_t^h):=\arg\max_{a\in\A}\phi(X_t^h,a)^\top\omega_t^h
\end{align*}
Condition on the history $\mathcal{H}_t$, i.e., the experiences by following the agent policies $\pi_{t'}^h$ for each episode $t'\leq t$, we estimate the $\wh\omega_t^h$ as follows; 
\begin{align*}
\wh\omega_t^h:= \left({\mathbf{\Phi}_t^h}^\top\mathbf{\Phi}_t^h+\lambda I\right)^{-1}{\mathbf{\Phi}_t^h}^\top\bm{\wb\nu}_t^h
\end{align*}
Lemma~\ref{lem:conf} states that under event $\Theta_T$, $\|\wh\omega_t^h-\opt^h\|_{\wo\chi_t}\leq\theta^h_t(\delta),~\forall h$. Conditioned on $\mathcal{H}_t$, the $\omega_t^h$ and $\opt^h$ are equally distributed, then we have
\begin{align*}
 \E\left[V^{\wt\omega_t^h({\pi^*}^h)}_{{\pi^*}^h}=\phi(X_t^h,{\pi^*}^h(X_t^h))^\top \wt\omega_t^h({\pi^*}^h)\big|\mathcal{H}_t\right]=\E\left[ V^{\wt\omega_t^h(\pi_t^h)}_{\pi_t^h}=\phi(X_t^h,{\pi_t^h}(X_t^h))^\top \wt\omega_t^h(\pi_t^h)\big|\mathcal{H}_t\right]   
\end{align*}
Therefore, for the regret we have

\begin{align*}
\textbf{BayesReg}_T:&=\sum_t^T\E\left[\underbrace{V^{\opt}_{{\pi^*}}(X_t^1)- V^{\opt}_{\pi_t}(X_t^1)}_{\Delta_t^{h=1}}|\mathcal{H}_t\right]\\
&=\sum_t^T\E\left[V^{\wt\omega_t^1(\pi_t^1)}_{\pi_t^1}(X_t^1)-V^{\wt\omega_t^1({\pi^*}^1)}_{{\pi^*}^1}(X_t^1)+V^{\opt}_{{\pi^*}}(X_t^1)- V^{\opt}_{\pi_t}(X_t^1)|\mathcal{H}_t\right]\\
&=\sum_t^T\E\left[ V^{\wt\omega_t^1(\pi_t^1)}_{\pi_t^1}(X_t^1)-V^{\opt}_{\pi_t}(X_t^1)+\underbrace{V^{\opt}_{{\pi^*}}(X_t^1)-V^{\wt\omega_t^1({\pi^*}^1)}_{{\pi^*}^1}(X_t^1)}_\text{$\leq 0$}|\mathcal{H}_t\right]\\
\end{align*}
Resulting in 
\begin{align*}
\textbf{BayesReg}_T&\leq\sum_t^T\E\left[ V^{\wt\omega_t^1(\pi_t^1)}_{\pi_t^1}(X_t^1)-V^{\opt}_{\pi_t}(X_t^1)|\mathcal{H}_t\right]
\end{align*}

Similar to optimism and defining $V^{\omega}_{\pi}(X_t^{h'};h)$ we have;
\begin{align*}
\textbf{BayesReg}_T&\leq\E\left[\sum_t^T \left[V^{\wt\omega_t^1(\pi_t^1)}_{\pi_t^1}(X_t^1)-V^{\opt}_{\pi_t}(X_t^1)|\mathcal{H}_t\right]\right]\\
&=\E\left[\sum_t^T \left[V^{\wt\omega_t^1(\pi_t^1)}_{\pi_t^1}(X_t^1)-V^{\opt}_{\pi_t}(X_t^1;1)+V^{\opt}_{\pi_t}(X_t^1;1)-V^{\opt}_{\pi_t}(X_t^1)|\mathcal{H}_t\right]\right]\\
&\leq\E\left[\sum_t^T\left[ \phi(X_t^1,\pi_t^1(X_t^1))^\top\wt\omega_t^1(\pi_t^1)-\phi(X_t^1,\pi_t^1(X_t^1))^\top\opt^1+V^{\opt}_{\pi_t}(X_t^1;1)-V^{\opt}_{\pi_t}(X_t^1)|\mathcal{H}_t\right]\right]\\
&=\E\left[\sum_t^T \left[\phi(X_t^1,\pi_t^1(X_t^1))^\top\left(\wt\omega_t^1(\pi_t^1)-\opt^1\right)+\underbrace{V^{\opt}_{\pi_t}(X_t^1;1)-V^{\opt}_{\pi_t}(X_t^1)}_{(\Delta_t^{h=2})}|\mathcal{H}_t\right]\right]
\end{align*}

For $\Delta_t^{h=2}$ we deploy the similar decomposition and upper bound as $\Delta_t^{h=1}$
\begin{align*}
\Delta_t^2:=V^{\opt}_{\pi_t}(X_t^1;1)-V^{\opt}_{\pi_t}(X_t^1)
\end{align*}
Since for both of $V^{\opt}_{\pi_t}(X_t^1;1)$ and $V^{\opt}_{\pi_t}(X_t^1)$ we follow the same distribution over policies and models for 1 time step, the reward at the first time step has the same distribution, therefore we have;
\begin{align*}
\Delta_t^{h=2}&=E\left[V^{\opt}_{\pi_t}(X_t^2;1)-V^{\opt}_{\pi_t}(X_t^2)|X_t^1,A_t^1=\pi_t(X_t^1),\mathcal{H}_t\right]
\end{align*}
resulting in 
\begin{align*}
\Delta_t^{h=2}&=E\left[V^{\opt}_{\pi_t}(X_t^2;1)-V^{\opt}_{\pi_t}(X_t^2)|X_t^1,A_t^1=\pi_t(X_t^1),\mathcal{H}_t\right]\\
&\leq\E\left[\phi(X_t^2,\pi_t^2(X_t^2))^\top\left(\wt\omega_t^2(\pi_t^2)-\opt^2\right)+\underbrace{V^{\opt^2}_{\pi_t}(X_t^2;2)-V^{\opt}_{\pi_t}(X_t^2)}_{(\Delta_t^{h=3})}|X_t^1,A_t^1=\pi_t(X_t^1),\mathcal{H}_t\right]
\end{align*}
Similarly we can defined $\Delta_t^{h=3}, \ldots \Delta_t^{h=H}$. The condition on $\mathcal{H}_t$ was required to come up with the mentioned decomposition through $\Delta_t^{h}$ and it is not needed anymore, therefore;

\begin{align}\label{eq:bayesunit}
\textbf{BayesReg}_T&\leq\E\left[\sum_t^T\sum_h^H \phi(X_t^h,\wt\pi_t(X_t^h))^\top\left(\wt\omega^h_t(\wt\pi^h_t)-\opt^h\right)\right]\nonumber\\
&=\E\left[\sum_t^T\sum_h^H \phi(X_t^h,\wt\pi_t(X_t^h))^\top{\wo\chi_{t-1}^h}^{-1/2}{\wo\chi_{t-1}^h}^{1/2}\left(\wt\omega^h_t(\wt\pi^h_t)-\opt^h\right)\right]\nonumber\\
&\leq\E\left[\sum_t^T\sum_h^H \|\phi(X_t^h,\wt\pi_t(X_t^h))\|_{{\wo\chi_{t-1}^h}^{-1}}\|\wt\omega_t^h(\wt\pi_t^h)-\opt^h\|_{\wo\chi_{t-1}^h}\right]\nonumber\\
&\leq\E\left[\sum_t^T\sum_h^H \|\phi(X_t^h,\wt\pi_t(X_t^h))\|_{{\wo\chi_{t-1}^h}^{-1}}2\theta^h_{t-1}(\delta)\right]\nonumber\\
\end{align}
Again similar to optimism we have the maximum expected cumulative reward condition on states of a episode is at most $1$, we have;
\begin{align*}
\textbf{BayesReg}_T&\leq\E\left[\sum_t^T\sum_h^H \min\lbrace\|\phi(X_t^h,\wt\pi_t(X_t^h))\|_{{\wo\chi_{t-1}^h}^{-1}}2\theta^h_{t-1}(\delta),1\rbrace\right]\nonumber\\
&\leq\E\left[\sum_t^T\sum_h^H2\theta^h_{t-1}(\delta) \min\lbrace\|\phi(X_t^h,\wt\pi_t(X_t^h))\|_{{\wo\chi_{t-1}^h}^{-1}},1\rbrace\right]\nonumber\\
\end{align*}

Moreover, at time $T$, we can use Jensen's inequality, exploit the fact that $\theta_t(\delta)$ is an increasing function of $t$ and have

\begin{align}\label{eq:regpsrl}
\textbf{BayesReg}_T&\leq2\E\left[\sqrt{TH\sum_h^H2\theta^h_{T}(\delta)^2\sum_t^T \min\lbrace\|\phi(X_t^h,\wt\pi_t(X_t^h))\|^2_{{\wo\chi_{t-1}^h}^{-1}},1\rbrace}\right]\nonumber\\
&\leq2\E\left[\sqrt{TH\sum_h^H2\theta^h_{T}(\delta)^2d\log(\lambda + TL^2/d)}\right]\nonumber\\
&\leq2\E\left[\sqrt{THd\log(\lambda + TL^2/d)\sum_h^H4\left(\sigma\sqrt{2
\log\left(H/\delta\right)
+d\log\left(1+TL^2/\lambda\right)}+\lambda^{1/2}L_\omega\theta^{h+1}_{T}(\delta)^2{\rho_\lambda^{h+1}}\right)^2}\right]\nonumber\\
&\leq
2\left(\sigma\sqrt{2
\log\left(1/\delta\right)
+d\log\left(1+TL^2/\lambda\right)}+\lambda^{1/2}L_\omega\right)\sqrt{2\wb\rho^H_\lambda(1) THd\log(\lambda + TL^2/d)}
\end{align}

Under event $\Theta_T$ which holds with probability at least $1-\delta$. If we set $\delta=1/T$ then the probability that the event $\Theta$ holds is $1-1/T$ and we get regret of at most the RHS of Eq.~\ref{eq:regpsrl}, otherwise with probability at most $1/T$ we get maximum regret of $T$, therefore 
\begin{align*}
\textbf{BayesReg}_T\leq 1+
2\left(\sigma\sqrt{2
\log\left(1/\delta\right)
+d\log\left(1+TL^2/\lambda\right)}+\lambda^{1/2}L_\omega\right)\sqrt{2\wb\rho^H_\lambda(1) TH d\log(\lambda + TL^2/d)}
\end{align*}
For the case of discounted reward, substituting $\wb\rho^H_\lambda(\gamma)$ instead of $\wb\rho^H_\lambda(1)$ results in the theorem statement.

%

\subsection{Proof of Lemmas}

\begin{lemma}\label{lem:M_t}
Let $\alpha\in\R^d$ be any vector and for any $t\geq0$ define
\begin{align*}
M_t^h(\alpha) :=\exp\left(\sum_i^t\left[\frac{\alpha^\top{{\phi}^h_i}\nu^h_i}{\sigma}-\frac{1}{2}\|{{\phi^h_i}^\top}\alpha\|^2_2\right]\right),~\forall h
\end{align*}
Then, for a stopping time under the filtration $\{\F_t^h\}_{t=0}^\infty$, $M_\tau^h(\alpha)\leq 1$.
\end{lemma}

\begin{proof}~Lemma~\ref{lem:M_t}

We first show that $\{M_t^h(\alpha)\}_{t=0}^\infty$ is a supermartingale sequence. Let
\begin{align*}
D_i^h(\alpha) = \exp\left(\frac{\alpha^\top {{\phi}^h_i}\nu^h_i}{\sigma}-\frac{1}{2}\|{{\phi}^h_i}^\top\alpha\|^2_2\right)
\end{align*}
Therefore, we can rewrite $\E\left[M^h_t(\alpha)\right]$ as follows:
\begin{align*}
\!\!\!\!\!\E\left[M_t^h(\alpha)|\F_{t-1}^h\right] = \E\left[D_1^h(\alpha)\ldots D_{t-1}^h(\alpha) D_{t}^h(\alpha)|\F^h_{t-1}\right] = D_1^h(\alpha)\ldots D_{t-1}^h(\alpha)\E\left[ D_{t}^h(\alpha)|\F^h_{t-1}\right]\leq M_{t-1}^h(\alpha)
\end{align*}
The last inequality follows since $\E\left[ D_{t}^h(\alpha)|\F_{t-1}^h\right]\leq 1$ due to Assumption~\ref{asm:subgaussian}, therefore since for the first time step $\E\left[M_1^h(\alpha)\right]\leq 1$, then $\E\left[M_t^h(\alpha)\right]\leq 1$. For a stopping time $\tau$, Define a variable $\wo{M}_t^\alpha = M_{\min\{t,\tau\}}^h(\alpha)$ and since $\E\left[M_\tau^h(\alpha)\right]= \E\left[\lim\inf_{t\rightarrow\infty}\wo{M}_t^h(\alpha)\right]\leq \lim\inf_{t\rightarrow \infty}\E\left[\wo{M}_t^h(\alpha)\right]\leq 1$, therefore the Lemma~\ref{lem:M_t} holds.
\end{proof}

\begin{lemma} \label{lem:self}[Extension to Self-normalized bound in \cite{abbasi-yadkori2011improved}] For a stopping time $\tau$ and filtration $\{\F_t^h\}_{t=0,h=1}^{\infty,h}$, with probability at least $1-\delta$
\begin{align*}
\|S_\tau^h\|_{{\wo\chi_\tau^h}^{-1}}^2\leq 2\sigma^2\log(\frac{\det\left(\wo\chi_\tau^h\right)^{1/2}\det(\wt\chi^h)^{-1/2}}{\delta})
\end{align*}
\end{lemma}
\begin{proof}~of Lemma~\ref{lem:self}.
Given the definition of the parameters of self-normalized process, we can rewrite $M_t^h(\alpha)$ as follows;
\begin{align*}
M_t^h(\alpha) = \exp\left(\frac{\alpha^\top S^h_t}{\sigma}-\frac{1}{2}\|\alpha\|^2_{\chi_t^h}\right)
\end{align*}
Consider $\Omega^h$ as a Gaussian random vector and $f(\Omega^h=\alpha)$ denotes the density with covariance matrix of ${\wo \chi^h}^{-1}$. Define $M_t^h:=\E\left[M_t^h(\Omega^h)|\F_\infty\right]$. Therefore we have $\E\left[M_t^h\right]=\E\left[\E\left[M_t^h(\Omega^h)|\Omega^h\right]\right]\leq 1$. Therefore
\begin{align*}
M_t^h=&\int_{\Re^d}\exp\left(\frac{\alpha^\top S_t^h}{\sigma}-\frac{1}{2}\|\alpha\|^2_{\chi_t^h}\right)f(\alpha)d\alpha\\
=&\int_{\Re^d}\exp\left(\frac{1}{2}\|\alpha-{\chi_t^h}^{-1}S_t/\sigma\|^2_{\chi_t^h}+\frac{1}{2}\|S_t^h/\sigma\|^2_{{\chi_t^h}^{-1}}\right)f(\alpha)d\alpha\\
=&\sqrt{\frac{\det\left(\wt\chi^h\right)}{\left(2\pi\right)^d}}\exp\left(\frac{1}{2}\|S_t^h/\sigma\|_{{\chi_t^h}^{-1}}\right)\int_{\Re^d}\exp\left(\frac{1}{2}\|\alpha-{\chi_t^h}^{-1}S_t^h/\sigma\|^2_{\chi_t^h}+\frac{1}{2}\|\alpha\|^2_{\wt\chi^h}\right)d\alpha
\end{align*}
Since $\chi_t^h$ and $\wt\chi^h$ are positive semi definite and positive definite respectively, we have 
\begin{align*}
\|\alpha-{\chi_t^h}^{-1}S^h_t/\sigma\|^2_{\chi_t^h}+\|\alpha\|^2_{\wt\chi^h}&= \|\alpha-\left(\wt\chi^h+{\chi_t^h}^{-1}\right)S_t^h/\sigma\|^2_{\wt\chi^h+\chi_t^h}+\|{\chi_t^h}^{-1}S^h_t/\sigma\|^2_{\chi^h_t}-\|S^h_t/\sigma\|^2_{\left(\wt\chi^h+\chi_t^h\right)^{-1}}\\
&=\|\alpha-\left(\wt\chi^h+{\chi_t^h}^{-1}\right)S^h_t/\sigma\|^2_{\wt\chi^h+\chi^h_t}+\|\S_t^h/\sigma\|^2_{{\chi_t^h}^{-1}}-\|S_t^h/\sigma\|^2_{\left(\wt\chi^h+\chi_t^h\right)^{-1}}
\end{align*}
Therefore, 
\begin{align*}
M_t^h=&\sqrt{\frac{\det\left(\wt\chi^h\right)}{\left(2\pi\right)^d}}\exp\left(\frac{1}{2}\|S_t^h/\sigma\|_{\left(\wt\chi^h+\chi^h_t\right)^{-1}}\right)\int_{\Re^d}\exp\left(\frac{1}{2}\|\alpha-\left(\wt\chi^h+{\chi_t^h}^{-1}\right)S^h_t/\sigma\|^2_{\wt\chi^h+\chi_t^h}\right)d\alpha\\
=&\left(\frac{\det\left(\wt\chi^h\right)}{\det\left(\wt\chi^h+\chi_t^h\right)}\right)^{1/2}\exp\left(\frac{1}{2}\|S_t/\sigma\|^2_{\left(\wt\chi^h+\chi_t^h\right)^{-1}}\right)
\end{align*}
Since $\E\left[M_\tau^h\right]\leq 1$ we have
\begin{align*}
\Prob\left(\|S_\tau^h\|^2_{\left(\wo\chi_\tau^h\right)^{-1}}\leq 2\sigma\log\left(\frac{\det\left(\wo\chi^h_\tau\right)^{1/2}}{\delta\det\left(\wt\chi^h\right)^{1/2}}\right)\right)
\leq \frac{\E\left[\exp\left(\frac{1}{2}\|S_\tau^h/\sigma\|^2_{\left(\wo\chi_\tau^h\right)^{-1}}\right)\right]}{\frac{\left(\det\left(\wo\chi_\tau^h\right)\right)^{1/2}}{\delta\left(\det\left(\wt\chi^h\right)\right)^{1/2}}}\leq \delta
\end{align*}
Where the Markov inequality is deployed for the final step. 
The stopping is considered to be the time step as the first time in the sequence when the concentration in the Lemma~\ref{lem:self} does not hold. 
\end{proof}

\begin{proof}[Proof of Lemma \ref{lem:conf}]
Consider the following estimator $\wh\omega_t^h$:
\begin{align*}
\wh\omega_t^h&= \left({\mathbf{\Phi}_t^h}^\top\mathbf{\Phi}_t^h+\lambda I\right)^{-1}\left({\mathbf{\Phi}_t^h}^\top\bm{\wb\nu}_t^h\right)\\
%
%
&=\left({\mathbf{\Phi}_t^h}^\top\mathbf{\Phi}_t^h+\lambda I\right)^{-1}\left({\mathbf{\Phi}_t^h}^\top\left( \bm{\wb\nu}_t^h-\bm{\wt\nu}_t^h+\bm{\wt\nu}_t^h\right)\right)\\
&= \left({\mathbf{\Phi}_t^h}^\top\mathbf{\Phi}_t^h+\lambda I\right)^{-1}\left({\mathbf{\Phi}_t^h}^\top\mathbf{\Phi}_t^h+\lambda I\right)\opt^h\\
&+\left({\mathbf{\Phi}_t^h}^\top\mathbf{\Phi}_t^h+\lambda I\right)^{-1}\left({\mathbf{\Phi}_t^h}^\top\bm{\nu}^h_t\right)\\
&+\left({\mathbf{\Phi}_t^h}^\top\mathbf{\Phi}_t^h+\lambda I\right)^{-1}\left({\mathbf{\Phi}_t^h}^\top(\bm{\wb\nu}_t^h-\bm{\wt\nu}_t^h)\right)\\
&-\lambda\left({\mathbf{\Phi}_t^h}^\top\mathbf{\Phi}_t^h+\lambda I\right)^{-1}\opt^h\\
\end{align*}
therefore, for any vector ${\zeta^h}\in\Re^d$
\begin{align*}
{\zeta^h}^\top\wh\omega_t^h-{\zeta^h}^\top\opt^h&={\zeta^h}^\top\left({\mathbf{\Phi}_t^h}^\top\mathbf{\Phi}_t^h+\lambda I\right)^{-1}\left({\mathbf{\Phi}_t^h}^\top\bm{\nu}_t^h\right)\\
&+
{\zeta^h}^\top\left({\mathbf{\Phi}_t^h}^\top\mathbf{\Phi}_t^h+\lambda I\right)^{-1}\left({\mathbf{\Phi}_t^h}^\top(\bm{\wb\nu}_t^h-\bm{\wt\nu}_t^h)\right)\\
&-{\zeta^h}^\top\lambda\left({\mathbf{\Phi}_t^h}^\top\mathbf{\Phi}_t^h+\lambda I\right)^{-1}\opt^h
\end{align*}
As a results, 
applying Cauchy-Schwarz inequality and inequalities $\|\opt^h\|_{{\wo\chi_t^h}^{-1}}\leq \frac{1}{\lambda\left(\wo\chi_t^h\right)}\|\opt^h\|_2\leq \frac{1}{\lambda}\|\opt^h\|_2$ we get

\begin{align*}
|{\zeta^h}^\top\wh\omega_t^h-{\zeta^h}^\top\opt^h|&\leq\|{\zeta^h}\|_{{\wo\chi_t^h}^{-1}}\|{\mathbf{\Phi}_t^h}^\top\bm{\nu}^h_t\|_{{\wo\chi_t^h}^{-1}}
+
\|{\zeta^h}\|_{{\wo\chi_t^h}^{-1}}\|{\mathbf{\Phi}_t^h}^\top(\bm{\wb\nu}^h_t
 -\bm{\wt\nu}^h_t)\|_{{\wo\chi_t^h}^{-1}}
+\lambda\|{\zeta^h}\|_{{\wo\chi_t^h}^{-1}}\|\opt^h\|_{{\wo\chi_t^h}^{-1}}\\
&\leq \|{\zeta^h}\|_{{\wo\chi_t^h}^{-1}}\left(\|{\mathbf{\Phi}_t^h}^\top\bm{\nu}_t^h\|_{{\wo\chi_t^h}^{-1}}+\|{\mathbf{\Phi}_t^h}^\top(\bm{\wb\nu}_t^h-\bm{\wt\nu}^h_t)\|_{{\wo\chi_t^h}^{-1}}+\lambda^{1/2}\|\opt^h\|_2\right)
\end{align*}
where applying self normalized Lemma \ref{lem:self}, for all $h$ with probability at least $1-H\delta$
\begin{align*}
|{\zeta^h}^\top\wh\omega_t^h-{\zeta^h}^\top\opt^h|\leq \|{\zeta^h}\|_{{\wo\chi_t^h}^{-1}}\left(2\sigma\log\left(\frac{\det\left(\wo\chi_t^h\right)^{1/2}}{\det\left(\wt\chi\right)^{1/2}}\right)+\lambda^{1/2}L_\omega+\|{\mathbf{\Phi}_t^h}^\top(\bm{\wb\nu}_t^h-\bm{\wt\nu}^h_t)\|_{{\wo\chi_t^h}^{-1}}\right)
\end{align*}
hold for any ${\zeta^h}$. By plugging in ${\zeta^h}=\wo\chi_t^h\left(\wh\omega_t^h-\opt^h\right)$ 
we get the following;
\begin{align*}
\|\wh\omega_t^h-\opt^h\|_{\wo\chi_t^h}\leq\theta_t^h(\delta)= \sigma\sqrt{2\log\left(\frac{\det(\wo\chi_t^h)^{1/2}\det(\lambda I)^{-1/2}}{\delta}\right)}+\lambda^{1/2}L_\omega+\|{\mathbf{\Phi}_t^h}^\top(\bm{\wb\nu}_t^h-\bm{\wt\nu}^h_t)\|_{{\wo\chi_t^h}^{-1}}
\end{align*}
For $h=H$ the last term in the above equation $\|{\mathbf{\Phi}_t^H}^\top(\bm{\wb\nu}_t^H-\bm{\wt\nu}^H_t)\|_{{\wo\chi_t^H}^{-1}}$ is zero therefore we have;

\begin{align*}
\|\wh\omega_t^H-\opt^H\|_{\wo\chi_t^H}\leq\theta_t^H(\delta)= \sigma\sqrt{2\log\left(\frac{\det(\wo\chi_t^H)^{1/2}\det(\lambda I)^{-1/2}}{\delta}\right)}+\lambda^{1/2}L_\omega
\end{align*}
For $h<H$ we need to account for the bias introduced by $\|{\mathbf{\Phi}_t^h}^\top(\bm{\wb\nu}_t^h-\bm{\wt\nu}^h_t)\|_{{\wo\chi_t^h}^{-1}}$. Due to the pesimism we have
\begin{align*}
\phi(x_t^{h+1},\pi^*(x_t^{h+1}))^\top \opt^{h+1}\geq \phi(x_t^{h+1},\check\pi_t(x_t^{h+1}))^\top \check\omega^{h+1}_{t-1}
\end{align*}
also due to the optimal action of the pesimistic model we have
\begin{align*}
\phi(x_t^{h+1},\pi^*(x_t^{h+1}))^\top \opt^{h+1}-\phi(x_t^{h+1},\check\pi_t(x_t^{h+1}))^\top \check\omega^{h+1}_{t-1}\leq \phi(x_t^{h+1},\pi^*(x_t^{h+1}))^\top \opt^{h+1}-\phi(x_t^{h+1},\pi^*(x_t^{h+1}))^\top \check\omega^{h+1}_{t-1}
\end{align*}
With high probability we have;
\begin{align*}
&\|{\wo\chi_t^{H-1}}^{-1/2}{\mathbf{\Phi}_t^{H-1}}^\top(\bm{\wb\nu}_t^{H-1}-\bm{\wt\nu}^{H-1}_t)\|_2\\
&\leq\|{\wo\chi_t^{H-1}}^{-1/2}{\mathbf{\Phi}_t^{H-1}}^\top\|_2\|(\bm{\wb\nu}_t^{H-1}-\bm{\wt\nu}^{H-1}_t)\|_2\\
&\leq\|(\bm{\wb\nu}_t^{H-1}-\bm{\wt\nu}^{H-1}_t)\|_2\\
&=\left(\sum_i^t\left(\phi(x_i^{h+1},\pi^*(x_i^{h+1}))^\top \opt^{h+1}-\phi(x_i^{h+1},\check\pi_t(x_i^{h+1}))^\top \check\omega^{h+1}_{t-1}\right)^2\right)^{1/2}\\
&\leq\left(\sum_i^t\left(\phi(x_i^{h+1},\pi^*(x_i^{h+1}))^\top \opt^{h+1}-\phi(x_i^{h+1},\pi^*(x_i^{h+1}))^\top \check\omega^{h+1}_{t-1}\right)^2\right)^{1/2}\\
&=\left(\sum_i^t\left(\phi(x_i^{h+1},\pi^*(x_i^{h+1}))^\top \left(\opt^{h+1}-\check\omega^{h+1}_{t-1}\right)\right)^2\right)^{1/2}\\
&=\left(\sum_i^t\left(\phi(x_i^{h+1},\pi^*(x_i^{h+1}))^\top {\wo\chi_t^{h+1}}^{-1/2} {\wo\chi_t^{h+1}}^{1/2}\left(\opt^{h+1}-\check\omega^{h+1}_{t-1}\right)\right)^2\right)^{1/2}\\
&\leq\left(\sum_i^t\left(\|\phi(x_i^{h+1},\pi^*(x_i^{h+1}))\|_{{\wo\chi_t^{h+1}}^{-1}} \theta^{h+1}_t(\delta)\right)^2\right)^{1/2}\\
&=\theta^{h+1}_t(\delta)\left(\sum_i^t\|\phi(x_i^{h+1},\pi^*(x_i^{h+1}))\|_{{\wo\chi_t^{h+1}}^{-1}}^2\right)^{1/2}\\
&\leq\theta^{h+1}_t(\delta)\rho_\lambda^{h+1}
\end{align*}
Therefore
\begin{align*}
\|\wh\omega_t^h-\opt^h\|_{\wo\chi_t^h}\leq\theta_t^h(\delta)= \sigma\sqrt{2\log\left(\frac{\det(\wo\chi_t^h)^{1/2}\det(\lambda I)^{-1/2}}{\delta}\right)}+\lambda^{1/2}L_\omega+\theta^{h+1}_t(\delta)\rho_\lambda^{h+1}
\end{align*}

The $\det(\wo\chi_t^h)$ can be written as $\det(\wo\chi_t^h)=\prod_j^d\alpha_j$ therefore, $trace(\wo\chi_t)=\sum_j^d\alpha_j$. We also know that;
\begin{align*}
    (\prod_j^d\alpha_j)^{1/d}\leq \frac{\sum_j^d\alpha_j}{d}
\end{align*}
A matrix extension to Lemma 10 and 11 in \cite{abbasi-yadkori2011improved} results in  $\det(\wo\chi_t^h)\leq \left(\frac{trace(\wo\chi_t^h)}{d}\right)^d$ while we have $trace(\wo\chi_t^h)=trace(\lambda I)+\sum_i^ttrace\left({\phi^h}_i^\top{\phi^h}_i)\right)\leq d\lambda+tL$, 
\begin{align}\label{eq:deter}
    \det(\wo\chi_t^h)\leq \left(\frac{d\lambda+tL}{d}\right)^d
\end{align}

therefore the main statement of Lemma \ref{lem:conf} goes through.
\end{proof}




\begin{proof}~Lemma~\ref{lemma:DeterminantLemma}
We have the following for the determinant of $\det\left(\wo\chi_T^h\right)$ through first matrix determinant lemma and second Sylvester's determinant identity;

\begin{align*}
\det\left(\wo\chi_T^h\right)&=\det\left(\wo\chi_{T-1}+\phi_T^h{\phi_T^h}^\top\right)\\
&=\det\left(1+{\phi^h_T}^\top{\wo\chi^h_{T-1}}^{-1}\phi^h_T\right)\det\left(\wo\chi^h_{T-1}\right)\\
&=\prod_t^T\det\left(1+\|\phi_t^h\|^2_{{\wo\chi_t^h}^{-1}}\right)\det\left(\wt\chi\right)
\end{align*}
Using the fact that $\log(1+\vartheta)\leq \vartheta$ we have

\begin{align}
\sum_t^T\log\left(1+\|\phi_t^h\|^2_{{\wo\chi_t^h}^{-1}}\right)\leq \left(\log\left(\det\left(\wo\chi^h_T\right)\right)-\log\left(\det\left(\wt\chi^h\right)\right)\right)\leq  d\log(\lambda + TL^2/d)
\end{align}

and the statement follows.

\end{proof}

\end{document}